\documentclass{article}  
\usepackage{times}
\usepackage{helvet}       
\usepackage{comment}
\usepackage{float}
\usepackage{cite}
\usepackage{booktabs}
\usepackage{cuted}
\usepackage[ruled, vlined, linesnumbered]{algorithm2e}                                  
\usepackage{bm}             
\usepackage{courier}
\usepackage{subcaption}
\usepackage[hyphens]{url}     
\usepackage{graphicx}
\usepackage{pifont}
\usepackage{amssymb,amsmath}
\usepackage{tikz}
\usetikzlibrary{automata,positioning}
\usepackage{mathtools}
\usepackage{tabularx}
\usepackage{tikz}
\usepackage{subfloat}
\usepackage[skins]{tcolorbox}
\usepackage{hyperref}

\usetikzlibrary{arrows, automata, positioning}
\tikzset{auto, >=stealth}
\tikzset{every edge/.append style={shorten >= 1pt}}
\tikzset{main node/.style={circle,draw,minimum size=1cm,inner sep=0pt},
            }

\newtheorem{theorem}{Theorem}

\newtheorem{corollary}{Corollary}

\newtheorem{definition}{Definition}

\newtheorem{lemma}{Lemma}
\newtheorem{proof}{Proof}

\newtheorem{remark}{Remark}
\newtheorem{task}{Task}

\usepackage{etoolbox}

\providetoggle{long}
 \settoggle{long}{true}

\newtcbox{\remind}{%
  enhanced jigsaw,nobeforeafter,size=fbox,sharp corners,
  shrink tight,
  extrude by=3pt,
  tcbox raise base,
  borderline={0.5pt}{-1pt}{red,opacity=0.75},
  opacityframe=0.75,
  opacityback=0.5,
}

\newcommand{\setOfQFunctions}{Q}
\newcommand{\setOfQFunctionsNew}{Q_{\textrm{new}}}            

\newcommand{\set}[1]{\{ #1 \}}
\newcommand{\reals}{\mathbb{R}}

\newcommand{\methodA}{JIRP}
\newcommand{\methodB}{QAS}
\newcommand{\methodC}{HRL}

\newcommand{\Traffic}{Autonomous Vehicle Scenario}
\newcommand{\traffic}{autonomous vehicle scenario}

\newcommand{\Office}{Office World Scenario}
\newcommand{\office}{office world scenario}
\newcommand{\OfficeA}{Task 2.1}
\newcommand{\officeA}{task 2.1}
\newcommand{\OfficeB}{Task 2.2}
\newcommand{\officeB}{task 2.2}
\newcommand{\OfficeC}{Task 2.3}
\newcommand{\officeC}{task 2.3} 
\newcommand{\OfficeD}{Task 2.4}
\newcommand{\officeD}{task 2.4}

\newcommand{\Craft}{Minecraft World Scenario}
\newcommand{\craft}{Minecraft world scenario}
\newcommand{\CraftA}{Task 3.1}
\newcommand{\craftA}{task 3.1}
\newcommand{\CraftB}{Task 3.2}
\newcommand{\craftB}{task 3.2}
\newcommand{\CraftC}{Task 3.3}
\newcommand{\craftC}{task 3.3}
\newcommand{\CraftD}{Task 3.4}
\newcommand{\craftD}{task 3.4}

\newcommand{\policy}{\pi}
\newcommand{\qValue}{q}
\newcommand{\qValueNew}{q_{\textrm{new}}}

\newcommand{\maxLengthEpisode}{m}

\newcommand{\implement}{encode}
\newcommand{\implementing}{encoding}

\newcommand{\algoName}{JIRP} % at some point we'll have to name our algorithm. Tentatively, I use RPML for Reinforcement Policy and (Reward) Machine Learning
\newcommand{\algoOptName}{JIRP with the algorithmic optimizations}

\newcommand{\hypothesisRM}{\mathcal H}
\newcommand{\hypothesisRMnew}{\mathcal H_{\textrm{new}}}
\newcommand{\counterexamples}{X}
\newcommand{\newCounterexamples}{X_{\textrm{new}}}

\newcommand{\init}{I}

\newcommand{\dfaStates}{V}
\newcommand{\dfaState}{v}

\newcommand{\dfaTransition}{\delta}
\newcommand{\dfaInputAlphabet}{\Sigma}

\newcommand{\dfaFinalStates}{F}
\newcommand{\dfa}{\mathfrak A}
\newcommand{\word}{\omega}
\newcommand{\dfaSymbol}{\ensuremath{b}}

\newcommand{\mdp}{\mathcal{M}}
\newcommand{\mdpStates}{S}
\newcommand{\mdpCommonState}{s}
\newcommand{\mdpInit}{\mdpCommonState_\init}
\newcommand{\mdpActions}{A}
\newcommand{\mdpCommonAction}{a}
\newcommand{\mdpProb}{p}
\newcommand{\mdpRewardFunction}{R}
\newcommand{\mdpDiscount}{\gamma}
\newcommand{\mdpTrajectory}{\zeta}
\newcommand{\mdpLabel}{\ell}
\newcommand{\mdpRewards}{r}
\newcommand{\trajectory}[1]{\ensuremath{\mdpCommonState_0 \mdpCommonAction_0\ldots \mdpCommonState_#1 \mdpCommonAction_#1 \mdpCommonState_{#1 + 1}}}

\newcommand{\rmLabels}{\mathcal P}
\newcommand{\rmLabelingFunction}{L}

\newcommand{\rmInputAlphabet}{2^\rmLabels}
\newcommand{\rmOutputAlphabet}{\mathbb{R}}

\newcommand{\machine}{\mathcal A}
\newcommand{\mealyStates}{V}

\newcommand{\mealyCommonState}{v}
\newcommand{\mealyCommonInput}{u}        
\newcommand{\mealyInputAlphabet}{\ensuremath{{2^\mathcal{P}}}}
\newcommand{\mealyInit}{{\mealyCommonState_\init}}
\newcommand{\mealyInitHat}{{\hat{\mealyCommonState}_\init}}
\newcommand{\mealyOutputAlphabet}{\ensuremath{\mathbb{R}}}
\newcommand{\mealyOutput}{\sigma}
\newcommand{\mealyTransition}{\delta}
\newcommand{\inputTrace}{\lambda}
\newcommand{\outputTrace}{\rho}

\newcommand{\mealyOutputNew}{\mealyOutput_{\textrm{new}}}
\newcommand{\mealyTransitionNew}{\mealyTransition_{\textrm{new}}}
\newcommand{\mealyStatesNew}{\mealyStates_{\textrm{new}}}
\newcommand{\mealyCommonStateNew}{\mealyCommonState_{\textrm{new}}}
\newcommand{\mealyInitNew}{\mealyInit_{\textrm{new}}}

\newcommand{\RPNIMealy}{RPNI-RM}

\title{Joint Inference of Reward Machines and Policies for Reinforcement Learning}

\date{}

\usepackage[margin=1in]{geometry}

\author{Zhe~Xu\thanks{These two authors have contributed equally; the rest of the authors
		are ordered alphabetically. \newline \indent$^\ast$Zhe~Xu and Bo Wu are with the Oden Institute
		for Computational Engineering and Sciences, University of Texas,
		Austin, Austin, TX 78712, Ivan Gavran, Rupak Majumdar and Daniel Neider are with the Max Planck Institute for Software Systems, 67663 Kaiserslautern, Germany, Ufuk Topcu is with the Department
		of Aerospace Engineering and Engineering Mechanics, and the Oden Institute
		for Computational Engineering and Sciences, University of Texas,
		Austin, Austin, TX 78712, e-mail: zhexu@utexas.edu, gavran@mpi-sws.org, ysa6549@gmail.com, rupak@mpi-sws.org, neider@mpi-sws.org, utopcu@utexas.edu, bwu3@utexas.edu.}, Ivan Gavran\footnotemark[1], Yousef Ahmad, Rupak Majumdar, Daniel Neider, Ufuk Topcu and Bo Wu
}

\begin{document}

\maketitle

\begin{abstract} 
	Incorporating \textit{high-level knowledge} is an effective way to expedite  
	reinforcement learning (RL), especially for complex tasks with sparse rewards. We investigate an RL problem where the high-level knowledge is in the form of \textit{reward machines}, i.e., a type of Mealy machine that encodes the reward functions. We focus on a setting in which this knowledge is \textit{a priori} not available to the learning agent. We develop an iterative algorithm that performs joint inference of reward machines and policies for RL (more specifically, q-learning). In each iteration, the algorithm maintains a \textit{hypothesis} reward machine and a \textit{sample} of RL episodes. It derives q-functions from the current hypothesis reward machine, and performs RL to update the q-functions. While performing RL, the algorithm updates the sample by adding RL episodes along which the obtained rewards are inconsistent with the rewards based on the current hypothesis reward machine. In the next iteration, the algorithm infers a new hypothesis reward machine from the updated sample. Based on an \textit{equivalence} relationship we defined between states of reward machines, we transfer the q-functions between the hypothesis reward machines in consecutive iterations. We prove that the proposed algorithm converges almost surely to an optimal policy in the limit if a \textit{minimal} reward machine can be inferred and the maximal length of each RL episode is sufficiently long. The experiments show that learning high-level knowledge in the form of reward machines can lead to fast convergence to optimal policies in RL, while standard RL methods such as q-learning and hierarchical RL methods fail to converge to optimal policies after a substantial number of training steps in many tasks. 
\end{abstract}

%The inference of reward machines Our solution combines the strength of automata learning with RL---our approach enables us to lift the assumption that an external subject provides the reward machine. 
%We address the problem of reinforcement learning under non-Markovian reward decision processes.
%We propose the Iterative Reinforcement Policy and Reward
%Machine Learning (RPML) approach that iteratively infers a reward machine that describes the non-Markovian reward structure and performs reinforcement learning on . 

	\iftoggle{long}{  
%---------- Introduction ----------
% !TEX root = main.tex               

\section{Introduction}                                                         

In many reinforcement learning (RL) tasks, agents only obtain sparse rewards for complex behaviors over a long period of time. In such a setting, learning is very challenging and incorporating high-level knowledge can help the agent explore the environment in a more efficient manner \cite{Taylor_ICML2007}.  This high-level knowledge may be expressed as different levels of temporal or behavioral abstractions, or a hierarchy of abstractions \cite{Nachum_NIPS2018,Abel2018,Akrour2018}. 

The existing RL work exploiting the hierarchy of abstractions often falls into the category of hierarchical RL \cite{sutton1999between,Dietterich2000MaxQ,parr1998HAM}. Generally speaking, hierarchical RL decomposes an RL problem into a hierarchy of subtasks, and uses a \textit{meta-controller} to decide which subtask to perform and a \textit{controller} to decide which action to take within a subtask \cite{Barto2003review}.

%in long-term The techniques of artificial intelligence combined with abundance of available data and increasing processing power have spurred the emergence of intelligent learning systems.
%One of such techniques is . 
% In RL, an agent explores its (unknown) environment and receives occasional rewards, thus learning what policy will maximize the overall reward.
% Specifically, a reward machine takes abstractions of the environment as input, and outputs rewards. 

For many complex tasks with sparse rewards, there exist high-level structural relationships among the subtasks \cite{aksaray2016q,andreas2017modular,li2017reinforcement,zhe_ijcai2019}. 
Recently, the authors in \cite{DBLP:conf/icml/IcarteKVM18} propose \textit{reward machines}, 
i.e., a type of Mealy machines, 
to compactly encode high-level structural relationships. 
They develop a method called \textit{q-learning for reward machines} (QRM) 
and show that QRM can converge almost surely to an optimal policy in the tabular case. 
Furthermore, 
QRM outperforms both q-learning and hierarchical RL 
for tasks where the high-level structural relationships can be encoded by a reward machine.

Despite the attractive performance of QRM, the assumption that the reward machine is explicitly known by the learning agent is unrealistic in many practical situations. 
The reward machines are not straightforward to encode, and more importantly, the high-level structural relationships among the subtasks are often implicit and unknown to the learning agent.
%The problem a) is recognized and addressed in~\cite{LTLAndBeyond}, but problem b), the more fundamental one, remains unsolved.

In this paper, we investigate the RL problem where the high-level knowledge in the form of reward machines is \textit{a priori} not available to the learning agent. 
We develop an iterative algorithm that performs joint inference of reward machines and policies (\algoName) for RL (more specifically, q-learning \cite{Watkins1992}). 
In each iteration, the \algoName\ algorithm maintains a \textit{hypothesis} reward machine and a \textit{sample} of RL episodes. It derives q-functions from the current hypothesis reward machine, and performs RL to update the q-functions. While performing RL, the algorithm updates the sample by adding \textit{counterexamples} (i.e., RL episodes in which the obtained rewards are inconsistent with the rewards based on the current hypothesis reward machine). 
The updated sample is used to infers a new hypothesis reward machine,
using automata learning techniques~\cite{DBLP:conf/nfm/NeiderJ13,oncina1992inferring}.
The algorithm converges almost surely to an optimal policy in the limit 
if a \textit{minimal} reward machine can be inferred 
and the maximal length of each RL episode is sufficiently long. 

%We propose a learning approach that iteratively infers a hypothesis reward machine from data collected from RL and performs RL to obtain an optimal policy for the inferred hypothesis reward machine. Staring from an initial hypothesis reward machine, we perform RL and collect \textit{counterexamples} to form a \textit{sample}. The counterexamples are \textit{state-action sequences} along which the obtained rewards are inconsistent with the hypothesis reward machine. We then adopt \textit{passive} inference techniques of Mealy machines \cite{DBLP:conf/nfm/NeiderJ13,oncina1992inferring} to infer a new reward machine that is consistent with the sample. Afterwards, we perform RL for the inferred reward machine, and add new counterexamples (if they exist) from RL to the sample. The same iterative procedure repeats until the policy converges.                                                                 

%A state of a reward machine is \textit{equivalent} to a state of another reward machine if the rewards obtained starting from the two states are always the same given any input from the environment. We prove that the optimal q-functions from \textit{equivalent} states of reward machines are the same. 
We use three optimization techniques in the proposed algorithm for its practical and efficient implementation. 
First, we periodically add \textit{batches} of counterexamples to the sample for inferring a new hypothesis reward machine. 
In this way, we can adjust the frequency of inferring new hypothesis reward machines. %by adjusting the period of checking counterexamples. 
Second, we utilize the experiences from previous iterations by transferring the q-functions between \emph{equivalent} states of two hypothesis reward machines.
% based on an \textit{equivalence} relationship that we define 
%(inspired by a technique implemented in~\cite{DBLP:conf/icml/IcarteKVM18}). 
%Specifically, the defined \textit{equivalence} relationship is between states of reward machines (see Section \ref{sec_transfer}), which is inspired by \cite{DBLP:conf/icml/IcarteKVM18}. We then prove that the optimal q-functions from \textit{equivalent} states of reward machines are the same. 
Lastly, we adopt a polynomial-time learning algorithm for inferring the hypothesis reward machines. 

We implement the proposed approach and two baseline methods (q-learning in \textit{augmented state space} and deep hierarchical RL \cite{Kulkarni2016hierarchical}) in three scenarios: an \traffic, an \office\ and a \craft. In the \traffic, the proposed approach converges to optimal policies within 100,000 training steps, while the baseline methods are stuck with near-zero average cumulative reward for up to two million training steps. In each of the \office\ and the \craft, over the number of training steps within which the proposed approach converges to optimal policies, the baseline methods reach only 60\% of the optimal average cumulative reward.

%\noindent\textbf{Related work}
%Many hierarchical RL approaches \cite{sutton1999between,parr1998HAM,Dietterich2000MaxQ} outperform standard RL approaches. The use of deep neural networks in hierarchical RL \cite{Kulkarni2016hierarchical} has enabled the approach to outperform typical deep RL approaches.
%
%With implicit high-level knowledge, a similar learning approach to this work has been recently proposed by \cite{zhe_ijcai2019}, where the inferred high-level knowledge is represented by \textit{temporal logic} formulas and used for RL-based transfer learning. In comparison with temporal logic formulas, the reward machines used in this paper are more expressive in representing the high-level structural relationships. 

%Besides, the iterative procedure of inferring reward machines and performing RL for the inferred reward machines can account for more complex tasks, when the inferred reward machine may be incorrect at the first attempt.

\subsection{Motivating Example}
\label{sec_example}

As a motivating example, let us consider an autonomous vehicle navigating a residential area, as sketched in Figure~\ref{fig:ex:road-map}.
As is common in many countries, some of the roads are priority roads.
While traveling on a priority road, a car has the right-of-way and does not need to stop at intersections.
%Typically, there exist special signs on every intersection on the priority roads and no special signs on the intersections on the normal roads. 
In the example of Figure~\ref{fig:ex:road-map}, all the horizontal roads are priority roads (indicated by gray shading), whereas the vertical roads are ordinary roads.

\vspace{-3mm}
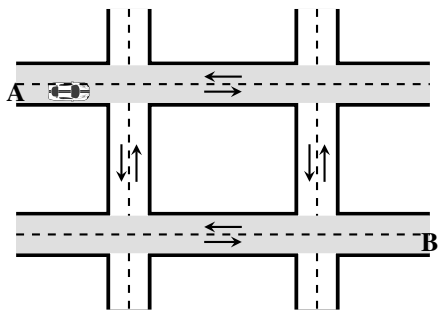
\begin{figure}[th]
	\centering
	\begin{tikzpicture}[thick,scale=1, every node/.style={transform shape}]
	% Road
	\draw[double=white, double distance=5mm, line width=.5mm] (0, 0) -- (5.5, 0) (0, -2) -- (5.5, -2) (1.5, 1) -- (1.5, -3) (4, 1) -- (4, -3);
	\fill[gray!25] ([yshift=-2.5mm]0, 0) rectangle ([yshift=2.5mm]5.5, 0) ([yshift=-2.5mm]0, -2) rectangle ([yshift=2.5mm]5.5, -2);
	\draw[->, thick] (2.5, -0.1)--(3, -0.1); 
	\draw[->, thick] (3, 0.1)--(2.5, 0.1);

	\draw[->, thick] (2.5, -2.1)--(3, -2.1);
    \draw[->, thick] (3, -1.9)--(2.5, -1.9);
    
	\draw[->, thick] (1.4, -0.8)--(1.4, -1.3);  
    \draw[->, thick] (1.6, -1.3)--(1.6, -0.8);                                                                            

	\draw[->, thick] (3.9, -0.8)--(3.9, -1.3);  
    \draw[->, thick] (4.1, -1.3)--(4.1, -0.8);        
	
	\node at (0.7, -0.1){\includegraphics[scale=0.02]{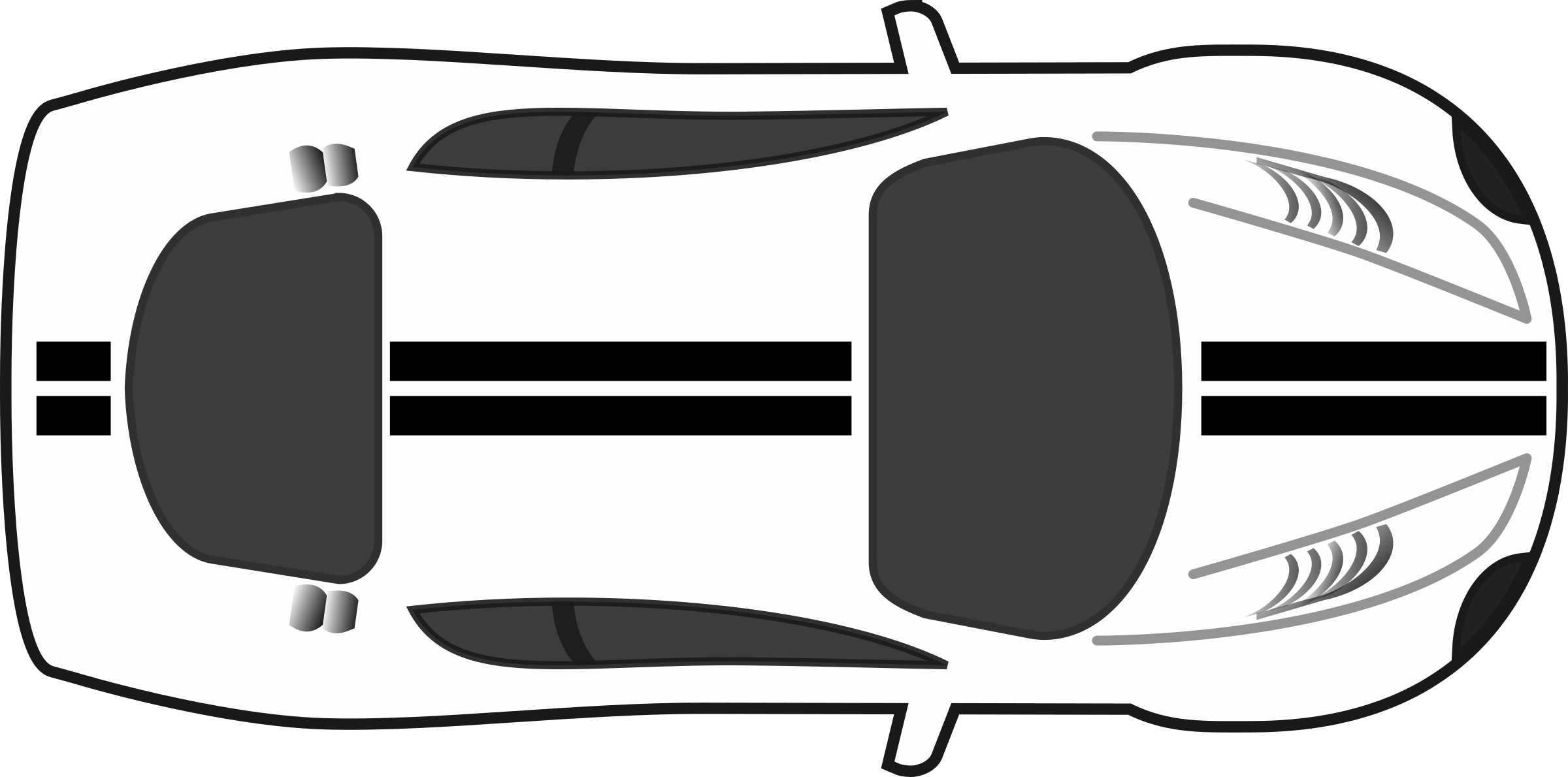}};
	
	% Markings
	\draw[dashed] (0, 0) -- (5.5, 0) (0, -2) -- (5.5, -2);
	\draw[dashed] (1.5, 1) -- (1.5, .25) (4, 1) -- (4, .25) (1.5, -.25) -- (1.5, -1.75) (4, -.25) -- (4, -1.75) (1.5, -2.25) -- (1.5, -3) (4, -2.25) -- (4, -3);
	
	% Labels
	\node[font=\bfseries] at (0, -0.1) {A};
	\node[font=\bfseries] at (5.5, -2.1) {B};
	
	\end{tikzpicture}
	
	\caption{Map of a residential area.} \label{fig:ex:road-map}
\end{figure}

Let us assume that the task of the autonomous vehicle is to drive from position ``A'' (a start position) on the map to position ``B'' while obeying the traffic rules.
To simplify matters, we are here only interested in the traffic rules concerning the right-of-way and how the vehicle acts at intersections with respect to the traffic from the intersecting roads.
Moreover, we make the following two further simplifications: (1) the vehicle correctly senses whether it is on a priority road and (2) the vehicle always stays in the road and goes straight forward while not at the intersections.                     

The vehicle is obeying the traffic rules if and only if
\begin{itemize}
	\item it is traveling on an ordinary road and stops for exactly one time unit at the intersections;
	\item it is traveling on a priority road and does not stop at the intersections.
\end{itemize}
%In particular, the vehicle violates the rules and does not get a reward if it yields on a priority road or on a normal road when there is no traffic on the intersecting road.
After a period of time (e.g., 100 time units), the vehicle receives a reward of 1 if it reaches B while obeying the traffic rules, otherwise it receives a reward of 0.

\section{Preliminaries}
\label{sec:preliminaries}
In this section we introduce necessary background on reinforcement learning and reward machines.

\subsection{Markov Decision Processes and Reward Machines}

\begin{definition}
A labeled Markov decision process is a tuple
$\mdp = (\mdpStates, \mdpInit, \mdpActions, \mdpProb, \mdpRewardFunction, \mdpDiscount, \rmLabels, \rmLabelingFunction)$
consisting of a finite state space $\mdpStates$, 
an agent's initial state $\mdpInit \in \mdpStates$, 
a finite set of actions $\mdpActions$,
and a probabilistic transition function
$\mdpProb \colon \mdpStates\times \mdpActions \times \mdpStates \rightarrow [0,1]$.
A reward function $\mdpRewardFunction: (\mdpStates \times \mdpActions)^+ \times \mdpStates \rightarrow \reals$
and a discount factor $\mdpDiscount \in [0,1)$ together specify payoffs to the agent.
Finally, a finite set $\rmLabels$ of propositional variables, 
and a labeling function $\rmLabelingFunction: \mdpStates\times \mdpActions\times \mdpStates\rightarrow 2^\rmLabels$
determine the set of relevant high-level events that the agent detects in the environment. We define the size of $\mdp$, denoted as $|\mdp|$, to be $|\mdpStates|$ (i.e., the cardinality of the set $\mdpStates$).
\end{definition}

A \emph{policy} is a function mapping states in $\mdpStates$ to a probability distribution over actions in $\mdpActions$.
At state $s\in \mdpStates$, an agent using policy $\policy$ picks an action $a$ with probability $\policy(s, a)$,
and the new state $s'$ is chosen with probability $\mdpProb(s, a, s')$. 
A policy $\policy$ and the initial state $\mdpInit$ together determine a stochastic process and we write 
$S_0 A_0 S_1 \ldots$ for the random trajectory of states and actions.

A \emph{trajectory} is a realization of this stochastic process:
a sequence of states and actions $s_0 a_0 s_1 \ldots s_k a_k s_{k+1}$,
with $s_0 = \mdpInit$. 
Its corresponding \emph{label sequence} is $\mdpLabel_0 \mdpLabel_1\ldots \mdpLabel_{k}$
where
$\rmLabelingFunction(\mdpCommonState_i, \mdpCommonAction_{i}, \mdpCommonState_{i+1}) = \mdpLabel_i$ for each $i \leq k$.
Similarly, the corresponding \emph{reward sequence} is $\mdpRewards_1\ldots\mdpRewards_k$,
where $\mdpRewards_i = \mdpRewardFunction(\trajectory{i})$, for each $i \leq k$. 
We call the pair $(\inputTrace, \outputTrace):=(\mdpLabel_1\ldots\mdpLabel_k, \mdpRewards_1\ldots\mdpRewards_k)$ a \emph{trace}.

A trajectory $s_0 a_0 s_1 a_1 \ldots s_k a_k s_{k+1}$ achieves a reward $\sum_{i=0}^{k} \mdpDiscount^i \mdpRewardFunction(s_0a_0\ldots s_i a_i s_{i+1})$.
% now that gamma is in the definition, it is not necessary
%where $\mdpDiscount \in [0,1)$ is a discount factor. 
In reinforcement learning, the objective of the agent is to maximize the expected cumulative reward, $\mathbb{E}_\policy[\sum_{i=0}^\infty \mdpDiscount^i \mdpRewardFunction(S_0A_0\ldots S_{i+1})]$.

Note that the definition of the reward function assumes that the reward is a function of the whole trajectory. 
A special, often used, case of this is a so-called \emph{Markovian} reward function, 
which depends only on the last transition
(i.e., $\mdpRewardFunction(\zeta\cdot(s,a)s') = \mdpRewardFunction(\zeta'\cdot(s,a)s')$
for all $\zeta, \zeta'\in (\mdpStates\times\mdpActions)^*$, where we use $\cdot$ to denote concatenation).

Our definition of MDPs corresponds to the ``usual'' definition of MDPs (e.g., \cite{Puterman}),  
except that we have introduced
a set of high-level propositions $\rmLabels$ 
and a labeling function $\rmLabelingFunction$ assigning sets of propositions (\emph{labels}) to each transition $(s, a, s')$ of an MDP.
We use these labels to define (general) reward functions through \emph{reward machines}.
Reward machines~\cite{DBLP:conf/icml/IcarteKVM18,LTLAndBeyond} are a type of finite-state machines---when in one of its finitely many states, upon 
reading a symbol, such a machine outputs a reward and transitions into a next state.\footnote{
	The reward machines we are using are the so-called \emph{simple reward machines} in the parlance of~\cite{DBLP:conf/icml/IcarteKVM18},
	where every output symbol is a real number.
}
%(Note that the states of reward machines are separate from the states of an underlying MDP.)

\begin{definition}%[Reward Mealy machines]
\label{def:rewardMealyMachines}
A \emph{reward machine} 
$\machine = (\mealyStates, \mealyInit, \mealyInputAlphabet, \mealyOutputAlphabet, \mealyTransition, \mealyOutput)$ consists of 
a finite, nonempty set $\mealyStates$ of states, 
an initial state $\mealyInit \in \mealyStates$, 
an input alphabet $\mealyInputAlphabet$,
an output alphabet $\mealyOutputAlphabet$, 
a (deterministic) transition function $\mealyTransition \colon \mealyStates \times \mealyInputAlphabet \to \mealyStates$, 
and an output function $\mealyOutput \colon \mealyStates \times \mealyInputAlphabet \to \mealyOutputAlphabet$.
We define the size of $\machine$, denoted as $|\machine|$, to be $|\mealyStates|$ (i.e., the cardinality of the set $\mealyStates$).
\end{definition}

Technically, a reward machine is a special instance of a Mealy machine~\cite{DBLP:books/daglib/0025557}:
the one that has real numbers as its output alphabet
and subsets of propositional variables (originating from an underlying MDP) as its input alphabet.
(To accentuate this connection, 
a defining tuple of a reward machine
explicitly mentions 
both the input alphabet $\mealyInputAlphabet$
and the output alphabet $\mealyOutputAlphabet$.)

The run of a reward machine $\machine$ on a sequence of labels $\mdpLabel_1\ldots \mdpLabel_k\in (\mealyInputAlphabet)^*$ is a sequence 
$\mealyCommonState_0 (\mdpLabel_1, \mdpRewards_1) \mealyCommonState_1 (\mdpLabel_2, \mdpRewards_2)\ldots (\mdpLabel_k, \mdpRewards_k) \mealyCommonState_{k+1}$ of states and label-reward pairs such that $\mealyCommonState_0 = \mealyInit$
and for all $i\in\set{0,\ldots, k}$, we have $\mealyTransition(\mealyCommonState_i, \mdpLabel_i) = \mealyCommonState_{i+1}$ and $\mealyOutput(\mealyCommonState_i,\mdpLabel_i) = \mdpRewards_i$.
We write $\machine(\mdpLabel_1\ldots\mdpLabel_k) = \mdpRewards_1\ldots\mdpRewards_k$ 
to connect the input label sequence 
%$\mdpLabel_1\ldots\mdpLabel_k$
to the sequence of rewards
%$\mdpRewards_1\ldots\mdpRewards_k$ 
produced by the machine $\machine$.
We say that
a reward machine $\machine$ \emph{\implement s} 
the reward function $\mdpRewardFunction$ of an MDP
if for every trajectory $\trajectory{k}$
and the corresponding label sequence $\mdpLabel_1\ldots \mdpLabel_k$, 
%the value of the reward function is captured by the last element of the reward sequence,
%$\mdpRewardFunction(\trajectory{k}) = \mdpRewards_k$.
the reward sequence equals $\machine(\mdpLabel_1\ldots \mdpLabel_k)$.
\footnote{In general, there can be multiple reward machines that encode the 
	reward function of an MDP: all such machines agree on the label 
	sequences that arise from trajectories of the underlying MDP, but they 
	might differ on label sequences that the MDP does not permit. For 
	clarity of exposition and without loss of generality, we assume 
	throughout this paper that there is a unique reward machine encoding the 
	reward function of the MDP under consideration. However, our algorithm 
	also works in the general case.
}

%\RM{omit?}
An interesting (and practically relevant) subclass of reward machines is given by Mealy machines with a specially marked subset of \emph{final states}, the output alphabet $\set{0, 1}$, and the output function mapping a transition to 1 if and only if the end-state is a final state and the transition is not a self-loop. 
Additionally, final states must not be a part of any cycle, except for a self-loop.
This special case can be used in reinforcement learning scenarios with \textit{sparse} reward
functions (e.g., see the reward machines used in the case studies in \cite{DBLP:conf/icml/IcarteKVM18}). 
%\RM{end omit}
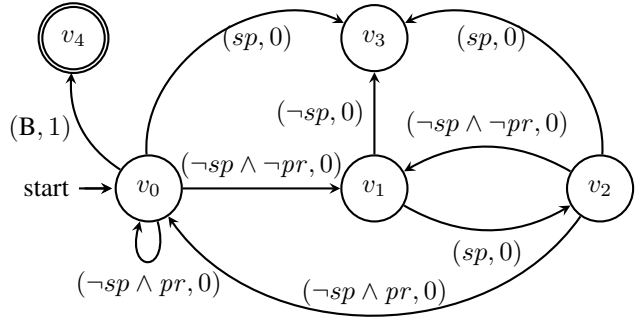
\begin{figure}[t]
	\centering
	\begin{tikzpicture}[thick,scale=1, every node/.style={transform shape}]
	\node[state,initial] (0) at (-0.5, 0) {$\mealyCommonState_0$};
	\node[state] (1) at (2.5, 0) {$\mealyCommonState_1$};
	\node[state] (2) at (5.5, 0) {$\mealyCommonState_2$};
	\node[state] (3) at (2.5, 2) {$\mealyCommonState_3$};
	\node[state, accepting] (4) at (-1.5, 2) {$\mealyCommonState_4$};
	\draw[<-, shorten <=1pt] (0.west) -- +(-.4, 0);
	\draw[->] (0) to[left] node[above, align=center] {$(\lnot sp \land \lnot\mathit{pr}, 0)$} (1);
	% \node[] at  (0.8,0.1) {$Low$};
	\draw[->] (0) to[loop below] node[align=center] {$(\lnot sp \land \mathit{pr}, 0)$} ();
	\draw[->] (1) to[bend right] node[sloped, below, align=center] {$(sp, 0)$} (2);
	\draw[->] (2) to[bend left=55] node[sloped, above, align=center] {$(\lnot sp \land \mathit{pr}, 0)$} (0);
	\draw[->] (2) to[bend right=60] node[left, align=center] {$(sp, 0)$} (3);
	\draw[->] (0) to[bend left=60] node[right, align=center] {$(sp, 0)$} (3);
	\draw[->] (1) to[right] node[left, align=center] {$(\lnot sp, 0)$} (3);
	%	\draw[->] (1) to[bend left] node[above, align=center] {$(\lnot sp \land \lnot\mathit{pr}, 0)$} (3);
	\draw[->] (0) to[bend left] node[left, align=center] {$(\textrm{B}, 1)$} (4);
	% \draw[->] (1) to[bend left] node[below, align=center] {$(B, 1)$} (4);
	% \draw[->] (2) to[bend left] node[below, align=center] {$(B, 1)$} (4);
	\draw[->] (2) to[bend right] node[above, align=center] {$(\lnot sp \land \lnot\mathit{pr}, 0)$} (1);
	\end{tikzpicture}	
	\caption{Reward machine for the autonomous vehicle. $sp$: stop at an intersection; $\lnot sp$: not stop at an intersection; $pr$: end in a priority road; $\lnot pr$: end in an ordinary road. An edge $(sp, 0)$ between $v_0$ and $v_3$ means that the reward machine will transition from $\mealyCommonState_0$ to $\mealyCommonState_3$ if the proposition (label) $sp$ becomes true and output a reward equal to zero. } 
	\label{fig:ex:reward-machine}
\end{figure}

For example, Figure~\ref{fig:ex:reward-machine} shows a reward machine for our motivating example. 
Intuitively, state $\mealyCommonState_0$ corresponds to the vehicle traveling on a priority road, while $\mealyCommonState_1$ and $\mealyCommonState_2$ correspond to the vehicle traveling and stopped on an ordinary road, respectively.
While in $\mealyCommonState_0$, the vehicle ends up in a sink state $\mealyCommonState_3$ (representing violation of the traffic rules) if it stops at an intersection ($\mathit{sp}$).
While in state $\mealyCommonState_1$, the vehicle gets to the sink state $\mealyCommonState_3$ if it does not stop at an intersection ($\lnot\mathit{sp}$), and gets to state $\mealyCommonState_2$ if it stops at an intersection ($\mathit{sp}$). While in state $\mealyCommonState_2$, the vehicle gets to the sink state $\mealyCommonState_3$ if it stops again at the same intersection ($\mathit{sp}$), gets back to state $\mealyCommonState_0$ if it turns left or turns right (thus ending in a priority road, i.e., $\lnot sp\wedge pr$), and gets back to state $\mealyCommonState_0$ if it goes straight (thus ending in an ordinary road, i.e., $\lnot sp\wedge \lnot pr$).
The reward machine switches among states $\mealyCommonState_0$, $\mealyCommonState_1$ and $\mealyCommonState_2$ if the vehicle is obeying the traffic rules.
Finally, the reward 1 is obtained if from $\mealyCommonState_0$ the goal position B is reached.
(Transitions not shown in Figure~\ref{fig:ex:reward-machine} are self-loops with reward 0.)

%\begin{figure}[th]
%	\centering
%	\begin{tikzpicture}
%	\node[state] (0) at (0, 0) {$p_0$};
%	\node[state] (1) at (3.5, 0) {$p_1$};
%	\node[state] (x) at (0, - 3.5) {$p_x$};
%	\node[state] (2) at (4.5, - 3.5) {$p_2$};
%	\draw[<-, shorten <=1pt] (0.west) -- +(-.4, 0);
%	\draw[->] (0) to[bend right] node[sloped, below, align=center] {$(\lnot B \land \mathit{tr} \land \lnot y, 0)$ \\ $(\lnot B \land \lnot \mathit{tr} \land y$,0)} (x);
%	\draw[->] (0) to[loop above] node[align=center] {$(\lnot B \land \lnot\textit{pr}, 0)$\\$(\lnot B \land \textit{tr} \land \mathit{y}, 0)$ \\ $(\lnot B \land \lnot\textit{tr} \land \lnot\mathit{y}, 0)$} ();
%	\draw[->] (0) to[bend right=10] node[sloped, above, align=center] {$(B, 1)$} (2);
%	\draw[->] (0) to[bend left=20] node {$(\lnot B \land \mathit{pr}, 0)$} (1);			
%	\draw[->] (1) to[loop above] node[align=center] {$(\lnot B \land \lnot\mathit{y}, 0)$ \\ $(\lnot B \land\mathit{pr}, 0)$} ();
%	\draw[->] (1) to[bend left] node [above]{$(\lnot B \land\lnot\mathit{pr}, 0)$} (0);
%	\draw[->] (1) to[bend left=20] node[align=left, sloped, below] {$(y,0)$} (x);
%	\draw[->] (1) to[bend left=10] node[align=center, sloped] {$(B, 1)$} (2);
%	\draw[->] (x) to[loop below] node {$(*,0)$} ();
%	\draw[->] (2) to[loop below] node {$(*,0)$} ();
%	\end{tikzpicture}	
%	\caption{Reward machine for traveling on priority roads} \label{fig:ex:reward-machine}
%\end{figure}

\begin{algorithm}[t]
	\DontPrintSemicolon
	\SetKwBlock{Begin}{function}{end function}       
	{ \textbf{Hyperparameter}: episode length \textit{eplength} } \;
	{ \textbf{Input:} a reward machine $(\mealyStates, \mealyInit, \mealyInputAlphabet, \mealyOutputAlphabet, \mealyTransition, \mealyOutput)$, 
		a set of q-functions $\setOfQFunctions = \set{\qValue^{\mealyCommonState} | \mealyCommonState \in \mealyStates }$ }\;
	$\mdpCommonState \gets \mathit{InitialState()}; \mealyCommonState \gets \mealyInit; \inputTrace \gets [ ]; \outputTrace \gets [ ]$ \;
	\For{$0 \leq t < \mathit{eplength}$ \label{algLine:taskLoopStart}} 
	{ { $\mdpCommonAction \gets \text{GetEpsilonGreedyAction}(\qValue^{\mealyCommonState}, \mdpCommonState)$ \label{algLine:choiceOfAction} } \;
		{ $\mdpCommonState' \gets \text{ExecuteAction}(\mdpCommonState, \mdpCommonAction)$ } \;
		{  $\mealyCommonState' \gets \mealyTransition( \mealyCommonState, \rmLabelingFunction(\mdpCommonState, \mdpCommonAction, \mdpCommonState') )$ \label{algLine:RMTransition} } \;
		{$\mdpRewards \gets \mealyOutput(\mealyCommonState, \rmLabelingFunction(\mdpCommonState, \mdpCommonAction, \mdpCommonState'))$ \,   \remind{or observe reward in \algoName}\label{algLine:Reward}}  \;
		{ $\text{update }\qValue^{\mealyCommonState}(\mdpCommonState, \mdpCommonAction) \text{ using reward }\mdpRewards$ \label{algLine:update} }\;
		\For{$\hat{\mealyCommonState} \in \mealyStates \setminus \set{\mealyCommonState}$  \label{algLine:learningLoopStart}  } 
		{ 
			{ $\hat\mealyCommonState' \gets \mealyTransition(\hat\mealyCommonState, \rmLabelingFunction(\mdpCommonState, \mdpCommonAction, \mdpCommonState'))$ }  \;
			{ $\hat{\mdpRewards} \gets \mealyOutput(\hat\mealyCommonState, \rmLabelingFunction(\mdpCommonState, \mdpCommonAction, \mdpCommonState'))$ } \;
			{ $\text{update }\qValue^{\hat\mealyCommonState}(\mdpCommonState, \mdpCommonAction) \text{ using reward }\hat{\mdpRewards}$ \label{algLine:learningLoopEnd}} \; 
			}
		%{ $\inputTrace \gets \inputTrace\frown \rmLabelingFunction(\mdpCommonState, \mdpCommonAction, \mdpCommonState'); \outputTrace  \gets \outputTrace \frown\mdpRewards$ \label{algLine:tracesConcatenation} } \;
		{ append  $\rmLabelingFunction(\mdpCommonState, \mdpCommonAction, \mdpCommonState')$ to $\inputTrace$; append $\mdpRewards$ to $\outputTrace$ \label{algLine:tracesConcatenation}} \;
		{ $\mdpCommonState \gets \mdpCommonState'; \mealyCommonState \gets \mealyCommonState'$ \label{algLine:taskLoopEnd}} \;
	}
	{ \Return $(\inputTrace, \outputTrace, \setOfQFunctions)$ }
	\caption{QRM\_episode}
	\label{alg:QRMepisode}
\end{algorithm}

% the agent knows neither its environment nor the exact task specification.
% It is therefore left to explore the environment and receive occasional rewards reinforcing desired behaviors.
% Based on this exploration, the agent learns a policy that maximizes the overall received reward~\cite{sutton2018reinforcement}.
% The (unknown) environment is usually modeled as a Markov decision process (MDP).

\subsection{Reinforcement Learning With Reward Machines}

In reinforcement learning, an agent explores the environment modeled by an MDP,
receiving occasional rewards according to the underlying reward function \cite{sutton2018reinforcement}.
One possible way to learn an optimal policy is tabular q-learning \cite{Watkins1992}.
There, the value of the function $q(s,a)$, 
that represents the expected future reward
for the agent taking action $a$ in state $s$,
is iteratively updated. 
Provided that all state-action pairs are seen infinitely often,
q-learning converges to an optimal policy in the limit,
for MDPs with a Markovian reward function \cite{Watkins1992}. 

The q-learning algorithm can be modified to learn an optimal policy when the general reward function is encoded by a reward machine
\cite{DBLP:conf/icml/IcarteKVM18}.
Algorithm~\ref{alg:QRMepisode} shows one episode of the QRM algorithm.
It maintains a set $\setOfQFunctions$ of q-functions, denoted as $\qValue^{\mealyCommonState}$ for each state $\mealyCommonState$ of the reward machine.

The current state $\mealyCommonState$ of the reward machine guides the exploration by determining which q-function is used to choose the next action (line \ref{algLine:choiceOfAction}).
However, in each single exploration step, the q-functions corresponding to all reward machine states are updated (lines \ref{algLine:update} and \ref{algLine:learningLoopEnd}).

%The exploration is guided by the information encoded in the reward machine (lines \ref{algLine:choiceOfAction} and~\ref{algLine:RMTransition}).
%It consists of two nested parts: the task loop (lines \ref{algLine:taskLoopStart} to \ref{algLine:taskLoopEnd}) and 
%the learning loop (lines \ref{algLine:learningLoopStart} to \ref{algLine:learningLoopEnd}).

%The outer loop is similar to classical q-learning and explores the MDP for a horizon equal to the horizon length, $\mathit{eplength}$.
%The exploration is guided by the information encoded in the reward machine (lines \ref{algLine:choiceOfAction} and~\ref{algLine:RMTransition}).
%All $q$-functions receive their updates---the one corresponding to the current reward machine state of the agent, 
%but also the ones corresponding to all the other states of the reward machine.

The modeling hypothesis of QRM is that the rewards are known, but the transition probabilities are unknown. Later, we shall relax the assumption that rewards are known and we shall instead \emph{observe} the rewards (in line~\ref{algLine:Reward}).
During the execution of the episode, traces $(\inputTrace, \outputTrace)$ of the reward machine are collected (line~\ref{algLine:tracesConcatenation}) and returned in the end.
While not necessary for q-learning, the traces will be useful in our algorithm to check the consistency of an \textit{inferred} reward machine with rewards received from the environment (see Section \ref{sec_baseline}).

%---------- Reinforcement Policy and Reward Machines Learning ----------
% \input{rlAndAutomataLearning}

% !TEX root = main.tex

\section{Joint Inference of Reward Machines and Policies (\algoName)}
\label{sec_baseline}
Given a reward machine, the QRM algorithm learns an optimal policy.
However, in many situations,  
assuming the knowledge of the reward function (and thus the reward machine) is unrealistic.
Even if the reward function is known,
encoding it in terms of a reward machine can be non-trivial.          
In this section, we describe an RL algorithm that \emph{iteratively}
infers (i.e., learns) the reward machine and the optimal policy for the reward machine.
 
% how to combine reinforcement learning and automata learning, which enables us to drop the assumption of the correct reward machine being provided by the user.
% We first present our overall framework and then prove that it preserves nice the property of learning in the limit of Q-learning and QRM-learning.

Our algorithm combines an automaton learning algorithm to infer hypothesis reward machines and the QRM algorithm for RL on the current candidate.
Inconsistencies between the hypothesis machine and the observed traces are used to trigger re-learning of the reward machine.
We show that the resulting iterative algorithm converges in the limit almost surely to the reward machine \implementing\ the reward function
 and to an optimal policy for this reward machine.

\subsection{\algoName\ Algorithm}

Algorithm~\ref{alg:RPML} describes our \algoName{} algorithm.
It starts with an initial hypothesis reward machine $\mathcal{H}$ and runs the QRM algorithm
to learn an optimal policy.
The episodes of QRM are used to collect traces and update q-functions.
As long as the traces are consistent with the current hypothesis reward machine,  
QRM explores more of the environment using the reward machine to guide the search.                          
%However, once a mismatch between the observed rewards and the rewards suggested by the hypothesis is obs\inputTraceerved
However, if a trace $(\inputTrace, \outputTrace)$ is detected that is inconsistent with the hypothesis reward machine (i.e., $\hypothesisRM(\inputTrace) \neq \outputTrace$, Line~\ref{alg:RPML:line:cexFound}), our algorithm stores it in a set $\counterexamples$ (Line~\ref{alg:RPML:addingToCounterexamples})---we call the trace $(\inputTrace, \outputTrace)$ a \emph{counterexample} and the set $\counterexamples$ a \emph{sample}.
Once the sample is updated, the algorithm re-learns a new hypothesis reward machine (Line~\ref{alg:RPML:line:infer}) and proceeds.
Note that we require the new hypothesis reward machine to be minimal (we discuss this requirement shortly).

\begin{algorithm}
	\DontPrintSemicolon
	\SetKwBlock{Begin}{function}{end function}       
%	\Begin($\text{ComputeIG} {(} \mathcal{S}^G_L=\{g_1, \dots,g_m\},\varphi_{\theta},\mathcal{F}^G_L {)}$)
%	{
         { Initialize the hypothesis reward machine $\hypothesisRM$ with a set of states $\mealyStates$} \;
%       	{ Initialize one q-function per state, $\qValue^{\mealyCommonState_0}, \ldots, \qValue^{\mealyCommonState_k}$ }\;
      	 { Initialize a set of q-functions $\setOfQFunctions = \set{\qValue^\mealyCommonState | \mealyCommonState \in \mealyStates}$ }\;              
         { Initialize $\counterexamples = \emptyset$ }\;
			\For{episode $n = 1, 2, \ldots$} 
		{ $(\inputTrace, \outputTrace, \setOfQFunctions) = \text{QRM\_episode}(\hypothesisRM, \setOfQFunctions)$ \;
		\If {$\hypothesisRM(\inputTrace) \neq \outputTrace$ \label{alg:RPML:line:cexFound} } 
		{ 
		add $(\inputTrace, \outputTrace)$ to $\counterexamples$ \label{alg:RPML:addingToCounterexamples}  \;
		 infer a new, minimal hypothesis reward machine $\hypothesisRM$ based on the traces in $\counterexamples$ \label{alg:RPML:line:infer} \;
		  re-initialize $\setOfQFunctions$ \label{alg:RPML:line:qFunctionReinit}
		  }
    	}
	\caption{\algoName}                                           
	\label{alg:RPML}
\end{algorithm}

%\begin{algorithm}[t]
%	\caption{\algoName}                                                                                                                                                                    
%	\label{alg:RPML}
%	\begin{algorithmic}[1]
%		\State Initialize the hypothesis reward machine $\hypothesisRM$
%		\State Initialize$\counterexamples = \emptyset$
%		\For{episode $n = 1, 2, \ldots$} 
%		 \State  $(\inputTrace, \outputTrace) = \text{QRM\_episode}(\hypothesisRM)$
%	  	 \If {$\hypothesisRM(\inputTrace) \neq \outputTrace$} \label{alg:RPML:line:cexFound}
%		     \State add $(\inputTrace, \outputTrace)$ to $\counterexamples$ \label{alg:RPML:addingToCounterexamples}
%     		     \State infer a new, minimal reward machine $\hypothesisRM$ based the traces in $\counterexamples$ \label{alg:RPML:line:infer}
%		\EndIf
%		\EndFor
%%		\State \Return $\qValue$
%	\end{algorithmic} 
%\end{algorithm}    

%---------- Passive Inference of Reward Machines ----------
\subsection{Passive Inference of Minimal Reward Machines}

Intuitively, a sample $\counterexamples \subset (\mealyInputAlphabet)^+ \times \mealyOutputAlphabet^+$ contains a finite number of counterexamples.
Consequently, we would like to construct a new reward machine $\hypothesisRM$ that is (a) \emph{consistent with $X$} in the sense that $\hypothesisRM(\inputTrace) = \outputTrace$ for each $(\inputTrace, \outputTrace) \in \counterexamples$ and (b) \emph{minimal}.
We call this task \emph{passive learning of reward machines}.
The phrase ``passive'' here refers to the fact that the learning algorithm is not allowed to query for additional information,
as opposed to Angluin's famous ``active'' learning framework~\cite{DBLP:journals/iandc/Angluin87}.

\begin{task} \label{task:passive-learning-of-reward-machines}
Given a finite set $\counterexamples \subset (\mealyInputAlphabet)^+ \times \mealyOutputAlphabet^+$, \emph{passive learning of reward machines} refers to the task of constructing a minimal reward machine $\hypothesisRM$ that is consistent with $X$ (i.e., that satisfies $\hypothesisRM(\inputTrace) = \outputTrace$ for each $(\inputTrace, \outputTrace) \in \counterexamples$).
\end{task}

%Since reward machines are in essence Mealy machines, we can directly apply existing, off-the-shelf passive learning algorithms for Mealy machines to solve Task~\ref{task:passive-learning-of-reward-machines}.
%Due to the close connection of Mealy machines and deterministic finite automata (DFAs), learning algorithms for Mealy machines are often (slight) modifications of learning algorithms for regular languages.
%In fact, numerous automata learning libraries, such as flexfringe~\cite{DBLP:conf/icsm/VerwerH17}, LearnLib~\cite{DBLP:conf/cav/IsbernerHS15}, and libalf~\cite{libalfTool}, implement passive learning algorithms for Mealy machines.
%In the remainder of this section, we briefly sketch two such algorithm, one based on state-merging and one based on SAT solving.
%For the sake of brevity, we refer to these two algorithms as \emph{\RPNIMealy} and \emph{\SATMealy}, respectively.

Note that this learning task asks to infer not an arbitrary reward machine but a \emph{minimal} one (i.e., a consistent reward machine with the fewest number of states among all consistent reward machines).
This additional requirement can be seen as an Occam's razor strategy~\cite{DBLP:conf/tacas/LodingMN16} and is crucial in that it guarantees \algoName\ to converge to the optimal policy in the limit.
Unfortunately, Task~\ref{task:passive-learning-of-reward-machines} is computationally hard in the sense that the corresponding decision problem
\begin{quote}
	\itshape
	``given a sample $\counterexamples$ and a natural number $k > 0$, does a consistent Mealy machine with at most $k$ states exist?''
\end{quote}
is NP-complete.
This is a direct consequence of Gold's (in)famous result for regular languages~\cite{DBLP:journals/iandc/Gold78}.

Since this  problem is computationally hard, a promising approach is to learn minimal consistent reward machines with the help of highly-optimized SAT solvers (\cite{DBLP:conf/icgi/HeuleV10}, \cite{DBLP:conf/nfm/NeiderJ13}, and \cite{DBLP:phd/dnb/Neider14} describe similar learning algorithms for inferring minimal deterministic finite automata from examples).
The underlying idea is to generate a sequence of formulas $\varphi_k^\counterexamples$ in propositional logic for increasing values of $k \in \mathbb N$ (starting with $k=1$) that satisfy the following two properties:
\begin{itemize}
	\item $\varphi_k^\counterexamples$ is satisfiable if and only if there exists a reward machine with $k$ states that is consistent with $\counterexamples$; and
	\item a satisfying assignment of the variables in $\varphi_k^\counterexamples$ contains sufficient information to derive such a reward machine.
\end{itemize}
By increasing $k$ by one and stopping once $\varphi_k^\counterexamples$ becomes satisfiable (or by using a binary search), an algorithm that learns a minimal reward machine that is consistent with the given sample is obtained.

Despite the advances in the performance of SAT solvers, 
this approach still does not scale to large problems.
Therefore, one often must resort to polynomial-time heuristics.

\subsection{Convergence in the Limit}

Tabular q-learning and QRM both eventually converge to a q-function defining an optimal policy almost surely.
% The same property holds for QRM learning algorithm~\cite{DBLP:conf/icml/IcarteKVM18}.
We show that the same desirable property holds for \algoName{}.
More specifically, in the following sequence of lemmas we show that---given a long enough exploration---\algoName{} will converge to the
reward machine that \implement s the reward function of the underlying MDP.
We then use this fact to show that overall learning process converges to an optimal policy (see Theorem~\ref{thm:convergenceInTheLimit}).

We begin by defining \textit{attainable trajectories}---trajectories that can possibly appear in the exploration of an agent.

\begin{definition} \label{def:attainable-trajectory}
Let $\mdp = (\mdpStates, \mdpInit, \mdpActions, \mdpProb, \mdpRewardFunction, \mdpDiscount, \rmLabels, \rmLabelingFunction)$ be a labeled MDP and $\maxLengthEpisode \in \mathbb N$ a natural number.
We call a trajectory $\mdpTrajectory = \mdpCommonState_0 \mdpCommonAction_0 \mdpCommonState_1 \ldots \mdpCommonState_k \mdpCommonAction_k \mdpCommonState_{k+1} \in (\mdpStates \times \mdpActions)^\ast \times \mdpStates$ \emph{$\maxLengthEpisode$-attainable} if (i) $k \leq \maxLengthEpisode$ and (ii) $\mdpProb(\mdpCommonState_i, \mdpCommonAction_i, \mdpCommonState_{i+1}) > 0$ for each $i \in \{ 0, \ldots, k \}$.
Moreover, we say that a trajectory $\mdpTrajectory$ is \emph{attainable} if there exists an $\maxLengthEpisode \in \mathbb N$ such that $\mdpTrajectory$ is $\maxLengthEpisode$-attainable.
%Moreover, if $\machine$ is a reward machine for $\mdp$, we call an input sequence $\tau_0 \ldots \tau_n$ \emph{$\maxLengthEpisode$-attainable}
%if there exists an $\maxLengthEpisode$-attainable trajectory $s_0 a_0 s_1 \ldots s_n a_n s_{n+1}$ such that $\tau_i = \rmLabelingFunction(s_i, a_i, s_{i+1})$ for each $i \in \{ 0, \ldots, n \}$.
%We call a pair of an attainable input sequence and its corresponding output sequence $(\tau_0 \ldots \tau_n, \eta_0 \ldots \eta_n)$ \emph{an $\maxLengthEpisode$-attainable trace}.
\end{definition}

An induction shows that \algoName{} almost surely explores every attainable trajectory in the limit (i.e., with probability $1$ when the number of episodes goes to infinity).

\begin{lemma} \label{lem:attainable-trajectories}
Let $\maxLengthEpisode \in \mathbb N$ be a natural number.
Then, \algoName\ with $\mathit{eplength} \geq \maxLengthEpisode$ almost surely 
explores every $\maxLengthEpisode$-attainable trajectory at least once 
in the limit.
\end{lemma}

Analogous to Definition~\ref{def:attainable-trajectory}, we call a label sequence $\inputTrace = \mdpLabel_0 \ldots \mdpLabel_k$ ($\maxLengthEpisode$-)attainable if there exists an ($\maxLengthEpisode$-)attainable trajectory $\mdpCommonState_0 \mdpCommonAction_0 \mdpCommonState_1 \ldots \mdpCommonState_k \mdpCommonAction_k \mdpCommonState_{k+1}$ such that $\mdpLabel_i = \rmLabelingFunction(\mdpCommonState_i, \mdpCommonAction_i, \mdpCommonState_{i + 1})$ for each $i \in \{ 0, \ldots, k \}$.
An immediate consequence of Lemma~\ref{lem:attainable-trajectories} is that \algoName\ almost surely explores every $m$-attainable label sequence in the limit.

\begin{corollary} \label{col:attainable-traces}
\algoName\ with $\mathit{eplength} \geq \maxLengthEpisode$ almost surely explores every $m$-attainable label sequence at least once in the limit.
\end{corollary}
 
If \algoName\ explores sufficiently many $\maxLengthEpisode$-attainable label sequences 
for a large enough value of $\maxLengthEpisode$, 
it is guaranteed to infer a reward machine that is ``good enough'' 
in the sense that it is equivalent to the reward machine \implementing\ the reward function $\mdpRewardFunction$ on all attainable label sequences.
This is formalized in the next lemma. 
%which is the key ingredient in proving that 
%\algoName\ learns an v optimal policy in the limit almost surely.

\begin{lemma} \label{lem:RPML-learns-correct-reward-machine} 
Let $\mdp$ be a labeled MDP and $\machine$ the reward machine \implementing\ the reward function of $\mdp$.
Then, \algoName\ with $\mathit{eplength} \geq 2^{|\mdp| + 1} (|\machine| + 1) - 1$ almost surely learns a reward machine in the limit that is equivalent to 
$\machine$ on all attainable label sequences.
%(i.e., when the number of episodes goes to infinity).
\end{lemma}

%We note that the machine inferred by \algoName\ from Lemma~\ref{lem:RPML-learns-correct-reward-machine} is not uniquely determined.
%There is a host of machines \implementing\ the reward function of the underlying MDP---
%namely, all that output the correct reward sequence for each attainable label sequence.
%Having eventually one of them as the hypothesis machine is the key ingredient in proving that \algoName\ learns an optimal policy in the limit almost surely.

Lemma~\ref{lem:RPML-learns-correct-reward-machine} guarantees that
\algoName\ will eventually learn the reward machine \implementing\ the reward function of an underlying MDP.
This is the key ingredient in proving that 
\algoName\ learns an optimal policy in the limit almost surely.

%From Lemma~\ref{lemma:everyTraceReachable} we immediately have that every possible trace (a sequence of labels) is reachable as well.
%This gives us all the necessary ingredients to prove that eventually \algoName{} will use the correct reward machine as its hypothesis, in the following lemma.
%
%\begin{lemma}
%\label{lemma:automataLearningConverges}
%Let $\mathcal M$ be the true reward machine for the MDP $(\mdpStates, \mdpInit, \mdpActions, \mdpProb, \nmrdpRewards,\mdpDiscount)$.
%Assuming that a minimal reward machine is inferred that is consistent with a received sample,  \algoName{} will eventually learn a hypothesis $\mathcal H$ such that $\mathcal{H}(\commonTrace) = \mathcal{M}(\outputTrace)$, for every trace $\commonTrace$.
%%\IG{There is an underlying assumption that the automaton learning algorithm inside \algoName{} always learns a minimal-size automaton}
%\end{lemma}
%\begin{proof}
%See Appendix.
%\end{proof}

\begin{theorem}
\label{thm:convergenceInTheLimit}
Let $\mdp$ be a labeled MDP and $\machine$ the reward machine \implementing\ the reward function of $\mdp$.
Then, \algoName\ with $\mathit{eplength} \geq 2^{|\mdp| + 1} (|\machine| + 1) - 1$ almost surely converges to an optimal policy in the limit.
\end{theorem}

% !TEX root = main.tex
%---------- optimization using equivalent states 
\section{Algorithmic Optimizations}                                        
\label{sec_opt} 
%%{\color{red}
%%TODO: Describe differences of the baseline algorithm (Algorithm~\ref{}) and the optimized version (Algorithm~\ref{alg:RPML-optimized}).
%%}
Section \ref{sec_baseline} provides the base algorithm with theoretical guarantees for convergence to an optimal policy. In this section, we present an improved algorithm (Algorithm~\ref{alg:RPML-optimized}) that includes three \textit{algorithmic optimizations}:\\
Optimization 1: batching of counterexamples (Section \ref{sec_batch});\\
Optimization 2: transfer of q-functions (Section \ref{sec_transfer});\\ 
Optimization 3: polynomial time learning algorithm for inferring the reward machines (Section \ref{sec_RPNI}).

The following theorem claims that Optimizations 1 and 2 retain the convergence guarantee of Theorem \ref{thm:convergenceInTheLimit}.
 
\begin{theorem}
	\label{thm:convergenceInTheLimit_opt}
	Let $\mdp$ be a labeled MDP and $\machine$ the reward machine \implementing\ the rewards of $\mdp$.
	Then, \algoName\ with Optimizations 1 and 2 with $\mathit{eplength} \geq 2^{|\mdp| + 1} (|\machine| + 1) - 1$ converges to an optimal policy in the limit.
\end{theorem}

It should be noted that although such
guarantee fails for Optimization 3, in practice the policies usually still converge to the optimal policies (see the case studies in Section \ref{sec_case}).

%, we present the batching of counterexamples to reduce the frequency of inferring the reward machines. In Section \ref{sec_transfer}, we present the transfer of the q-functions from the previously inferred reward machine to the newly inferred reward machine. In Section \ref{sec_RPNI}, we adopt a polynomial time learning algorithm for inferring the reward machines.

%In order for \algoName{} to perform well in practice, 
%we had to introduce a number of optimizations.
%
%First, instead of using a reward machine learning algorithm that provably returns a minimal machine consistent with the sample,\mdpLabel_
%a polynomial time algorithm that infers a consistent machine,
%with a \emph{best-effort} approach to finding a minimal one is used.
%Second, rather than learning a new machine every time a counterexample is found,
%we learn on a batched set of counterexamples mixed 
%with some encountered examples consistent with a current hypothesis machine.
%Finally, instead of initializing fresh $q$-functions for each newly learnt hypothesis machine
%(and throwing away all the knowledge gathered through exploration),
%we reuse $q$-functions associated to states of rejected hypothesis whenever possible.
%
%We discuss these three optimizations in the remainder of this section in detail.

\begin{algorithm}[t]
	\DontPrintSemicolon
	\SetKwBlock{Begin}{function}{end function}       
	{ Initialize the hypothesis reward machine $\hypothesisRM$ with a set of states $\mealyStates$} \;
	{ Initialize a set of q-functions $\setOfQFunctions = \set{\qValue^\mealyCommonState | \mealyCommonState \in \mealyStates}$ }\;
	{ Initialize $\counterexamples = \emptyset$ and $\newCounterexamples = \emptyset$ } \;
	\For{episode $n = 1, 2, \ldots$} 
	{ $(\inputTrace, \outputTrace, \setOfQFunctions) = \text{QRM\_episode}(\hypothesisRM, \setOfQFunctions)$ \;
		\If {$\hypothesisRM(\inputTrace) \neq \outputTrace$  }
		{ add $(\inputTrace, \outputTrace)$ to $\newCounterexamples$  }
		\If {$\big(\text{mod}(n, N)=0$ and $\newCounterexamples \neq\emptyset\big)$ \label{algLine:cexFound}}
		{{ $\counterexamples\gets \counterexamples\cup\newCounterexamples$ } \;
			{ infer $\hypothesisRMnew$ using $\counterexamples$ \label{algLine:infer} } \;
			%{ $\newCounterexamples\gets \emptyset$ } \;
			{ $\setOfQFunctionsNew \gets \text{Transfer}_q(\setOfQFunctions, \hypothesisRM, \hypothesisRMnew)$  } \;
			{ $\hypothesisRM\gets\hypothesisRMnew$, $\setOfQFunctions\gets\setOfQFunctionsNew$, $\newCounterexamples\gets \emptyset$} }                              
	}
	\caption{\algoName\ with algorithmic optimizations}                            
	\label{alg:RPML-optimized}   
\end{algorithm}	

\subsection{Batching of Counterexamples}
\label{sec_batch}

%QRM explores more of the environment using the reward machine to guide the search. 
%Initially, Algorithm~\ref{alg:RPML-optimized} stores each trace in the sample $\counterexamples$ until at least $m$ positive examples are in $\counterexamples$. Then, Algorithm~\ref{alg:RPML-optimized} infers a new reward machine $\hypothesisRMnew$ using $\counterexamples$ and proceeds with the QRM algorithm for $\hypothesisRM$. 
%However, once a mismatch between the observed rewards and the rewards suggested by the hypothesis is observed
Algorithm \ref{alg:RPML} infers a new hypothesis reward machine whenever a counterexample is encountered. This could incur a high computational cost. In order to adjust the frequency of inferring new reward machines, Algorithm~\ref{alg:RPML-optimized} stores each counterexample in a set $\newCounterexamples$. After each period of $N$ episodes (where $N\in\mathbb{Z}_{>0}$ is a user-defined hyperparameter), if $\newCounterexamples$ is non-empty, we add $\newCounterexamples$ to the sample $\counterexamples$ and infer a new hypothesis reward machine $\hypothesisRMnew$ (lines~\ref{algLine:cexFound} to \ref{algLine:infer}). Then, Algorithm~\ref{alg:RPML-optimized} proceeds with the QRM algorithm for $\hypothesisRMnew$. The same procedure repeats until the policy converges. 

%---------- Transfer of q-functions ----------

\subsection{Transfer of Q-functions} 
\label{sec_transfer}

In Algorithm \ref{alg:RPML}, after a new hypothesis reward machine is inferred, the q-functions are re-initialized and the experiences from the previous iteration of RL are not utilized. To utilize experiences in previous iterations, we provide a method to transfer the q-functions from the previously inferred reward machine to the newly inferred reward machine (inspired by the curriculum learning implementation in~\cite{DBLP:conf/icml/IcarteKVM18}). The transfer of q-functions is based on \textit{equivalent} states of two reward machines as defined below.  

\begin{definition}
  \label{def_equi}
For a reward machine $\machine$ and a state $\mealyCommonState\in\mealyStates$, let $\machine[\mealyCommonState]$ be the machine with $\mealyCommonState$ as the
initial state. Then, for two reward machines $\machine$ and $\hat{\machine}$, two states $\mealyCommonState\in\mealyStates$ and $\hat{\mealyCommonState}\in\hat{\mealyStates}$ are equivalent, denoted by $\mealyCommonState\sim\hat{\mealyCommonState}$, if and only if
$\machine[\mealyCommonState](\inputTrace) = \machine[\hat{\mealyCommonState}](\inputTrace)$ for all label sequences $\inputTrace$. 

\end{definition} 

With Definition \ref{def_equi}, we provide the following theorem claiming equality of optimal q-functions for equivalent states of two reward machines. We use $q^{\ast\mealyCommonState}(s,a)$ to denote the optimal q-function for state $v$ of the reward machine.

\begin{theorem}
	Let $\machine=(\mealyStates, \mealyInit, 2^\rmLabels, \mealyOutputAlphabet, \mealyTransition,  \mealyOutput)$ and $\hat{\machine}=(\hat{\mealyStates}, \mealyInitHat, 2^\rmLabels, \mealyOutputAlphabet, \hat{\mealyTransition}, \hat{\mealyOutput})$ be two reward machines \implementing\ the rewards of a labeled MDP $\mdp=(\mdpStates, \mdpInit, \mdpActions, \mdpProb, \mdpRewardFunction, \mdpDiscount, \rmLabels, \rmLabelingFunction)$. For states $\mealyCommonState\in \mealyStates$ and $\hat{\mealyCommonState}\in\hat{\mealyStates}$, if $\mealyCommonState\sim\hat{\mealyCommonState}$, then for every $s\in \mdpStates$ and $a\in \mdpActions$, $q^{\ast\mealyCommonState}(s,a)=q^{\ast\hat{\mealyCommonState}}(s,a)$.
	\label{th_q} 
\end{theorem}    

Algorithm \ref{alg:transfer} shows the procedure to transfer the q-functions between the hypothesis reward machines in consecutive iterations. For any state of the hypothesis reward machine in the current iteration, we check if there exists an equivalent state of the hypothesis reward machine in the previous iteration. If so, the corresponding q-functions are transferred (line \ref{transfer_line}). As shown in Theorem \ref{th_q}, the optimal q-functions for two equivalent states are the same. 

\begin{algorithm}
	\DontPrintSemicolon
	\SetKwBlock{Begin}{function}{end function}       
	{ \textbf{Input:} a set of q-functions $\setOfQFunctions = \set{\qValue^\mealyCommonState | \mealyCommonState \in \mealyStates}$,
		hypothesis reward machines $\hypothesisRM=(\mealyStates, \mealyInit, \mealyInputAlphabet, \mealyOutputAlphabet, \mealyTransition, \mealyOutput)$, $\hypothesisRMnew=(\mealyStatesNew, \mealyInitNew, \mealyInputAlphabet, \mealyOutputAlphabet, \mealyTransitionNew, \mealyOutputNew)$ } \;
	{ Initialize $\setOfQFunctionsNew = \set{\qValueNew^{\mealyCommonStateNew}~|~\qValueNew^{\mealyCommonStateNew}\in\mealyStatesNew}$ }\;
	\For{$\mealyCommonStateNew\in\mealyStatesNew$, $\mealyCommonState\in \mealyStates$} 
	{
		\If{$\mealyCommonState\sim\mealyCommonStateNew$} 
		{
			{ $\qValueNew^{\mealyCommonStateNew}\gets\qValue^{\mealyCommonState}$}\; \label{transfer_line}
		}
	}
	{ Return $\setOfQFunctionsNew$ }
	\caption{$\text{Transfer}_q$}                                                                                                                                                                      
	\label{alg:transfer}
\end{algorithm}

%\begin{algorithm}[t]
%	\caption{$\text{Transfer}_q$}                                                                                                                                                                      
%	\label{alg:transfer}                     
%	\begin{algorithmic}[1]
%		\State \textbf{Input:} q-function for each state $\qValue^{\mealyCommonState_0}, \ldots, \qValue^{\mealyCommonState_k}$,
%		hypothesis reward machines $\hypothesisRM=(\mealyStates, \mealyInit, \mealyInputAlphabet, \mealyOutputAlphabet, \mealyTransition, \mealyOutput)$, $\hypothesisRMnew=(\mealyStatesNew, \mealyInitNew, \mealyInputAlphabet, \mealyOutputAlphabet, \mealyTransitionNew, \mealyOutputNew)$
%		\State Initialize $\qValueNew$ 
%		\For{$\mealyCommonStateNew\in\mealyStatesNew$, $\mealyCommonState\in \mealyStates$}
%		\If{$\mealyCommonState\sim\mealyCommonStateNew$}
%		\State $\qValueNew^{\mealyCommonStateNew}\gets\qValue^{\mealyCommonState}$
%		\EndIf
%		\EndFor
%		\State \Return $\qValueNew$
%	\end{algorithmic} 
%\end{algorithm}     

\subsection{A Polynomial Time Learning Algorithm for Reward Machines}                              
\label{sec_RPNI}

In order to tackle scalability issues of the SAT-based machine learning algorithm, we propose to use a modification of the popular Regular Positive Negative Inference (RPNI) algorithm~\cite{oncina1992inferring} adapted for learning reward machines.
This algorithm, which we name \RPNIMealy, proceeds in two steps.

In the first step, \RPNIMealy\ constructs a partial, tree-like reward machine $\machine$ from a sample $\counterexamples$ where
\begin{itemize}
	\item each prefix $\mdpLabel_1 \ldots \mdpLabel_i$ of a trace $(\mdpLabel_1 \ldots \mdpLabel_k, \mdpRewards_1 \ldots \mdpRewards_k) \in \counterexamples$ $(i \leq k)$ corresponds to a unique state $\mealyCommonState_{\mdpLabel_1 \ldots \mdpLabel_i}$ of $\machine$; and
	\item for each trace $(\mdpLabel_1 \ldots \mdpLabel_k, \mdpRewards_1 \ldots \mdpRewards_k) \in \counterexamples$ and $i \in \{ 0, \ldots, k-1 \}$, a transition leads from state $\mealyCommonState_{\mdpLabel_1 \ldots \mdpLabel_i}$ to state $\mealyCommonState_{\mdpLabel_1 \ldots \mdpLabel_{i+1}}$ with input $\mdpLabel_{i+1}$ and output $\mdpRewards_{i+1}$.
\end{itemize}
Note that $\machine$ fits the sample $\counterexamples$ perfectly in that $\machine(\inputTrace) = \outputTrace$ for each $(\inputTrace, \outputTrace) \in \counterexamples$ and the output of all other inputs is undefined (since the reward machine is partial).
In particular, this means that $\machine$ is consistent with $\counterexamples$.

In the second step, \RPNIMealy\ successively tries to merge the states of $\machine$.
The overall goal is to construct a reward machine with fewer states but more input-output behaviors.
For every candidate merge (which might trigger additional state merges to restore determinism), the algorithm checks whether the resulting machine is still consistent with $\counterexamples$.
Should the current merge result in an inconsistent reward machine, it is reverted and \RPNIMealy\ proceeds with the next candidate merge; otherwise, \RPNIMealy\ keeps the current merge and proceeds with the merged reward machine. This procedure stops if no more states can be merged.
Once this is the case, any missing transition is directed to a sink state, where the output is fixed but arbitrary.

Since \RPNIMealy\ starts with a consistent reward machine and keeps intermediate results only if they remain consistent, its final output is clearly consistent as well.
Moreover, merging of states increases the input-output behaviors, hence generalizing from the (finite) sample.
Finally, let us note that the overall runtime of \RPNIMealy\ is polynomial in the number of symbols in the given sample because the size of the initial reward machine $\machine$ corresponds to the number of symbols in the sample $\counterexamples$ and each operation of \RPNIMealy\ can be performed in polynomial time.

% !TEX root = main.tex

\section{Case Studies}
\label{sec_case}
%The Case Study 1 simulates an office environment in a grid world and uses increasingly complex task specifications, whereas the Case Study 2 is conducted in a larger grid world, simulating Minecraft environment.
\begin{figure*}[t]
	\centering
	\begin{subfigure}[b]{0.3\textwidth}
		\centering
		\includegraphics[width=\textwidth]{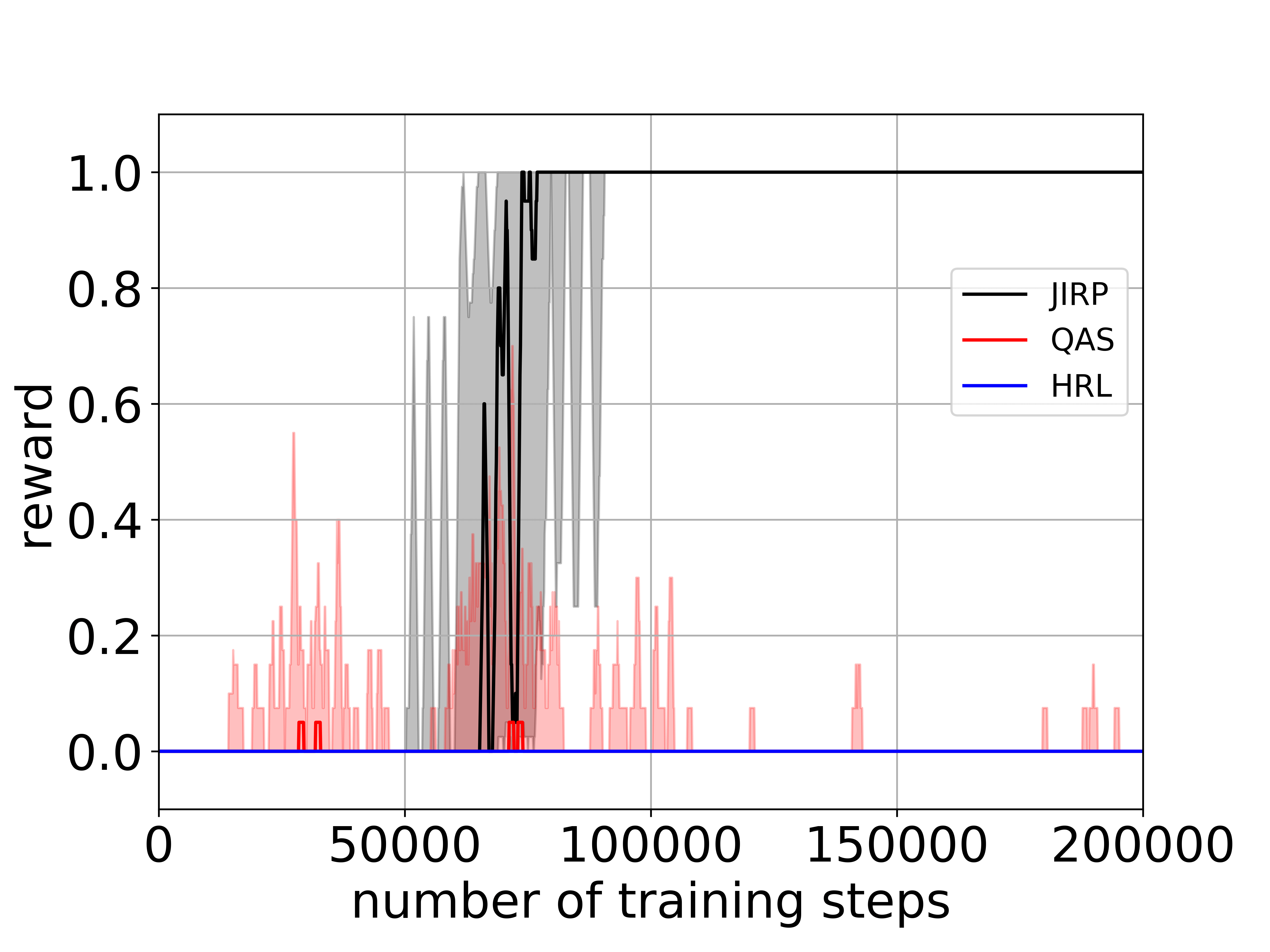}
		\caption{}
	\end{subfigure}
	\hfill
	\begin{subfigure}[b]{0.3\textwidth}
		\centering
		\includegraphics[width=\textwidth]{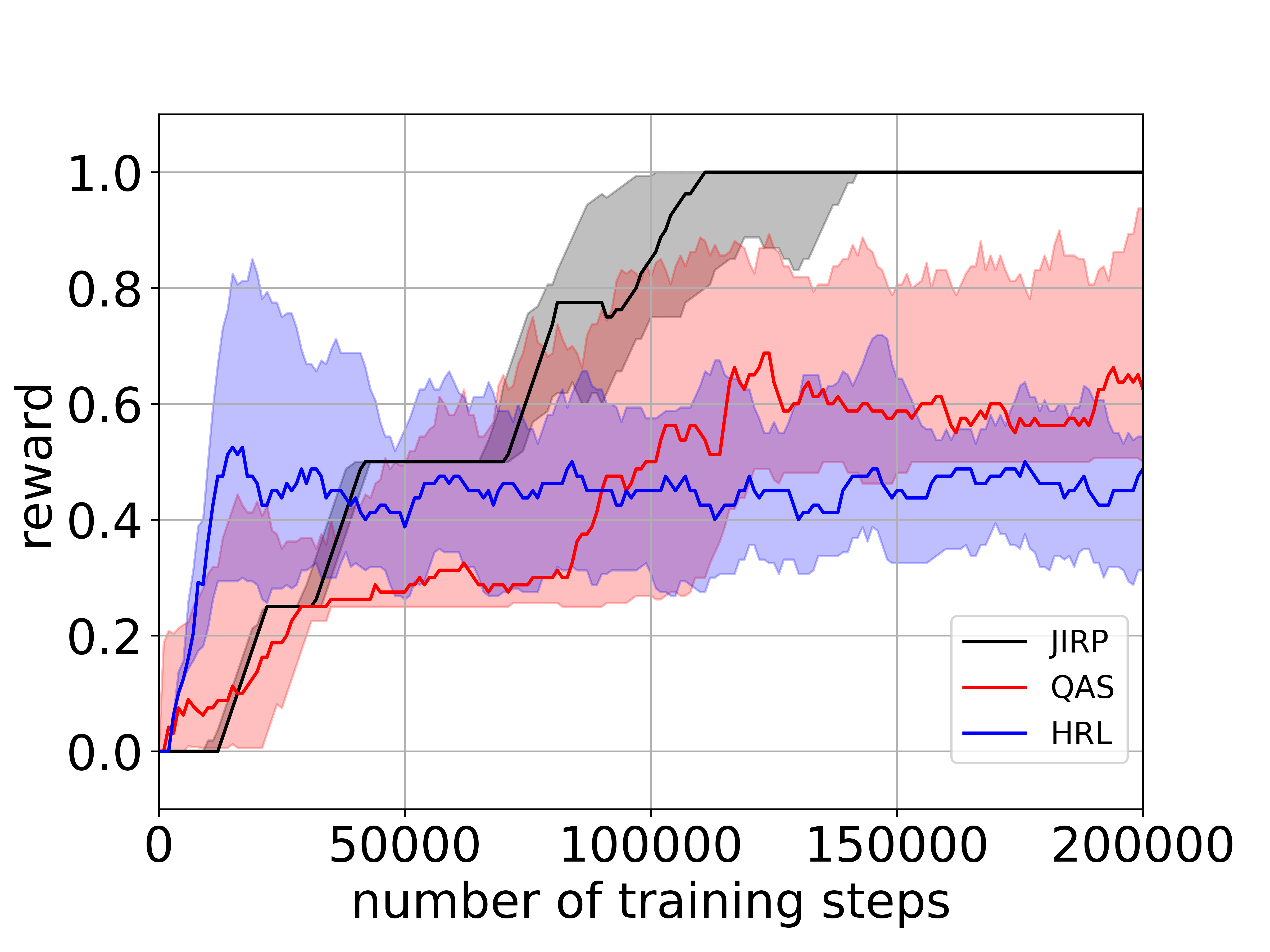}
		\caption{}
	\end{subfigure}
	\hfill
\begin{subfigure}[b]{0.3\textwidth}
	\centering
	\includegraphics[width=\textwidth]{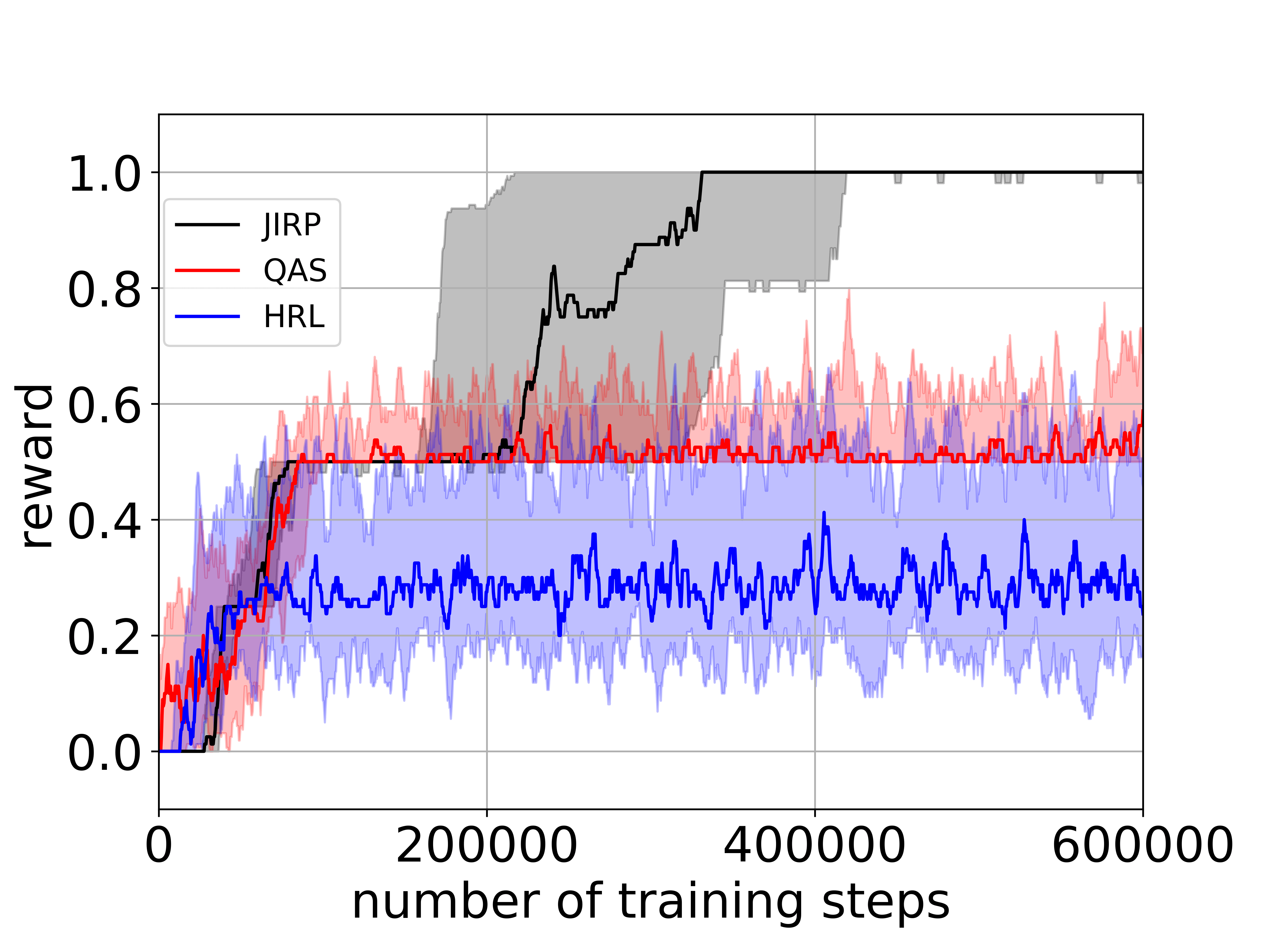}
	\caption{}
\end{subfigure}
	\caption{Cumulative rewards of 10 independent simulation runs averaged for every 10 training steps in (a) \traffic;  (b) \office\ (averaged for four tasks); and (c) \craft\ (averaged for four tasks).}  
	\label{results_all}
\end{figure*}

%\subsection{Implementation}
%%The Algorithm~\ref{alg:RLWithAutomataLearning} was implemented using a number of optimizations and simplifying assumptions.
%%First, we notice that all the tasks from both case studies can be described by a special subclass of reward machines.
%%While reward machines in general correspond to Mealy machines, the subclass containing all the tasks correspond to deterministic finite automata.
%%As a rule, algorithms for learning deterministic finite automata are much faster that their Mealy machines counterparts. 

In this section, we apply the proposed approach to three different scenarios: 
1) autonomous vehicle scenario;
2) office world scenario adapted from~\cite{DBLP:conf/icml/IcarteKVM18}, 
and 
3) Minecraft world scenario adapted from~\cite{andreas2017modular}. We use the libalf~\cite{libalfTool} implementation of RPNI~\cite{oncina1992inferring} as the algorithm to infer reward machines. 
The detailed description of the tasks in the three different scenarios can be found in the supplementary material.

%The three scenarios are characterized by sparse reward functions, 
%corresponding to a subclass of reward machines.
%Such reward machines, firstly, contain a specially marked subset of \emph{final states}.
%Secondly, the output alphabet contains only two values (e.g., 0 and 1), 
%and the output function maps a transition to 1 if and only if the end-state is a final state. 
%Finally, the outgoing transitions from final states all end in a sink state.
%Such reward machines correspond to deterministic finite automata.

%Additionally, the outgoing transitions from final states all end in a sink state.

%An interesting (and practically relevant) subclass of reward machines are given by deterministic finite automata.
%This special case corresponds to a Mealy machine with a specially marked subset of \emph{final states}, 
%the output alphabet $\set{0, 1}$, and the output function mapping a transition to 1 iff the end-state is a final state. 
%Additionally, the outgoing transitions from final states all end in a sink state.

%The detailed results for each of the three scenarios can be found in the supplementary material.  
%In comparison with the RPML approach, we perform RL with the following two baseline methods:\\ 

We compare \algoName{}\ (with algorithmic optimizations) with the two following baseline methods:
\begin{itemize}
	\item \methodB\ (q-learning in augmented state space): to incorporate the extra information of the labels (i.e., high-level events in the environment), we perform q-learning \cite{Watkins1992} in an augmented state space with an extra binary vector representing whether each label has been encountered or not.
	\item \methodC\ (hierarchical reinforcement learning): we use a meta-controller for deciding the subtasks (represented by encountering each label) and use the low-level controllers expressed by neural networks \cite{Kulkarni2016hierarchical} for deciding the actions at each state for each subtask.
\end{itemize}

%As described in Preliminaries section, the characteristics of RPNI are that it scales better than other algorithms, but only does ``best-effort'' attempts of minimizing the learnt automato
%Nonetheless, as our experiments suggest, it does a good job learning a correct automaton after a few re-learning iterations.

%we do not start the re-learning procedure each time a counterexample is discovered.
%Instead, we do it at earliest 20 episodes after the last re-learning step.
%By doing so, we are trading-off informed exploration for a reduced number of invocations of automata-learning procedure.
%When creating a correcting sample, we keep all the counterexamples noticed thus far, and sample 5 traces from $\allExperiences$, a set of all agent's experiences. 

\subsection{\Traffic}
%Initially, we generate a sample of at least $m=10$ positive traces and all other negative traces. Afterwards, in every $N=20$ episodes we check 
%whether there exists counterexamples. If so, we add them in the sample and infers the reward machine again. Section \ref{sec_example}
We consider the autonomous vehicle scenario as introduced in the motivating example in Section \ref{sec_example}. The set of actions is $\mdpActions = \set{Straight, Left, Right, Stay}$, corresponding to going straight, turning left, turning right and staying in place. For simplicity, we assume that the labeled MDP is deterministic (i.e, the slip rate is zero for each action). The vehicle will make a U-turn if it reaches the end of any road.

%The transition probability $\mdpProb$ is defined by 
%\[ 
%\mdpProb(s, a, s') = 
%\begin{cases}
%	0 &\neg (s \xrightarrow{a} s')\\
%	0.7 &s \xrightarrow{a} s' \land (s'.x, s'.y) \in \mathcal{J} \land s'.\mathit{tr} = \top \\
%	0.3 & s \xrightarrow{a} s' \land (s'.x, s'.y) \in \mathcal{J} \land s'.\mathit{tr} = \bot \\ 
%	0.5 &s \xrightarrow{a} s' \land (s'.x, s'.y) \not\in \mathcal{J}\\
%\end{cases}
%\]
%where $s \xrightarrow{a} s'$ defines possible transitions.                                                     
%More precisely,  $s \xrightarrow{a} s'\equiv (s.x, s.y) + a = (s'.x, s'.y) \land s'.\mathit{pStart} \Leftrightarrow (s'.x, s'.y) \in \mathit{PriorityStart} \land s'.\mathit{pEnd} \Leftrightarrow (s'.x, s'.y) \in \mathit{PriorityEnd}$, where \textit{PriorityStart} and \textit{PriorityEnd} are sets of locations containing signs for start and end of priority road.
% The set of locations $\mathcal J$ represents an area for which there is expected increased traffic.
% Note that the randomness comes from the the actions of other cars, so the same action can bring the care into a state with traffic and the one without traffic.

The set of labels is $\set{\mathit{sp}, \mathit{pr}, \textrm{B}}$ and the labeling function $\rmLabelingFunction$ is defined by                                     
\begin{align*}
\mathit{sp} \in \rmLabelingFunction(s,a,s') &\Leftrightarrow a = stay \wedge s\in\mathcal{J}, \\ 
\mathit{pr} \in \rmLabelingFunction(s,a,s') &\Leftrightarrow s'.priority=\top\wedge s\in\mathcal{J}, \\
\textrm{B} \in \rmLabelingFunction(s,a,s') &\Leftrightarrow s'.x = x_{\textrm{B}} \land s'.y = y_{\textrm{B}},
\end{align*}
where $s'.priority$ is a Boolean variable that is true ($\top$) if and only if $s'$ is on the priority roads, $\mathcal J$ represents the set of locations where the vehicle is entering an intersection, $s'.x$ and $s'.y$ are the $x$ and $y$ coordinate values at state $s$, and $x_{\textrm{B}}$ and $y_{\textrm{B}}$ are $x$ and $y$ coordinate values at B (see Figure \ref{fig:ex:road-map}). 

We set $eplength=100$ and $N=100$. Figure \ref{fig:ex:inferred-reward-machine} shows the inferred hypothesis reward machine in the last iteration of \algoName{}\ in one typical run. The inferred hypothesis reward machine is different from the true reward machine in Figure~\ref{fig:ex:reward-machine}, but it can be shown that these two reward machines are equivalent on all attainable label sequences.

%For example, at $\mealyCommonState_0$ (intuitively representing the vehicle being on the priority roads), if the vehicle stops at an intersection, the state gets to $\mealyCommonState_1$ and can never return to $\mealyCommonState_0$ as the vehicle on a priority road cannot stop and end in a normal road ($sp\wedge\mathit{pr}$); and at $\mealyCommonState_1$, if the vehicle does not stop at an intersection and the state gets to $\mealyCommonState_2$, then the state also can never return to $\mealyCommonState_0$ again. 
\vspace{-8mm}
\begin{figure}[th]
	\centering
	\begin{tikzpicture}[thick,scale=1, every node/.style={transform shape}]
	\node[state,initial] (0) at (-1.5, 0) {$\mealyCommonState_0$};
	\node[state] (1) at (2, 0) {$\mealyCommonState_1$};
	\node[state] (3) at (4.5, 1) {$\mealyCommonState_3$};
	\node[state, accepting] (2) at (-2, 1.5) {$\mealyCommonState_2$};
	\draw[<-, shorten <=1pt] (0.west) -- +(-.4, 0);
	\draw[->] (0) to[left] node[above, align=center] {$(\lnot sp \land \lnot\mathit{pr}, 0)$} (1);
	% \node[] at  (0.8,0.1) {$Low$};
	\draw[->] (0) to[loop below] node[align=center] {$(\lnot sp \land \mathit{pr}, 0)$} ();
	\draw[->] (0) to[bend left=65] node[sloped, below, align=center] {$(sp\wedge pr, 0)$} (1);
	\draw[->] (0) to[bend left=70] node[sloped, below, align=center] {$(sp\wedge \lnot pr, 0)$} (3);
	% \draw[->] (0) to[right] node[sloped, below, align=center] {$(\lnot sp\wedge \lnot pr, 0)$} (1);
	\draw[->] (1) to[bend left=65] node[above, align=center] {$(sp \land \lnot\mathit{pr}, 0)$} (0);
	\draw[->] (1) to[bend right=65] node[sloped, below, align=center] {$(\lnot sp\wedge \lnot pr, 0)$} (3);
	\draw[->] (1) to[loop below] node[align=center] {$(\mathit{pr}\vee \textrm{B}, 0)$} ();
	\draw[->] (3) to[right] node[sloped, above, align=center] {$(\lnot sp\wedge pr, 0)$} (1);
	\draw[->] (3) to[loop above] node[align=center] {$(\lnot\mathit{pr}, 0)$} ();
	\draw[->] (0) to[bend left] node[left, align=center] {$(\textrm{B}, 1)$} (2);
	% \draw[->] (1) to[bend left] node[below, align=center] {$(B, 1)$} (4);
	% \draw[->] (2) to[bend left] node[below, align=center] {$(B, 1)$} (4);
	\end{tikzpicture}	
	\caption{Inferred hypothesis reward machine in the last iteration of \algoName{}\ in one typical run in the \traffic. } 
	\label{fig:ex:inferred-reward-machine}
\end{figure}
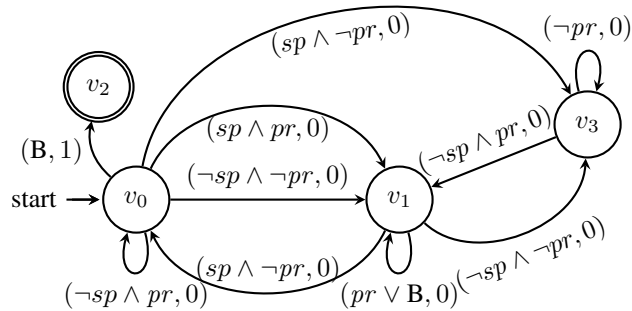

Figure \ref{results_all} (a) shows the cumulative rewards with the three different methods in the \traffic. The \algoName\ approach converges to optimal policies within 100,000 training steps, while \methodB\ and \methodC\ are stuck with near-zero cumulative reward for up to two million training steps (with the first 200,000 training steps shown in Figure \ref{results_all} (a)). 

\subsection{\Office}

We consider the \office\ in the 9$\times$12 grid-world. The agent has four possible actions at each time step: move north, move south, move east and move west. After each action, the robot may slip to each of the two adjacent cells with probability of 0.05. We
use four tasks with different high-level structural relationships among subtasks such as getting the coffee, getting mails and going to the office (see Appendix~\ref{sec:app_office} for details).

%\begin{table}[h]
%	\centering
%	\caption{Average rewards for 4 tasks and 3 methods in Officeworld Scenario.}
%	\begin{tabular}{lllll}
%		\toprule
%		& \OfficeA & \OfficeB & \OfficeC & \OfficeD \\
%		\midrule
%		Method I          & 0.93 & 0.79 & 0.51 & 0.43 \\ 
%		Method II         & 0.39 & 0.10 & 0.10 & 0.23 \\ 
%		Method III        & 0.963 & 0.309 & 0.302 & 0.155  \\
%		\bottomrule
%	\end{tabular}
%	\label{reward_table}  
%\end{table}    

We set $eplength=1000$ and $N=30$. Figure \ref{results_all} (b) shows the cumulative rewards with the three different methods in the \office. The \algoName\ approach converges to the optimal policy within 150,000 training steps, while \methodB\ and \methodC\ reach only 60\% of the optimal average cumulative reward within 200,000 training steps.

%For \officeA, all the three methods converge to the optimal performance within 200000 training steps. For  \officeB, \officeC\ and \officeD, only \methodA\ converges to the optimal performance within 200000 training steps. 

\subsection{\Craft}

We consider the Minecraft
example in a 21$\times$21 gridworld. The four actions and the slip rates are the same as in the \office.  We
use four tasks including making plank, making stick, making bow and making bridge (see Appendix~\ref{sec:app_craft} for details).

%\begin{table}[h]
%	\centering
%	\caption{Averaged rewards for 4 tasks and 3 methods in \Craft.}
%	\begin{tabular}{lllll}
%		\toprule
%		& \CraftA & \CraftB & \CraftC & \CraftD \\
%		\midrule
%		Method I          & 0.64 & 0.60 & 0.71 & 0.73 \\ 
%		Method II         & 0.15 & 0.15 & 0.20 & 0.21 \\ 
%		Method III         & 0.025 & 0.015  & 0 & 0.043 \\
%		\bottomrule
%	\end{tabular}
%	\label{reward_table2}  
%\end{table}       

We set $eplength=400$ and $N=30$. Figure \ref{results_all} (c) shows the cumulative rewards with the three different methods in the \craft. 
The \algoName\ approach converges to the optimal policy within 400,000 training steps, while \methodB\ and \methodC\ reach only 50\% of the optimal average cumulative reward within 600,000 training steps.

%For \craftA, all the three methods converge to the optimal performance within 200000 training steps. For \officeB, \officeC~ and \officeD, only \methodA\ converges to the optimal performance within 200000 training steps. 

% !TEX root = main.tex

\section{Conclusion}
We proposed an iterative  approach that alternates between reward machine inference and reinforcement learning (RL) for the inferred reward machine. We have shown the improvement of RL performances using the proposed method.

This work opens the door for utilizing automata learning in RL. First, the same methodology can be applied to other forms of RL, such as model-based RL, or actor-critic methods. Second, we will explore methods that can infer the reward machines incrementally (based on inferred reward machines in the previous iteration). Finally, the method to transfer the q-functions between equivalent states of reward machines can be also used for transfer learning between different tasks where the reward functions are encoded by reward machines.

%
%}{}

%---------- Bibliography ----------
\bibliographystyle{IEEEtran}
\bibliography{references}

	\maketitle
%\iftoggle{long}{  
%\title{Joint Inference of Reward Machines and Policies for Reinforcement Learning}
%
%\begin{document}
%	
%\maketitle

\appendix

%\onecolumn[ajkdkl]

%\twocolumn

%\begin{strip}
%	\begin{gather*}\nonumber
%	\centering
%	\begin{split}
%	\scalebox{1.3}{\textbf{\textrm{Joint Inference of Reward Machines and Policies for Reinforcement Learning:}}}\\
%    \scalebox{1.3}{\textbf{\textrm{Supplementary Materials}}}
%	\end{split}
%	\end{gather*}
%\end{strip}

%\begin{strip}
%\begin{align*}\nonumber
%\centering
%\begin{split}
%\scalebox{1.5}{\textbf{\textrm{Joint Inference of Reward}}}&~ \scalebox{1.5}{\textbf{\textrm{Machines and Policies for Reinforcement Learning:}}}\\
%&\scalebox{1.5}{\textbf{\textrm{Supplementary Materials}}}
%\end{split}
%\end{align*}
%\end{strip}

%\vspace{2in}
%
%\hfill \break

\section{Proof of Lemma 1 }

\begin{proof}
We first prove that \algoName\ with $\mathit{eplength} \geq \maxLengthEpisode$ explores every $\maxLengthEpisode$-attainable trajectory with a positive (non-zero) probability.
We show this claim by induction over the length $i$ of trajectories.
\begin{description}
	\item[Base case:]
	The only trajectory of length $i = 0$, $\mdpCommonState_\init$, is always explored because it is the initial state of every exploration.
	
	\item[Induction step:]
	Let $i = i'+ 1$ and $\mdpTrajectory = \mdpCommonState_0 \mdpCommonAction_0 \mdpCommonState_1 \ldots \mdpCommonState_{i'} \mdpCommonAction_{i'} \mdpCommonState_{i}$ be an $\maxLengthEpisode$-attainable trajectory of length $i \leq m$.
	Then, the induction hypothesis yields that \algoName{} explores each $\maxLengthEpisode$-attainable trajectory $\mdpCommonState_0 \mdpCommonAction_0 \mdpCommonState_1 \ldots s_{i'}$ (of length $i' = i-1$).
	Moreover \algoName\ continues its exploration because $\mathit{eplength} \geq \maxLengthEpisode > i'$.
	At this point, every action $\mdpCommonAction_{i'}$ will be chosen with probability at least $\epsilon\times \frac{1}{|\mdpActions_{\mdpCommonState_{i'}}|}$, where $\mdpActions_{\mdpCommonState_{i'}} \subseteq \mdpActions$ denotes the set of available actions in the state $\mdpCommonState_{i'}$ (this lower bound is due to the $\epsilon$-greedy strategy used in the exploration).
	Having chosen action $\mdpCommonAction_{i'}$, the state $\mdpCommonState_{i}$ is reached with probability $\mdpProb(\mdpCommonState_{i'}, \mdpCommonAction_{i'}, \mdpCommonState_{i}) > 0$ because $\mdpTrajectory$ is $\maxLengthEpisode$-attainable.
	Thus, the trajectory $\mdpTrajectory$ is explored with a positive probability. 
\end{description}

Since \algoName\ with $\mathit{eplength} \geq \maxLengthEpisode$ explores every $\maxLengthEpisode$-attainable trajectory with a positive probability, the probability of an $\maxLengthEpisode$-attainable trajectory not being explored becomes $0$ in the limit (i.e., when the number of episodes goes to infinity).
Thus, \algoName\ almost surely (i.e., with probability $1$) explores every $\maxLengthEpisode$-attainable trajectory in the limit.
\end{proof}

%---------- Proof of Lemma 2 ----------
\section{Proof of Lemma 2}

In order to prove Lemma~\ref{lem:RPML-learns-correct-reward-machine}, we require a few (basic) definitions from automata and formal language theory.% , specifically those of deterministic finite automata.

An \emph{alphabet} $\dfaInputAlphabet$ is a nonempty, finite set of \emph{symbols} $\dfaSymbol \in \dfaInputAlphabet$.
A \emph{word} $\word = b_0 \ldots b_n$ is a finite sequence of symbols.
The empty sequence is called \emph{empty word} and denoted by $\varepsilon$.
The \emph{length} of a word $\omega$, denoted by $|\word|$ is the number of its symbols.
We denote the set of all words over the alphabet $\dfaInputAlphabet$ by $\dfaInputAlphabet^\ast$.

Next,  we recapitulate the definition of deterministic finite automata.

\begin{definition}%[Nondeterministic finite automata]
A \emph{deterministic finite automaton (DFA)} is a five-tuple $\dfa = (\dfaStates, \dfaState_\init, \dfaInputAlphabet, \dfaTransition, \dfaFinalStates)$ consisting of a nonempty, finite set $\dfaStates$ of states, an initial state $\dfaState_\init \in \dfaStates$, an input alphabet $\dfaInputAlphabet$, a transition function $\dfaTransition \colon \dfaStates \times \dfaInputAlphabet \to \dfaStates$, and a set $\dfaFinalStates \subseteq \dfaStates$ of final states.
The size of a DFA, denoted by $|\dfa|$, is the number $|\dfaStates|$ of its states.
\end{definition}

%\hfill \break
%\hfill \break

A \emph{run} of a DFA $\dfa = (\dfaStates, \dfaState_\init, \dfaInputAlphabet, \dfaTransition, \dfaFinalStates)$ on an input word $\word = \dfaSymbol_0 \ldots \dfaSymbol_k$ is a sequence $\dfaState_0 \ldots \dfaState_{k+1}$ of states such that $\dfaState_0 = \dfaState_\init$ and $\dfaState_{i+1} = \dfaTransition(\dfaState_i, \dfaSymbol_i)$ for each $i \in \{ 0, \ldots, k \}$.
A run $\dfaState_0 \ldots \dfaState_{k+1}$ of $\dfa$ on a word $\word$ is \emph{accepting} if $\dfaState_{k+1} \in \dfaFinalStates$, and $\word$ is \emph{accepted} if there exists an accepting run.
The \emph{language} of a DFA $\dfa$ is the set $L(\dfa) = \{ \word \in \dfaInputAlphabet^\ast \mid \text{$\dfa$ accepts $\word$} \}$.
As usual, we call two DFAs $\dfa_1$ and $\dfa_2$ \emph{equivalent} if $L(\dfa_1) = L(\dfa_2)$.
Moreover, let us recapitulate the well-known fact that two non-equivalent DFAs have a ``short'' word that witnesses their non-equivalence.

\begin{theorem}[\cite{DBLP:books/daglib/0025557}, Theorem 3.10.5] \label{thm:dfa-symmetric-difference}
Let $\dfa_1$ and $\dfa_2$ be two DFAs with $L(\dfa_1) \neq L(\dfa_2)$.
Then, there exists a word $\word$ of length at most $|\dfa_1| + |\dfa_2| - 1$ such that $\word \in L(\dfa_1)$ if and only if $\word \notin L(\dfa_2)$.
\end{theorem}

As the next step towards the proof of Lemma~\ref{lem:RPML-learns-correct-reward-machine}, we remark that every reward machine over the input alphabet $\rmInputAlphabet$ and output alphabet $\rmOutputAlphabet$ can be translated into an ``equivalent'' DFA as defined below.
This DFA operates over the combined alphabet $\rmInputAlphabet \times \rmOutputAlphabet$ and accepts a word $(\mdpLabel_0, \mdpRewards_0) \ldots (\mdpLabel_k, \mdpRewards_k)$ if and only if $\machine$ outputs the reward sequence $\mdpRewards_0 \ldots \mdpRewards_k$ on reading the label sequence $\mdpLabel_0 \ldots \mdpLabel_k$.

\begin{lemma} \label{lem:Mealy-to-DFA}
Given a reward machine $\machine = (\mealyStates, \mealyInit, \mealyInputAlphabet, \mealyOutputAlphabet, \mealyTransition, \mealyOutput)$, one can construct a DFA $\dfa_\machine$ with $|\machine| + 1$ states such that
\begin{equation}
L(\dfa_\machine) = \bigl \{ (\mdpLabel_0, \mdpRewards_0) \ldots (\mdpLabel_k, \mdpRewards_k) \in (\rmInputAlphabet \times \rmOutputAlphabet)^\ast \mid \machine(\mdpLabel_0 \ldots \mdpLabel_k) = \mdpRewards_0 \ldots \mdpRewards_k \bigr \}.
\end{equation}
\end{lemma}

\begin{proof}[Proof of Lemma~\ref{lem:Mealy-to-DFA}]
Let $\machine = (\mealyStates_\machine, \mealyCommonState_{\init, \machine}, \rmInputAlphabet, \rmOutputAlphabet, \mealyTransition_\machine, \mealyOutput_\machine)$ be a reward machine.
Then, we define a DFA $\dfa_\machine = (\dfaStates, \dfaState_\init, \dfaInputAlphabet, \dfaTransition, \dfaFinalStates)$ over the combined alphabet $\rmInputAlphabet \times \rmOutputAlphabet$ by
\begin{itemize}
	\item $\dfaStates = \mealyStates_\machine \cup \{ \bot \}$ with $\bot \notin \mealyStates_\machine$;
	\item $\dfaState_\init = \mealyCommonState_{\init, \machine}$;
	\item $\Sigma = \mealyInputAlphabet \times \mealyOutputAlphabet$;
	\item $\dfaTransition \bigl( \mealyCommonState, (\mdpLabel, \mdpRewards) \bigr) = \begin{cases} \mealyCommonState' & \text{if $\mealyTransition_\machine(\mealyCommonState, \mdpLabel) = \mealyCommonState'$ and $\mealyOutput_\machine(\mealyCommonState, \mdpLabel) = \mdpRewards$;} \\ \bot & \text{otherwise;} \end{cases}$
	\item $\dfaFinalStates = \mealyStates_\machine$.
\end{itemize}

In this definition, $\bot$ is a new sink state to which $\dfa_\machine$ moves if its input does not correspond to a valid input-output pair produced by $\machine$.
A straightforward induction over the length of inputs to $\dfa_\machine$ shows that it indeed accepts the desired language.
In total, $\dfa_\machine$ has $|\machine| + 1$ states.
\end{proof}

%Given a labeled MPD $\mpd$, one can construct a DFA accepting all attainable traces in a straightforward manner.
%More precisely, one first constructs a nondeterministic reward machine $\dfa$ that has the same set of states as $\mpd$ and the set
%\[ \bigl\{ () \mid \bigr \} \]
%of transitions; all states are final states

Similarly, one can construct a DFA $\dfa_\mdp$ that accepts exactly the attainable traces of an MDP $\mdp$.
First, viewing labels $\rmLabelingFunction(\mdpCommonState, \mdpCommonAction, \mdpCommonState')$ as input symbols, marking every state as an accepting state, and keeping only those transitions for which $p(\mdpCommonState, \mdpCommonAction, \mdpCommonState') > 0$, $\mdp$ can be viewed as a non-deterministic finite automaton. 
Second, using the standard determinization algorithm~\cite{DBLP:books/daglib/0025557}, one can create an equivalent DFA with an exponential blowup in the number of states.

\begin{remark} \label{rem:MDP-to-DFA}
Given a labeled MDP $\mdp$, one can construct a DFA $\dfa_\mdp$ with at most $2^{|\mdp|}$ states that accepts exactly the admissible label sequences of $\mdp$.
\end{remark}

Next, we show that if two reward machines disagree on an attainable label sequence, then we can provide a bound on the length of such a sequence.

%\iftoggle{long}{  

\begin{lemma} \label{lem:reward-machine-difference}
Let $\mdp = (\mdpStates, \mdpInit, \mdpActions, \mdpProb, \mdpRewardFunction, \mdpDiscount, \rmLabels, \rmLabelingFunction)$ be a labeled MDP and $\machine_1, \machine_2$ two reward machines with input alphabet $\rmInputAlphabet$.
If there exists an attainable label sequence $\inputTrace$ such that $\machine_1(\inputTrace) \neq \machine_2(\inputTrace)$, then there also exists an $\maxLengthEpisode$-attainable label sequence $\inputTrace^\star$ with $\maxLengthEpisode \leq 2^{|\mdp|} (|\machine_1| + |\machine_2| + 2) - 1$ such that $\machine_1(\inputTrace^\star) \neq \machine_2(\inputTrace^\star)$.
\end{lemma}

\begin{proof}[Proof of Lemma~\ref{lem:reward-machine-difference}]

Let $\mdp = (\mdpStates, \mdpInit, \mdpActions, \mdpProb, \mdpRewardFunction, \mdpDiscount, \rmLabels, \rmLabelingFunction)$ be a labeled MDP and $\machine_1$, $\machine_2$ two reward machines with input alphabet $\rmInputAlphabet$.
As a first step, we construct the DFAs $\dfa_\mdp = (\dfaStates', \dfaState'_\init, \rmInputAlphabet, \dfaTransition', \dfaFinalStates')$ according to Remark~\ref{rem:MDP-to-DFA} and the DFAs $\dfa_{\machine_i} = (\dfaStates''_i, \dfaState''_{\init, i}, \rmInputAlphabet \times \rmOutputAlphabet, \dfaTransition''_i, \dfaFinalStates''_i)$ for $i \in \{ 1, 2 \}$ according to Lemma~\ref{lem:Mealy-to-DFA}.

Next, we construct the input-synchronized product $\dfa_{\mdp \times \machine_i} = (\dfaStates''', \dfaState'''_\init, \rmInputAlphabet, \dfaTransition''', \dfaFinalStates''')$ of a $\dfa_\mdp$ and $\dfa_{\machine_i}$ by
\begin{itemize}
	\item $\dfaStates''' = \dfaStates' \times \dfaStates''_i$;
	\item $\dfaState'''_I = (\dfaState'_\init, \dfaState''_{\init, i})$;
	\item $\dfaTransition''' \bigl( (\dfaState', \dfaState''_i), (\mdpLabel, \mdpRewards) \bigr) = \bigl( \dfaTransition'(\dfaState', \mdpLabel), \dfaTransition''_i(\dfaState''_i, (\mdpLabel, \mdpRewards)) \bigr)$; and
	\item $\dfaFinalStates''' = \dfaFinalStates' \times \dfaFinalStates''_i$,
\end{itemize}
which synchronizes $\dfa_\mdp$ and the input-component of $\dfa_{\machine_i}$.
A straightforward induction over the lengths of inputs to $\dfa_{\mdp \times \machine_i}$ shows that $(\mdpLabel_0, \mdpRewards_0) \ldots (\mdpLabel_k, \mdpRewards_k) \in L(\dfa_{\mdp \times \machine_i})$ if and only if $\mdpLabel_0 \ldots \mdpLabel_k$ is an attainable label sequence such that $\machine_i(\mdpLabel_0 \ldots \mdpLabel_k) = \mdpRewards_0 \ldots \mdpRewards_k$.
Moreover, note that $\dfa_{\mdp \times \machine_i}$ has $2^{|\mdp|} (|\machine_i| + 1)$ states.

If there exists an attainable label sequence $\inputTrace$ such that $\machine_1(\inputTrace) \neq \machine_2(\inputTrace)$, then $L(\dfa_{\mdp \times \machine_1}) \neq L(\dfa_{\mdp \times \machine_2})$ by construction of the DFAs $\dfa_{\mdp \times \machine_1}$ and $\dfa_{\mdp \times \machine_2}$.
In this situation, Theorem~\ref{thm:dfa-symmetric-difference} guarantees the existence of a word $\word = (\mdpLabel_0, \mdpRewards_0) \ldots (\mdpLabel_{\maxLengthEpisode-1}, \mdpRewards_{\maxLengthEpisode-1}) \in (\rmInputAlphabet \times \rmOutputAlphabet)^\ast$ of size
\begin{align*}
	\maxLengthEpisode & \leq 2^{|\mdp|} (|\machine_1| + 1) + 2^{|\mdp|}  (|\machine_2| + 1) - 1 \\
	& = 2^{|\mdp|}  (|\machine_1| + |\machine_2| + 2) - 1
\end{align*}
such that $\word \in L(\dfa_{\mdp \times \machine_1})$ if and only if $\word \notin L(\dfa_{\mdp \times \machine_2})$.

Let now $\inputTrace^\star = \mdpLabel_0 \ldots \mdpLabel_{\maxLengthEpisode-1}$.
By construction of the DFAs $\dfa_{\mdp \times \machine_1}$ and $\dfa_{\mdp \times \machine_2}$, we know that $\machine_1(\inputTrace^\star) \neq \machine_2(\inputTrace^\star)$ holds.
Moreover, $\inputTrace^\star$ is an $\maxLengthEpisode$-attainable label sequence with the desired bound on $\maxLengthEpisode$.
\end{proof}

We are now ready to prove Lemma~\ref{lem:RPML-learns-correct-reward-machine}.

\begin{proof}[Proof of Lemma~\ref{lem:RPML-learns-correct-reward-machine}]
Let $X_0, X_1, \ldots$ be the sequence of samples that arise in the run of \algoName\ whenever new counterexamples are added to $X$ (in Line~\ref{alg:RPML:addingToCounterexamples} of Algorithm \ref{alg:RPML}).
We now make two observations about this sequence, which help us prove Lemma~\ref{lem:RPML-learns-correct-reward-machine}.
\begin{enumerate}
	\item \label{obs:sequence-of-samples:1} The sequence $X_0, X_1, \ldots$ grows strictly monotonically (i.e., $X_0 \subsetneq X_1 \subsetneq \cdots$).
	The reasons for this are twofold.
	First, \algoName\ always adds counterexamples to $X$ and never removes them (which establishes $X_0 \subseteq X_1 \subseteq \cdots$).
	Second, whenever a counterexample $(\inputTrace_i, \outputTrace_i)$ is added to $X_i$ to form $X_{i+1}$, then $(\inputTrace_i, \outputTrace_i) \notin X_i$.
	To see why this is the case, remember that \algoName\ always constructs hypothesis reward machines that are consistent with the current sample.
	Thus, the reward machine $\hypothesisRM_i$ is consistent with $X_i$.
	However, $(\inputTrace_i, \outputTrace_i)$ was added because $\hypothesisRM_i(\inputTrace_i) \neq \outputTrace_i$.
	Hence, $(\inputTrace_i, \outputTrace_i)$ cannot have been an element of $X_i$.
	\item \label{obs:sequence-of-samples:2} The true reward machine $\machine$, the one that \implement s the reward function $\mdpRewardFunction$, is by definition consistent with all samples $X_i$ that are generated during the run of \algoName.
\end{enumerate}

Once a new counterexample is added, \algoName\ learns a new reward machine.
Let $\hypothesisRM_0, \hypothesisRM_1, \ldots$ be the sequence of these reward machines, where $\hypothesisRM_i$ is computed based on the sample $X_i$.
As above, we make two observations about this sequence.

\begin{enumerate} \setcounter{enumi}{2}
	\item \label{obs:sequence-of-hypotheses:1} We have $|\hypothesisRM_i| \leq |\hypothesisRM_{i+1}|$.
	Towards a contradiction, assume that $|\hypothesisRM_i| > |\hypothesisRM_{i+1}|$.
	Since \algoName\ always computes consistent reward machines and $X_i \subsetneq X_{i+1}$ (see Observation~\ref{obs:sequence-of-samples:1}), we know that $\hypothesisRM_{i+1}$ is not only consistent with $X_{i+1}$ but also with $X_i$ (by definition of consistency).
	Moreover, \algoName\ always computes consistent reward machines of minimal size.
	Thus, since $\hypothesisRM_{i+1}$ is consistent with $X_i$ and $|\hypothesisRM_{i+1}| < |\hypothesisRM_i|$, the reward machine $\hypothesisRM_i$ is not minimal, which yields the desired contradiction.
	
	\item \label{obs:sequence-of-hypotheses:2} We have $\hypothesisRM_i \neq \hypothesisRM_j$ for each $j \in \{ 0, \ldots, i-1 \}$; in other words, the reward machines generated during the run of \algoName\ are semantically distinct.
	This is a consequence of the facts that  $(\inputTrace_j, \outputTrace_j)$ was a counterexample to $\hypothesisRM_j$ (i.e., $\hypothesisRM_j(\inputTrace_j) \neq \outputTrace_j$) and the learning algorithm for reward machines always constructs consistent reward machines (which implies $\hypothesisRM_i(\inputTrace_j) = \outputTrace_j$).
\end{enumerate}

Observations~\ref{obs:sequence-of-samples:2} and \ref{obs:sequence-of-hypotheses:1} now provide $|\machine|$ as an upper bound on the size of the hypothesis reward machines constructed in the run of \algoName.
Since there are only finite many reward machines of size $|\machine|$ or less, Observation~\ref{obs:sequence-of-hypotheses:2} then implies that there exists an $i^\star \in \mathbb N$ after which no new reward machine is inferred.
Thus, it is left to show that $\hypothesisRM_{i^\star}(\inputTrace) = \machine(\inputTrace)$ for all attainable label sequences $\inputTrace$.

Towards a contradiction, assume that there exists an attainable label sequence $\inputTrace$ such that $\hypothesisRM_{i^\star}(\inputTrace) \neq \machine(\inputTrace)$.
Lemma~\ref{lem:reward-machine-difference} then guarantees the existence of an $\maxLengthEpisode$-attainable label sequence $\inputTrace^\star$ with
\begin{align*}
	\maxLengthEpisode & \leq 2^{|\mdp|} (|\hypothesisRM_{i^\star}| + |\machine| + 2) - 1 \\
	& \leq 2^{|\mdp|}  (2|\machine| + 2) - 1 \\
	& = 2^{|\mdp| + 1}  (|\machine| + 1) - 1
\end{align*}
such that $\hypothesisRM_{i^\star}(\inputTrace^\star) \neq \machine(\inputTrace^\star)$.
By Corollary~\ref{col:attainable-traces}, \algoName\ almost surely explores the label sequence $\inputTrace^\star$ in the limit because we assume $\mathit{eplength} \geq 2^{|\mdp| + 1}  (|\machine| + 1) - 1 = \maxLengthEpisode$.
Thus, the trace $(\inputTrace^\star, \outputTrace^\star)$, where $\outputTrace^\star = \machine(\inputTrace^\star)$, is almost surely returned as a new counterexample, resulting in a new sample $X_{i^\star + 1}$.
This triggers the construction of a new reward machine $\hypothesisRM_{i^\star + 1}$ (which will then be different from all previous reward machines).
However, this contradicts the assumption that no new reward machine is constructed after $\hypothesisRM_{i^\star}$.
Thus, $\hypothesisRM_{i^\star}(\inputTrace) = \machine(\inputTrace)$ holds for all attainable input sequences $\inputTrace$.
\end{proof}

%---------- Proof of Theorem 1 ----------
\section{Proof of Theorem~\ref{thm:convergenceInTheLimit}}
To prove Theorem~\ref{thm:convergenceInTheLimit}, we use the fact that \algoName\ will eventually learn a reward machine equivalent to the reward machine $\machine$ on all attainable label sequences (see Lemma~\ref{lem:RPML-learns-correct-reward-machine}).
Then, closely following the proof of Theorem 4.1 from~\cite{DBLP:conf/icml/IcarteKVM18}, 
we construct an MDP $\mdp_\machine$, 
show that using the same policy for $\mdp$ and $\mdp_\machine$ yields same rewards,
and, due to convergence of q-learning for $\mdp_\machine$,  
conlcude that \algoName{}\ converges towards an optimal policy for $\mdp$.
%The reward function of the MDP $\mdp$ is non-Markovian with respect to its set of states.
%However, if observed with respect to the product of its states and the states of a reward machine $\machine$ for $\mdp$, the reward is again Markovian.
%This follows directly from the definition of a reward machine corresponding to a non-Markovian reward function.
Lemma~\ref{lem:markovianViewOfNMRDP} describes the construction of the mentioned MDP $\mdp_\machine$.

\begin{lemma}
\label{lem:markovianViewOfNMRDP}
Given an MDP $\mdp = (\mdpStates, \mdpInit, \mdpActions, \mdpProb, \mdpRewardFunction, \mdpDiscount, \rmLabels, \rmLabelingFunction)$ 
with a non-Markovian reward function defined by a reward machine 
$\machine = (\mealyStates, \mealyCommonState_{\init}, \rmInputAlphabet, \rmOutputAlphabet, \mealyTransition, \mealyOutput)$, 
one can construct an MDP $\mdp_\machine$ whose reward function is Markovian such that every attainable label sequence of $\mdp_\machine$ gets the same reward as in $\mdp$. 
Furthermore, any policy for $\mdp_\machine$ achieves the same expected reward in $\mdp$.
\end{lemma}

\begin{proof}
Let $\mdp = (\mdpStates, \mdpInit, \mdpActions, \mdpProb, \mdpRewardFunction, \mdpDiscount, \rmLabels, \rmLabelingFunction)$ be a labeled MDP and $\machine = (\mealyStates, \mealyInit, \rmInputAlphabet, \rmOutputAlphabet, \mealyTransition, \mealyOutput)$ a reward machine encoding its reward function.
We define the product MDP $\mdp_\machine = (\mdpStates', \mdpCommonState'_{\init}, \mdpActions, \mdpProb', \mdpRewardFunction', \mdpDiscount', \rmLabels', \rmLabelingFunction')$ by 
\begin{itemize}
	\item $\mdpStates' = \mdpStates \times \mealyStates$;
	\item $\mdpCommonState'_{\init} = (\mdpInit, \mealyCommonState_{\init})$;
	\item $\mdpActions = \mdpActions$; 
	\item $\mdpProb' \bigl( (\mdpCommonState, \mealyCommonState), \mdpCommonAction, (\mdpCommonState', \mealyCommonState') \bigr)$\\
    $= \begin{cases}
	\mdpProb(\mdpCommonState, \mdpCommonAction, \mdpCommonState') & \mealyCommonState' = \mealyTransition(\mealyCommonState, \rmLabelingFunction(\mdpCommonState, \mdpCommonAction, \mdpCommonState'));\\
	0 & \text{otherwise};
\end{cases}
$
	\item $\rmLabels' = \rmLabels$; $\rmLabelingFunction'= \rmLabelingFunction$;
	\item $\mdpRewardFunction' \bigl( (\mdpCommonState, \mealyCommonState), \mdpCommonAction, (\mdpCommonState', \mealyCommonState') \bigr) = \mealyOutput \bigl(\mealyCommonState, \rmLabelingFunction(\mdpCommonState, \mdpCommonAction, \mdpCommonState') \bigr)$; and
	\item $\mdpDiscount' = \mdpDiscount$.
\end{itemize}

The described MDP has a Markovian reward function that matches $\mdpRewardFunction$, 
the (non-Markovian) reward function of $\mdp$
defined by the reward machine $\machine$
(Definition~\ref{def:rewardMealyMachines}).
Since the reward functions and discount factors are the same, the claims follow.
\end{proof}

Lemma~\ref{lem:RPML-learns-correct-reward-machine} shows that eventually $\hypothesisRM$, the reward machine learned by \algoName, will be equivalent to $\machine$ on all attainable label sequences.
Thus, using Lemma~\ref{lem:markovianViewOfNMRDP}, an optimal policy for MDP $\mdp_\hypothesisRM$ will also be optimal for $\mdp$.

When running episodes of QRM (Algorithm~\ref{alg:QRMepisode}) under the reward machine $\hypothesisRM$, an update of a $q$-function connected to a state of a reward machine corresponds to updating the $q$ function for $\mdp_\hypothesisRM$.
Because $\textit{eplength} \geq |\mdp|$,
the fact that QRM uses $\epsilon$-greedy strategy 
and that updates are done in parallel for all states of the reward machine $\hypothesisRM$, 
we know that every state-action pair of the MDP $\mdp_\hypothesisRM$
will be seen infinitely often.
Hence, convergence of q-learning for $\mdp_\hypothesisRM$ to an optimal policy 
is guaranteed~\cite{Watkins1992}.
Finally, because of Lemma~\ref{lem:markovianViewOfNMRDP},
\algoName\ converges to an optimal policy, too.

We have proved that 
if the number of episodes goes to infinity,
and the length of an episode is at least 
$2^{|\mdp| + 1} (|\machine| + 1) - 1$,
then \algoName{} converges towards an optimal policy.

\section{Proof of Theorem~\ref{thm:convergenceInTheLimit_opt}}
In order to prove Theorem~\ref{thm:convergenceInTheLimit_opt}, we first need to prove the following lemma. 
\begin{lemma}
	\label{lem:RPML-learns-correct-reward-machine-opt}
	Let $\mdp$ be a labeled MDP and $\machine$ the reward machine \implementing\ the rewards of $\mdp$.
	Then, \algoName\ with Optimizations 1 and 2 with $\mathit{eplength} \geq 2^{|\mdp| + 1}  (|\machine| + 1) - 1$ learns a reward machine that is equivalent to $\machine$ on all attainable traces in the limit (i.e., when the number of episodes goes to infinity).
\end{lemma}

\begin{proof}
	With Optimizations 1 and 2, let $X_0, X_1, \ldots$ (with slight abuse of notation from the proof of Lemma \ref{lem:RPML-learns-correct-reward-machine}) be the sequence of sets that arise in the run of \algoOptName\ whenever the non-empty set of new counterexamples $\newCounterexamples$ are added to the set $X$. Then, it can be shown that Observation 1, 2, 3 and 4 in the proof of Lemma \ref{lem:RPML-learns-correct-reward-machine} still hold and thus Lemma \ref{lem:RPML-learns-correct-reward-machine-opt} holds.
\end{proof}

With Lemma \ref{lem:RPML-learns-correct-reward-machine-opt} and following the analysis in the proofs of Theorem \ref{thm:convergenceInTheLimit}, Theorem~\ref{thm:convergenceInTheLimit_opt} holds.

\section{Proof of Theorem \ref{th_q}}

To prove Theorem \ref{th_q}, we first recapitulate the definition of $k$-horizon optimal discounted action value functions \cite{Givan2003}.

\begin{definition}
	\label{define_horizon}
	Let $\machine=(\mealyStates, \mealyInit, 2^\rmLabels, \mealyOutputAlphabet, \mealyTransition,  \mealyOutput)$ be a reward machine \implementing\ the rewards of a labeled MDP $\mdp=(\mdpStates, \mdpInit, \mdpActions, \mdpProb, \mdpRewardFunction, \mdpDiscount, \rmLabels, \rmLabelingFunction)$. We define the $k$-horizon optimal discounted action value function $q^{\ast\mealyCommonState}_k(s,a)$ recursively as follows:
	\begin{align}
	\begin{split}
	q^{\ast\mealyCommonState}_k(s, a)=&\sum_{\mealyCommonState'\in \mealyStates}\sum_{s'\in S}T(s, \mealyCommonState, a, s', \mealyCommonState')\\&\times [\mealyOutput(\mealyCommonState, \rmLabelingFunction(s,a,s'))
	+\gamma\max_{a'\in \mdpActions}q^{\ast\mealyCommonState'}_{k-1}(s',a')] \nonumber,
	\end{split}
	\label{q_horizon}
	\end{align}      
	where 
	\[
	T(s, \mealyCommonState, a, s', \mealyCommonState')=\begin{cases}
	p(s, a, s'),  ~~\mbox{if}~\mealyCommonState'=\delta(\mealyCommonState, \rmLabelingFunction(s,a,s'));\nonumber \\ 
	0, ~~~~~~~~~~~~~~~~~\mbox{otherwise},
	\end{cases}
	\]
	and $q^{\ast\mealyCommonState}_0(s, a)=0$ for every $\mealyCommonState\in \mealyStates$, $s\in \mdpStates$ and $a\in \mdpActions$.
\end{definition}

We then give the following lemma based on the equivalent relationship formalized in Definition \ref{def_equi}.
\begin{lemma}
	Let $\machine=(\mealyStates, \mealyInit, 2^\rmLabels, \mealyOutputAlphabet, \mealyTransition,  \mealyOutput)$ and $\hat{\machine}=(\hat{\mealyStates}, \mealyInitHat, 2^\rmLabels, \mealyOutputAlphabet, \hat{\mealyTransition}, \hat{\mealyOutput})$ be two reward machines \implementing\ the rewards of a labeled MDP $\mdp=(\mdpStates, \mdpInit, \mdpActions, \mdpProb, \mdpRewardFunction, \mdpDiscount, \rmLabels, \rmLabelingFunction)$. For state $\mealyCommonState\in \mealyStates$ and state $\hat{\mealyCommonState}\in \hat{\mealyStates}$, if $\mealyCommonState\sim \hat{\mealyCommonState}$, then for every $\mealyCommonInput \in \mealyInputAlphabet$, we have $\mealyTransition(\mealyCommonState, \mealyCommonInput)\sim \hat{\mealyTransition}(\hat{\mealyCommonState}, \mealyCommonInput)$.
	\label{lemma:equivalentSuccessors}
\end{lemma}  

\begin{proof}       
	For a Mealy machine $\machine =  (\mealyStates, \mealyInit, \mealyInputAlphabet, \mealyOutputAlphabet, \mealyTransition, \mealyOutput)$,
	we extend the output function $\mealyOutput$ to an output function $\mealyOutput^+ \colon \mealyStates \times (\mealyInputAlphabet)^+ \to \mealyOutput^+$ over (nonempty) words: 
	$\mealyOutput^+(\mealyCommonState, \mealyCommonInput) = \mealyOutput(\mealyCommonState, \mealyCommonInput)$ and $\mealyOutput^+(\mealyCommonState, \mealyCommonInput\cdot\inputTrace) = \mealyOutput(\mealyCommonState, \mealyCommonInput) \cdot \mealyOutput^+(\mealyTransition(\mealyCommonState, \mealyCommonInput), \inputTrace)$,  for every $\mealyCommonState \in \mealyStates$, $\mealyCommonInput \in \mealyInputAlphabet$, and $\inputTrace \in (\mealyInputAlphabet)^+$,  where we use $\cdot$ to denote concatenation.
	% In a similar fashion, we extend the transition function to $\mealyTransition^+ \colon \mealyStates \times (\mealyInputAlphabet)^+ \to \mealyStates$.
	
	For two Mealy machines $\machine= (\mealyStates, \mealyInit, \mealyInputAlphabet, \mealyOutputAlphabet, \mealyTransition, \mealyOutput)$, and $\hat{\machine} = (\hat{\mealyStates}, \mealyInitHat, \mealyInputAlphabet, \mealyOutputAlphabet,\hat{\mealyTransition}, \hat{\mealyOutput})$ (over the same input and output alphabet), two states $\mealyCommonState \in \mealyStates, \hat{\mealyCommonState} \in \hat{\mealyStates}$ and any label sequence $\inputTrace$, we have $\machine[\mealyCommonState](\inputTrace) = \hat{\machine}[\hat{\mealyCommonState}](\inputTrace)$, if and only if $\mealyOutput^+(\mealyCommonState, \inputTrace) = \hat{\mealyOutput}^+(\hat{\mealyCommonState}, \inputTrace)$. Therefore, from Definition \ref{def_equi} we have $\mealyCommonState \sim \hat{\mealyCommonState}$ if and only if $\mealyOutput^+(\mealyCommonState, \inputTrace) = \hat{\mealyOutput}^+(\hat{\mealyCommonState}, \inputTrace)$ for all $\inputTrace \in (\mealyInputAlphabet)^+$.
	
	Thus, we have
	\begin{align}
	\mealyOutput^{+}(\mealyTransition(\mealyCommonState, \mealyCommonInput), \inputTrace) &= \mealyOutput^{+}(\mealyCommonState, \mealyCommonInput \cdot \inputTrace) \nonumber \\
	&\stackrel{\text{(a)}}{=} \hat{\mealyOutput}^{+}(\hat{\mealyCommonState}, \mealyCommonInput \cdot \inputTrace)  \\
	&= \hat{\mealyOutput}^{+}(\hat{\mealyTransition}(\hat{\mealyCommonState}, \mealyCommonInput), \inputTrace), \nonumber
	\end{align}
	where (a) follows from the equivalence of $\mealyCommonState$ and $\hat{\mealyCommonState}$. Therefore, for every $\mealyCommonInput \in \mealyInputAlphabet$, we have $\machine[\mealyTransition(\mealyCommonState, \mealyCommonInput)](\inputTrace) = \hat{\machine}[\hat{\mealyTransition}(\hat{\mealyCommonState}, \mealyCommonInput)](\inputTrace)$ holds for all label sequences $\inputTrace$. 
	
\end{proof}    

With Definition \ref{define_horizon} and Lemma \ref{lemma:equivalentSuccessors}, we proceed to prove that for two equivalent states, the corresponding $k$-horizon optimal discounted action value functions are the same (as formalized in the following lemma).

\begin{lemma}
	Let $\machine=(\mealyStates, \mealyInit, 2^\rmLabels, \mealyOutputAlphabet, \mealyTransition,  \mealyOutput)$ and $\hat{\machine}=(\hat{\mealyStates}, \mealyInitHat, 2^\rmLabels, \mealyOutputAlphabet, \hat{\mealyTransition}, \hat{\mealyOutput})$ be two reward machines \implementing\ the rewards of a labeled MDP $\mdp=(\mdpStates, \mdpInit, \mdpActions, \mdpProb, \mdpRewardFunction, \mdpDiscount, \rmLabels, \rmLabelingFunction)$. For states $\mealyCommonState\in\mealyStates$ and $\hat{\mealyCommonState}\in\hat{\mealyStates}$, if $\mealyCommonState\sim\hat{\mealyCommonState}$, then for every $s\in \mdpStates$ and $a \in \mdpActions$, $q^{\ast\mealyCommonState}_k(s,a)=q^{\ast\hat{\mealyCommonState}}_k(s,a)$ for every $k$.
	\label{lemma_q} 
\end{lemma}                

\begin{proof}
We use induction to prove Lemma \ref{lemma_q}. For $k=1$, we have for every $s\in \mdpStates$ and $a\in \mdpActions$,
\begin{align}
\begin{split}
q^{\ast\mealyCommonState}_1(s, a)&= \sum_{\mealyCommonState'\in \mealyStates}\sum_{s'\in S} T(s, \mealyCommonState, a, s', \mealyCommonState')\mealyOutput(\mealyCommonState, \rmLabelingFunction(s,a,s'))\\
&=\sum_{s'\in S} T(s, \mealyCommonState, a, s', \delta(\mealyCommonState, \rmLabelingFunction(s,a,s')))\mealyOutput(\mealyCommonState, \rmLabelingFunction(s,a,s'))  \\
&=\sum_{s'\in S} p(s, a, s')\mealyOutput(\mealyCommonState, \rmLabelingFunction(s,a,s'))  \\
&\stackrel{\text{(b)}}{=} \sum_{s'\in S} p(s, a, s')\hat{\mealyOutput}(\hat{\mealyCommonState}, \rmLabelingFunction(s,a,s'))  \\
&=\sum_{s'\in S} \hat{T}(s, \hat{\mealyCommonState}, a, s', \hat{\delta}(\hat{\mealyCommonState}, \rmLabelingFunction(s,a,s')))\hat{\mealyOutput}(\hat{\mealyCommonState}, \rmLabelingFunction(s,a,s'))\\
&=\sum_{\hat{\mealyCommonState}'\in\hat{\mealyStates}}\sum_{s'\in S} \hat{T}(s, \hat{\mealyCommonState}, a, s', \hat{\mealyCommonState}')\hat{\mealyOutput}(\hat{\mealyCommonState}, \rmLabelingFunction(s,a,s'))  \\
&= q^{\ast\hat{\mealyCommonState}}_1(s, a) \nonumber
\end{split}
\end{align}
where the equality (b) comes from the fact that $\mealyCommonState \sim \hat{\mealyCommonState}$, and
\[
\hat{T}(s, \hat{\mealyCommonState}, a, s', \hat{\mealyCommonState}')=\begin{cases}
p(s, a, s'),  ~~\mbox{if}~\hat{\mealyCommonState}'=\hat{\delta}(\hat{\mealyCommonState}, \rmLabelingFunction(s,a,s'));\nonumber \\ 
0, ~~~~~~~~~~~~~~~~~\mbox{otherwise}.
\end{cases}
\]

%Specifically, as $\mealyCommonState \sim \hat{\mealyCommonState}$, we have the followings: for any $s$, $a$ and $s'$, $\hat{\mealyOutput}(\hat{\mealyCommonState}, \rmLabelingFunction(s,a,s'))$; and for any $s$, $a$, $s'$, $\mealyCommonState$ and $\mealyCommonState'$, $\mealyCommonState'=\delta(\mealyCommonState, \rmLabelingFunction(s,a,s'))$ holds if and only if $\mealyCommonState'=\delta(\mealyCommonState, \rmLabelingFunction(s,a,s'))$

Now we assume that for every state $\mealyCommonState\in\dfaStates$ and state $\hat{\mealyCommonState}\in \hat{\mealyStates}$, if $\mealyCommonState\sim \hat{\mealyCommonState}$, then we have that $q^{\ast\mealyCommonState}_{k-1}(s, a)=q^{\ast\hat{\mealyCommonState}}_{k-1}(s, a)$ holds for every $s\in S$ and every $a\in A$. We proceed to prove that for every state $\mealyCommonState\in\mealyStates$ and state $\hat{\mealyCommonState}\in \hat{\mealyStates}$, if $\mealyCommonState\sim \hat{\mealyCommonState}$, then we have that $q^{\ast \mealyCommonState}_{k}(s, a)=q^{\ast\hat{\mealyCommonState}}_{k}(s, a)$ holds for every $s\in S$ and every $a\in A$.

For every $s\in \mdpStates$ and every $a\in \mdpActions$, we have
\begin{align}
\begin{split} 
q^{\ast\mealyCommonState}_k(s, a) &= \sum_{\mealyCommonState'\in \mealyStates}\sum_{s'\in S}T(s, \mealyCommonState, a, s', \mealyCommonState') \\ &\times [\mealyOutput(\mealyCommonState, \rmLabelingFunction(s,a,s'))+\gamma\max_{a'\in \mdpActions}q^{\ast\mealyCommonState'}_{k-1}(s',a')]\\
&= \sum_{s'\in S}T(s, \mealyCommonState, a, s', \delta(\mealyCommonState, \rmLabelingFunction(s,a,s'))) \\ &\times [\mealyOutput(\mealyCommonState, \rmLabelingFunction(s,a,s'))+\gamma\max_{a'\in \mdpActions}q^{\ast\delta(\mealyCommonState, \rmLabelingFunction(s,a,s'))}_{k-1}(s',a')]\\
&= \sum_{s'\in S}p(s, a, s') \\ & \times[\mealyOutput(\mealyCommonState, \rmLabelingFunction(s,a,s'))+\gamma\max_{a'\in \mdpActions}q^{\ast\delta(\mealyCommonState, \rmLabelingFunction(s,a,s'))}_{k-1}(s',a')]\\
&\stackrel{\text{(c)}}{=} \sum_{s'\in S}p(s, a, s') \\ & \times[\hat{\mealyOutput}(\hat{\mealyCommonState}, \rmLabelingFunction(s,a,s'))+\gamma\max_{a'\in \mdpActions}q^{\ast\hat{\delta}(\hat{\mealyCommonState}, \rmLabelingFunction(s,a,s'))}_{k-1}(s',a')]\\
&= \sum_{s'\in S}\hat{T}(s, \hat{\mealyCommonState}, a, s', \hat{\delta}(\hat{\mealyCommonState}, \rmLabelingFunction(s,a,s'))) \\& \times [\mealyOutput(\hat{\mealyCommonState}, \rmLabelingFunction(s,a,s'))+\gamma\max_{a'\in \mdpActions}q^{\ast\hat{\delta}(\hat{\mealyCommonState}, \rmLabelingFunction(s,a,s'))}_{k-1}(s',a')] \\
&= \sum_{\hat{\mealyCommonState}'\in\hat{\mealyStates}}\sum_{s'\in S}\hat{T}(s, \hat{\mealyCommonState}, a, s', \hat{\mealyCommonState}') \\& \times [\mealyOutput(\hat{\mealyCommonState}, \rmLabelingFunction(s,a,s'))+\gamma\max_{a'\in \mdpActions}q^{\ast\hat{\delta}(\hat{\mealyCommonState}, \rmLabelingFunction(s,a,s'))}_{k-1}(s',a')]
\end{split}\\                                          
&= q^{\ast\hat{\mealyCommonState}}_k(s, a) \nonumber                                                
\end{align}    
where the equality (c) comes from the fact that $\mealyCommonState \sim \hat{\mealyCommonState}$ and $\delta(\mealyCommonState, \rmLabelingFunction(s,a,s'))\sim\hat{\delta}(\hat{\mealyCommonState}, \rmLabelingFunction(s,a,s'))$ (according to Lemma \ref{lemma:equivalentSuccessors}).

Therefore, it is proven by induction that Lemma \ref{lemma_q} holds.

% \label{eq:assumption} 

\end{proof}

%From the equality of $k$-horizon optimal discounted action values for two equivalent states, letting $k$ grow to infinity, we prove the equality of optimal q-function values for two equivalent states.

With Lemma \ref{lemma_q}, we now proceed to prove Theorem \ref{th_q}. According to Lemma \ref{lemma_q}, if $\mealyCommonState\sim\hat{\mealyCommonState}$, then for every $s\in \mdpStates$ and $a\in \mdpActions$, $q^{\ast\mealyCommonState}_k(s,a)=q^{\ast\hat{\mealyCommonState}}_k(s,a)$ for every $k$. When $k\rightarrow\infty$, $q^{\ast}_k(\mealyCommonState,s,a)$ and $q^{\ast}_k(\hat{\mealyCommonState},s,a)$ converge to the fixed points as follows. 
\begin{align}
\begin{split}
q^{\ast\mealyCommonState}(s, a)=&\sum_{\mealyCommonState'\in \mealyStates}\sum_{s'\in S}T(s, \mealyCommonState, a, s', \mealyCommonState')\\&\times [\mealyOutput(\mealyCommonState, \rmLabelingFunction(s,a,s'))
+\gamma\max_{a'\in \mdpActions}q^{\ast\mealyCommonState'}(s',a')], \nonumber \\
q^{\ast\hat{\mealyCommonState}}(s, a)=&\sum_{\hat{\mealyCommonState}'\in\hat{\mealyStates}}\sum_{s'\in \mdpStates}T(s, \hat{\mealyCommonState}, a, s', \hat{\mealyCommonState}')\\&\times [\hat{\mealyOutput}(\hat{\mealyCommonState}, \rmLabelingFunction(s,a,s'))
+\gamma\max_{a'\in \mdpActions}q^{\ast\hat{\mealyCommonState}'}(s', a')].
\end{split}
\label{q_fixed}
\end{align}    

Therefore, for every $s\in \mdpStates$ and $a\in \mdpActions$, $q^{\ast\mealyCommonState}(s,a)=q^{\ast\hat{\mealyCommonState}}(s,a)$.

\iftoggle{long}{  
\section{Details in \Traffic}
\label{sec:app_traffic}
We provide the detailed results in \traffic. Figure \ref{fig:ex:road-map-grid} shows the gridded map of the roads in a residential area. The set of actions is $\mdpActions = \set{Straight, Left, Right, Stay}$, corresponding to going straight, turning left, turning right and staying in place. We use a simplified version of transitions at the intersections. For example, at (1,7), if the vehicle stays, then it ends at (1, 7); if the vehicle goes straight, then it ends at (4, 7) at the next step; if 
the vehicle turns left, then it ends at (3, 9) at the next step; and if the vehicle turns right, then it ends at (2, 6) at the next step. The vehicle will make a U-turn if it reaches the end of any road. For example, if the vehicle reaches (10, 7), then it will reach (10, 8) at the next step.

\begin{figure}[th]
	\centering                              
	\includegraphics[width=10cm]{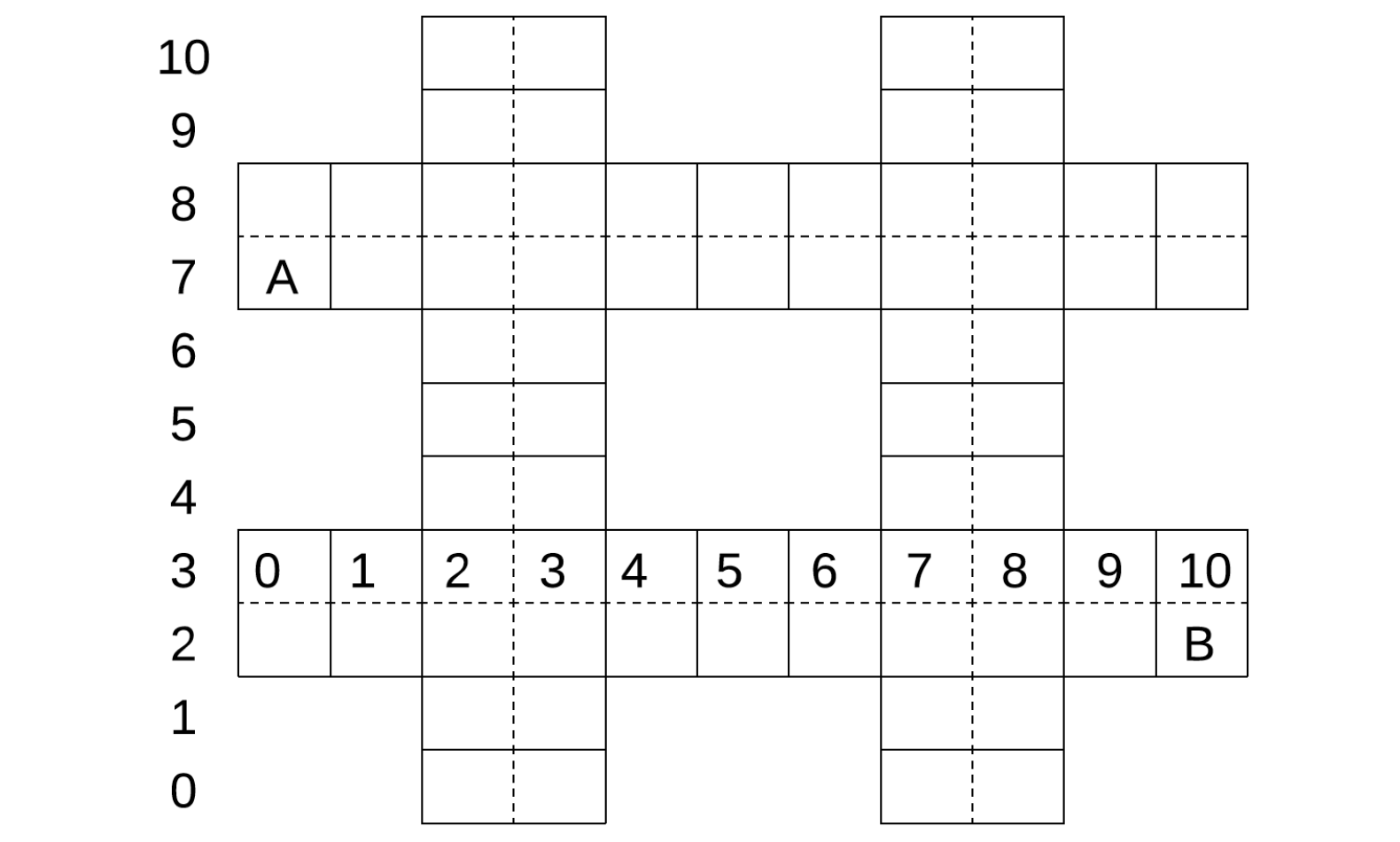}                                         
	\caption{Gridded map of a residential area in \traffic.} \label{fig:ex:road-map-grid}
\end{figure}   

Figure \ref{case0} shows the cumulative rewards of 10 independent simulation runs averaged for every 10 training steps for the \traffic. 

\begin{figure*}[t]
	\centering
	\begin{subfigure}[b]{0.3\textwidth}  
		\centering
		\includegraphics[width=\textwidth]{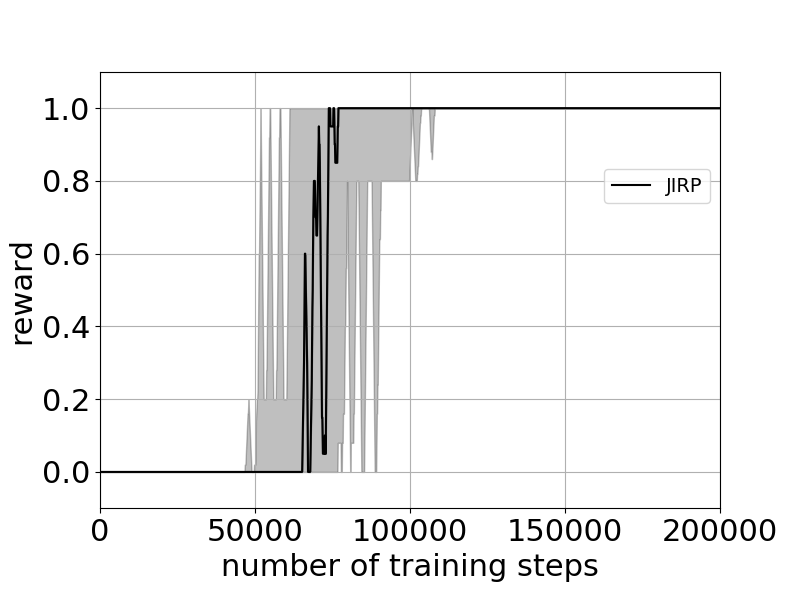}
		\caption{}
	\end{subfigure}
	\begin{subfigure}[b]{0.3\textwidth}
		\centering
		\includegraphics[width=\textwidth]{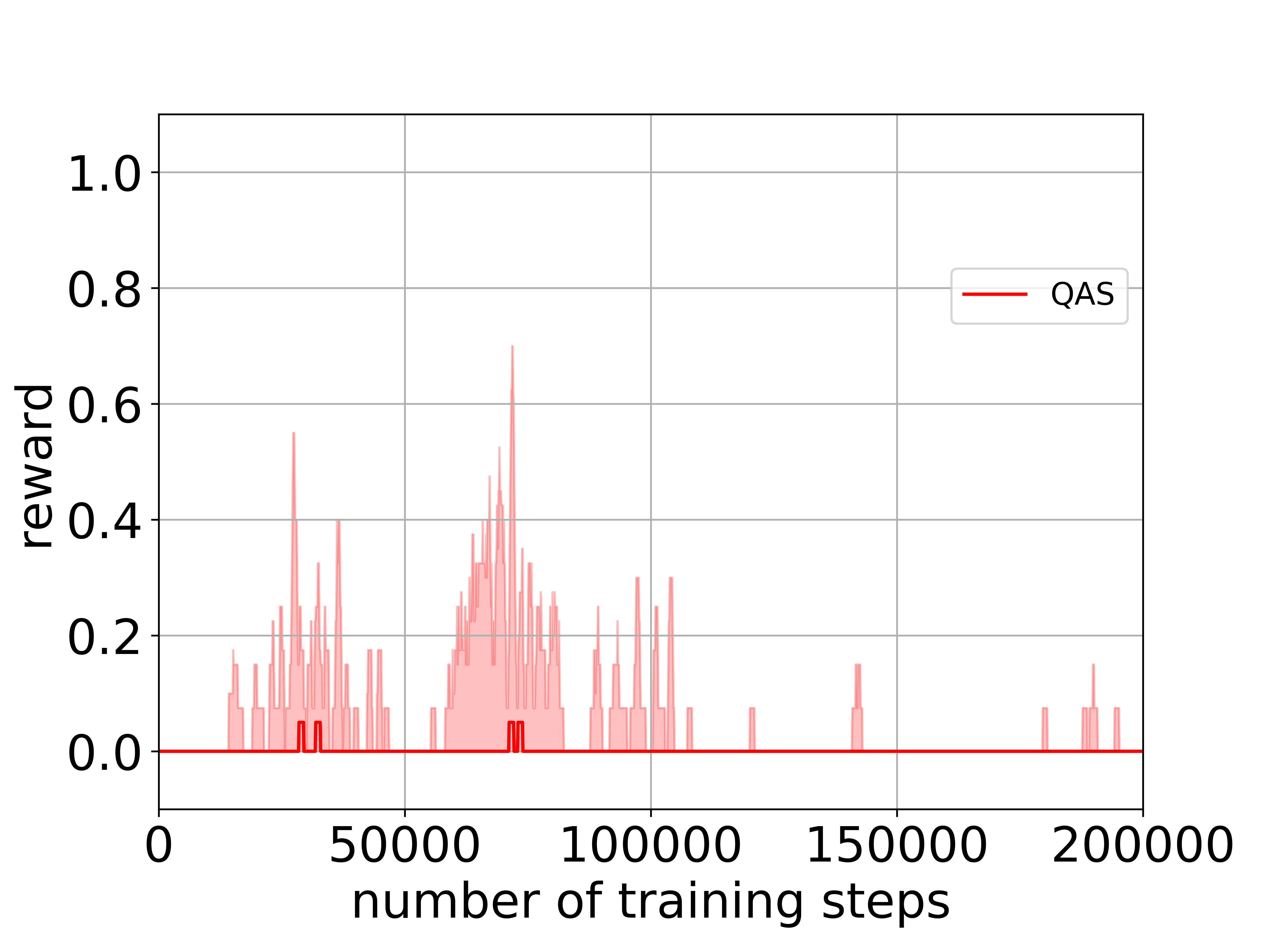}
		\caption{}
	\end{subfigure}
	\begin{subfigure}[b]{0.3\textwidth}
		\centering
		\includegraphics[width=\textwidth]{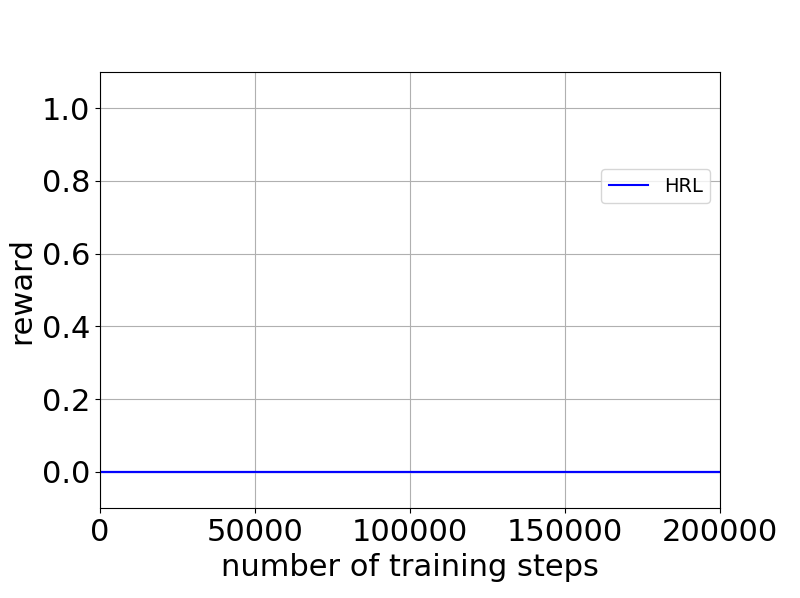}
		\caption{}
	\end{subfigure}
	\caption{Cumulative rewards of 10 independent simulation runs averaged for every 10 training steps in the \traffic: (a) \methodA; (b) \methodB; (c) \methodC.}  
	\label{case0}
\end{figure*}

\section{Details in \Office}
\label{sec:app_office}
We provide the detailed results in the \office. Figure \ref{office_map} shows the map in the \office. We use the triangle to denote the initial position of the agent.
We consider the following four tasks:\\
\textbf{\OfficeA}: get coffee at $\textrm{c}$ and deliver the coffee to the office $\textrm{o}$;\\
\textbf{\OfficeB}: get mail at $\textrm{m}$ and deliver the coffee to the office $\textrm{o}$;\\
\textbf{\OfficeC}: go to the office $\textrm{o}$, then get coffee at $\textrm{c}$ and go back to the office $\textrm{o}$, finally go to mail at $\textrm{m}$; \\
\textbf{\OfficeD}: get coffee at $\textrm{c}$ and deliver the coffee to the office $\textrm{o}$, then come to get coffee at $\textrm{c}$ and deliver the coffee to the frontdesk $\textrm{d}$.
\begin{figure}[t]
	\centering
	\includegraphics[width=0.45\textwidth]{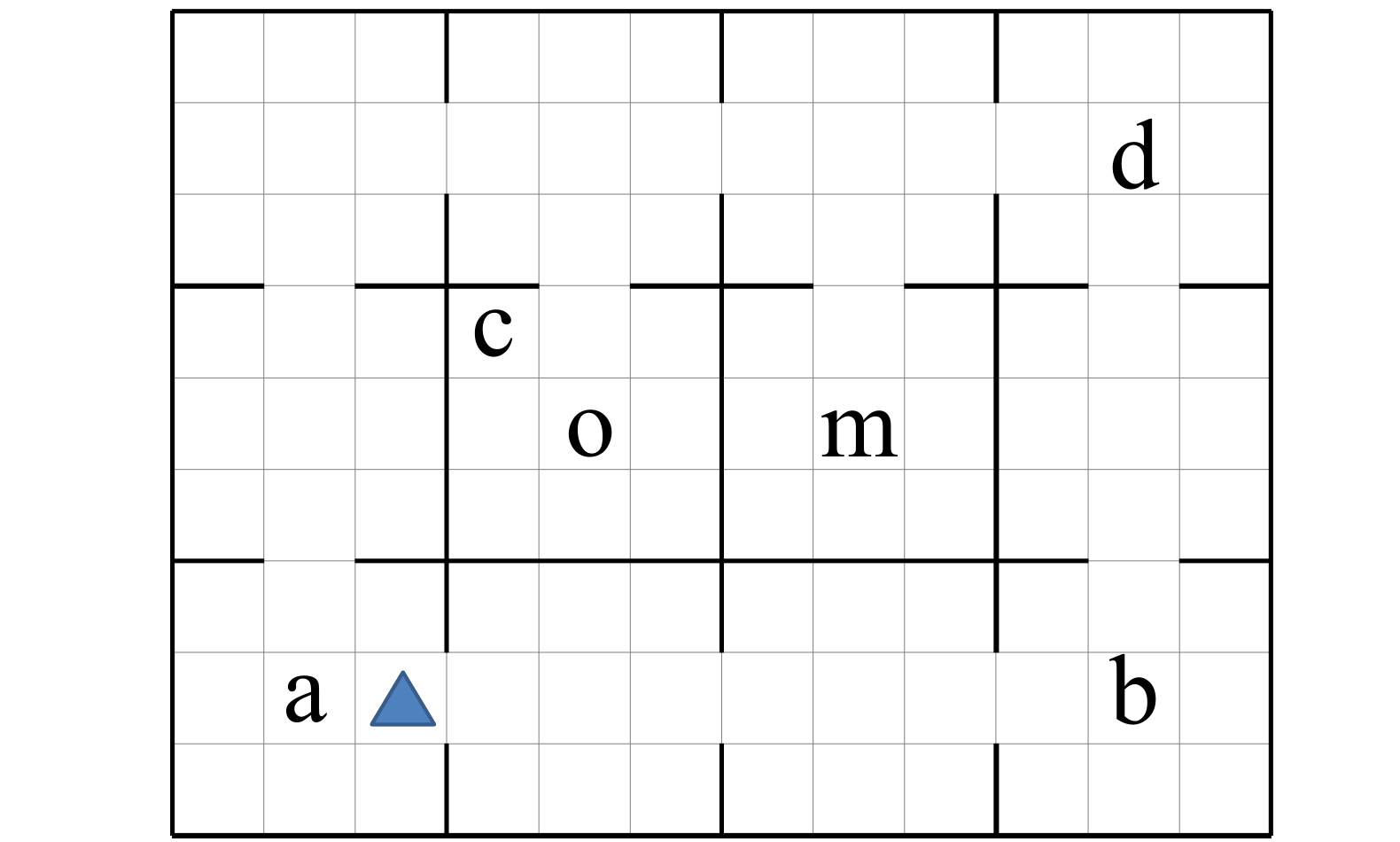}
	\caption{The map in the \office.}
	\label{office_map}
\end{figure}

\subsection{\OfficeA}
For \officeA, Figure \ref{rm_c1_t1} shows the inferred hypothesis reward machine in the last iteration of \algoName{}. Figure \ref{case1_task1} shows the cumulative rewards of 10 independent simulation runs averaged for every 10 training steps for \officeA. 

\begin{figure}[H]
	\centering
	\begin{tikzpicture}[shorten >=1pt,node distance=2cm,on grid,auto] 
	\node[state,initial] (q_0)   {$\mealyCommonState_0$}; 
	\node[state] (q_1) [right=of q_0] {$\mealyCommonState_1$}; 
	\node[state,accepting] (q_2) [right=of q_1] {$\mealyCommonState_2$}; 
	\path[->] 
	(q_0) edge  node {($\textrm{c}$, 0)} (q_1)
	edge [loop above] node {($\lnot \textrm{c}$, 0)} ()
	(q_1) edge  node  {($\textrm{o}$, 1)} (q_2)
	edge [loop above] node {($\lnot \textrm{o}$, 0)} ();
	\end{tikzpicture}
	\caption{The inferred hypothesis reward machine for \officeA~ in the last iteration of \algoName{}.}  
	\label{rm_c1_t1}
\end{figure}
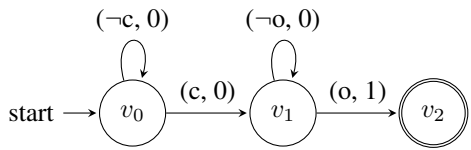

\begin{figure*}[t]
	\centering
	\begin{subfigure}[b]{0.3\textwidth}  
		\centering
		\includegraphics[width=\textwidth]{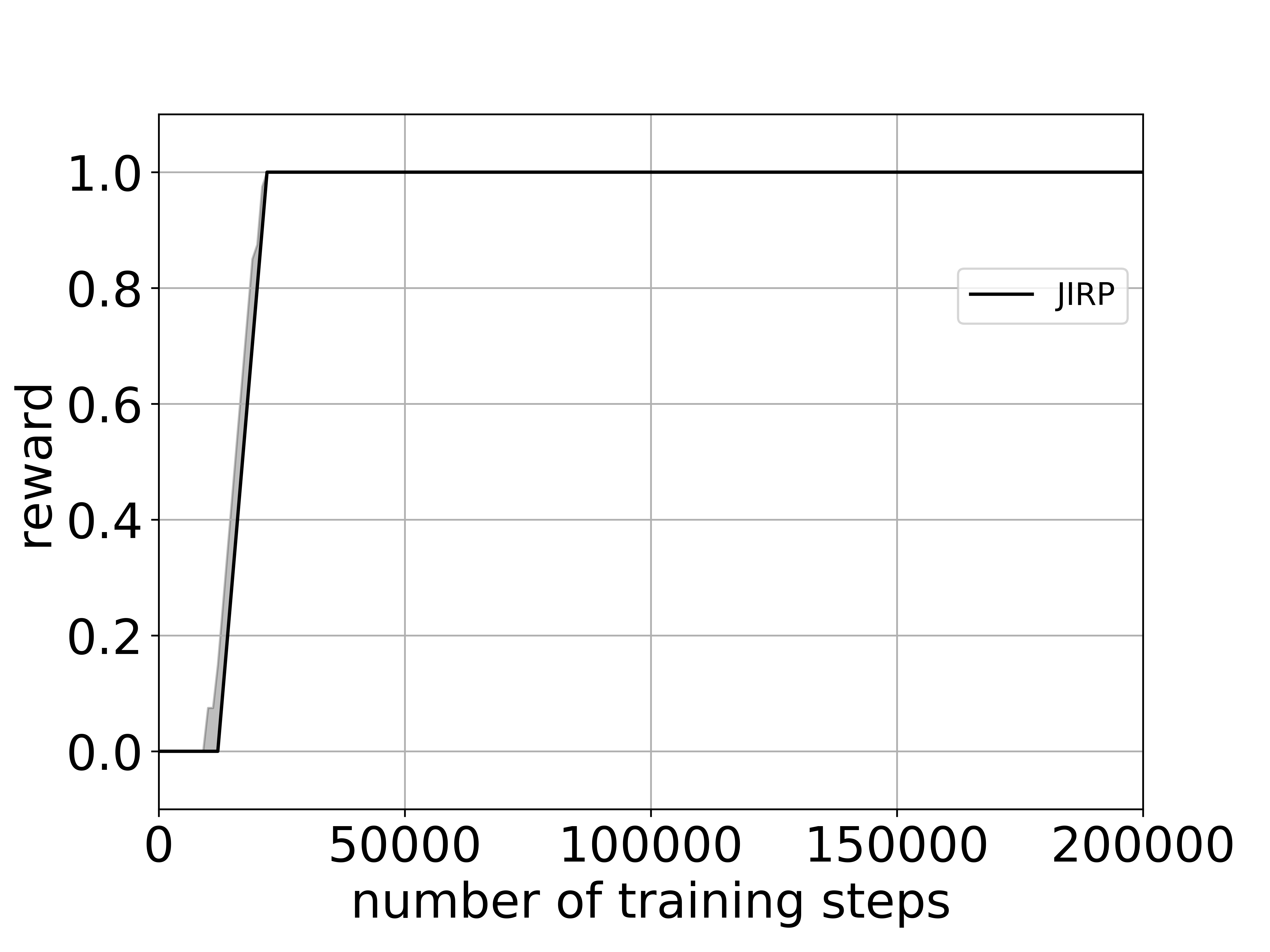}
		\caption{}
	\end{subfigure}
	\begin{subfigure}[b]{0.3\textwidth}
		\centering
		\includegraphics[width=\textwidth]{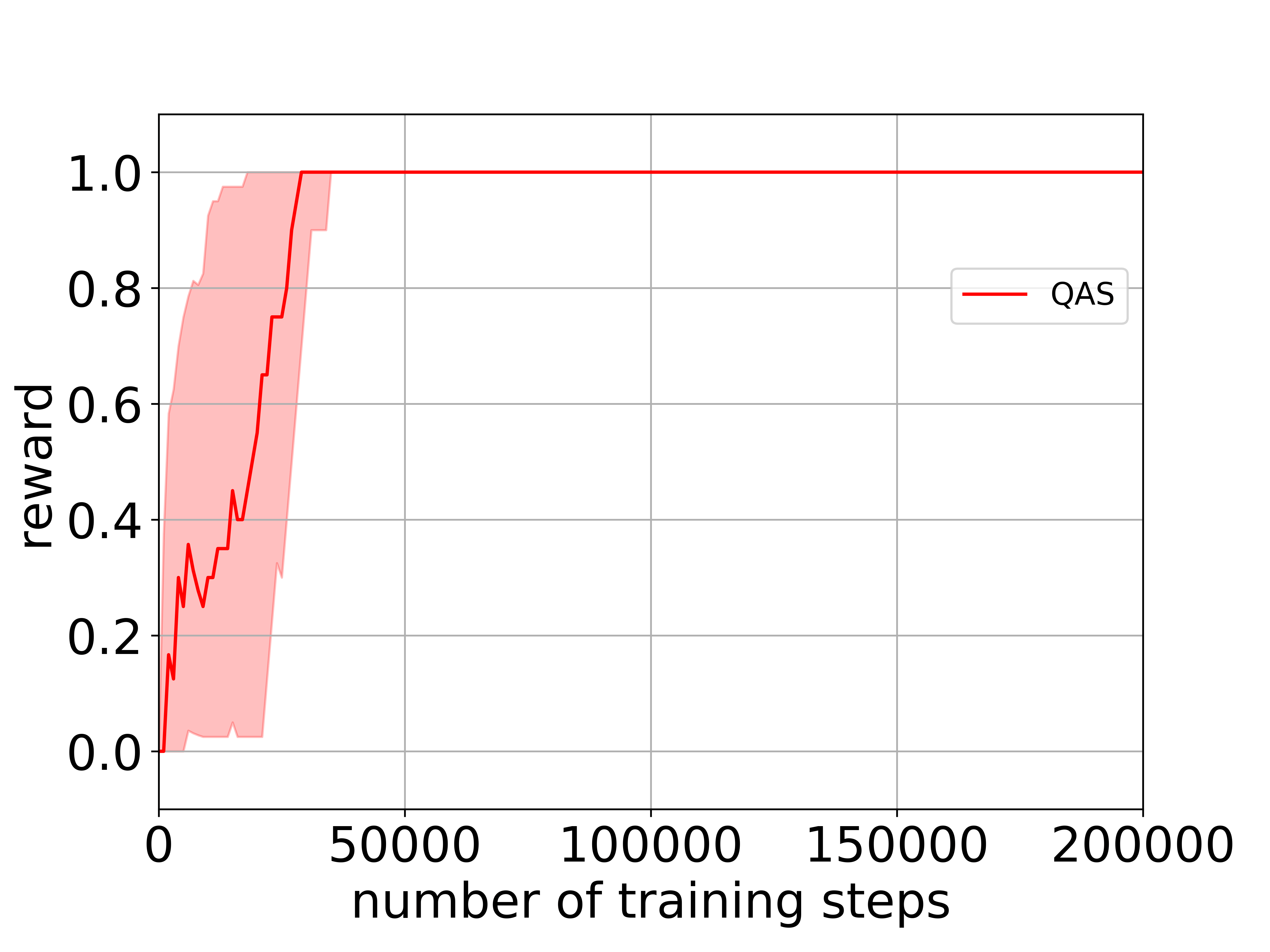}                                         
		\caption{}
	\end{subfigure}
	\begin{subfigure}[b]{0.3\textwidth}
		\centering
		\includegraphics[width=\textwidth]{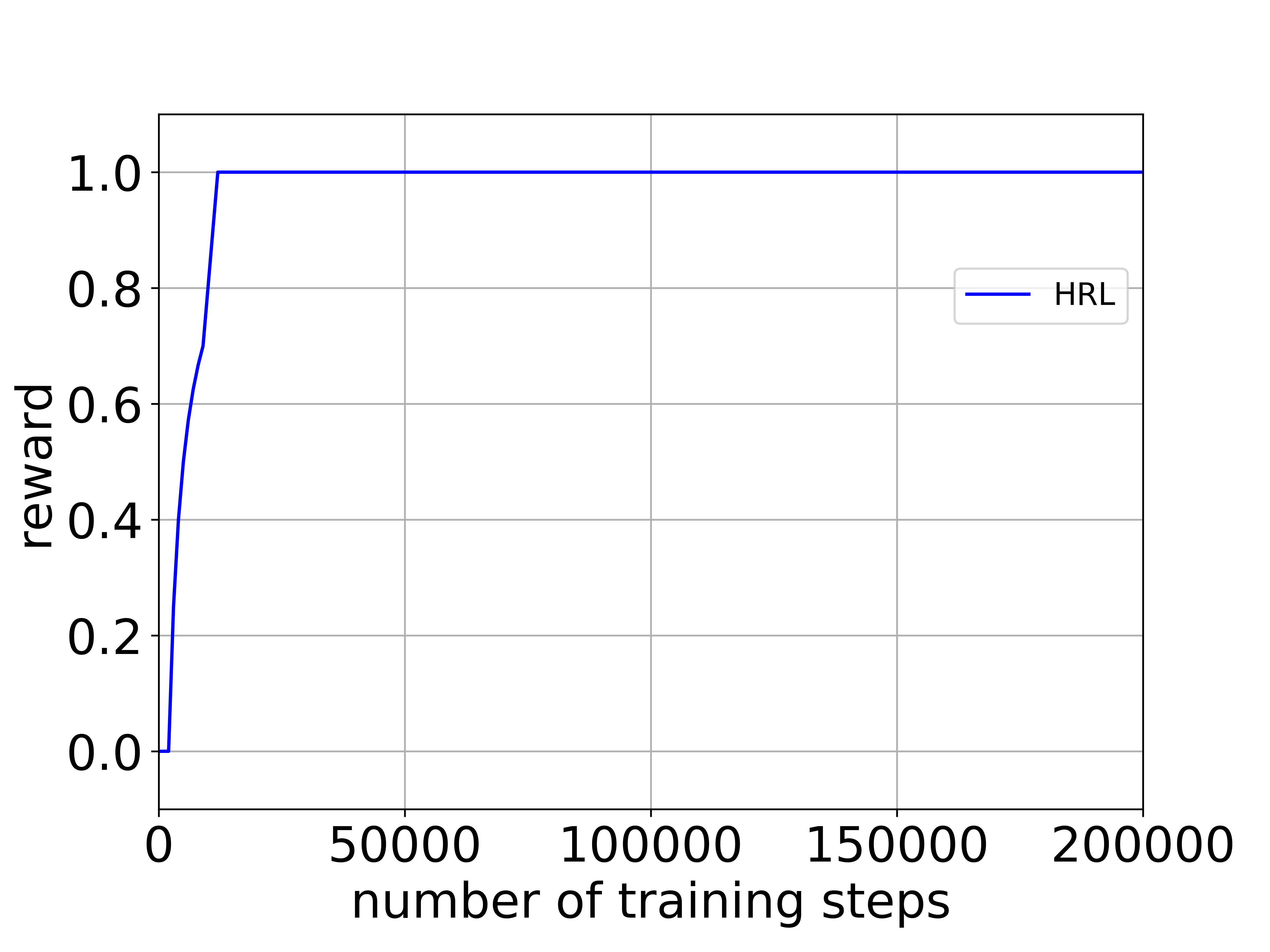}
		\caption{}
	\end{subfigure}
	\caption{Cumulative rewards of 10 independent simulation runs averaged for every 10 training steps for \officeA ~in the \office: (a) \methodA; (b) \methodB; (c) \methodC.}  
	\label{case1_task1}
\end{figure*}

\subsection{\OfficeB}
%For \officeB, the inferred hypothesis reward machine after 15000 training steps is the same as the ground truth reward machine. 
For \officeB, Figure \ref{rm_c1_t2} shows the inferred hypothesis reward machine in the last iteration of \algoName{}. Figure \ref{case1_task2} shows the cumulative rewards of 10 independent simulation runs averaged for every 10 training steps for \officeB.

\begin{figure}[H]
	\centering
	\begin{tikzpicture}[shorten >=1pt,node distance=2cm,on grid,auto] 
	\node[state,initial] (q_0)   {$\mealyCommonState_0$}; 
	\node[state] (q_1) [right=of q_0] {$\mealyCommonState_1$}; 
	\node[state,accepting] (q_2) [right=of q_1] {$\mealyCommonState_2$}; 
	\path[->] 
	(q_0) edge  node {($\textrm{m}$, 0)} (q_1)
	edge [loop above] node {($\lnot \textrm{m}$, 0)} ()
	(q_1) edge  node  {($\textrm{o}$, 1)} (q_2)
	edge [loop above] node {($\lnot \textrm{o}$, 0)} ();
	\end{tikzpicture}
	\caption{The inferred hypothesis reward machine for \officeB~ in the last iteration of \algoName{}.}  
	\label{rm_c1_t2}
\end{figure}
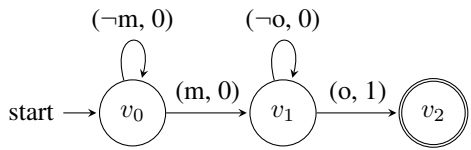

\begin{figure*}[t]
	\centering
	\begin{subfigure}[b]{0.3\textwidth}  
		\centering
		\includegraphics[width=\textwidth]{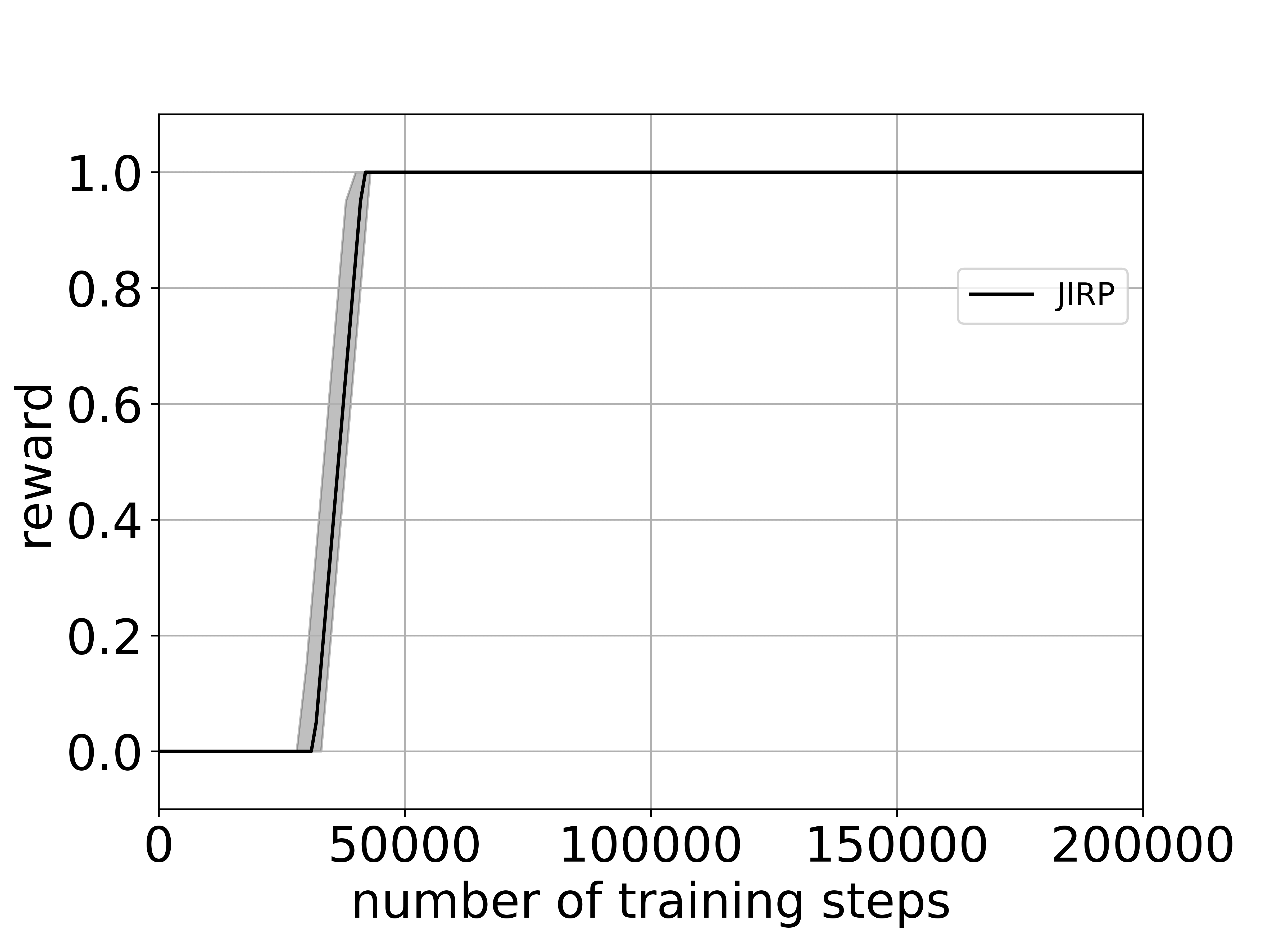}
		\caption{}
	\end{subfigure}
	\begin{subfigure}[b]{0.3\textwidth}
		\centering
		\includegraphics[width=\textwidth]{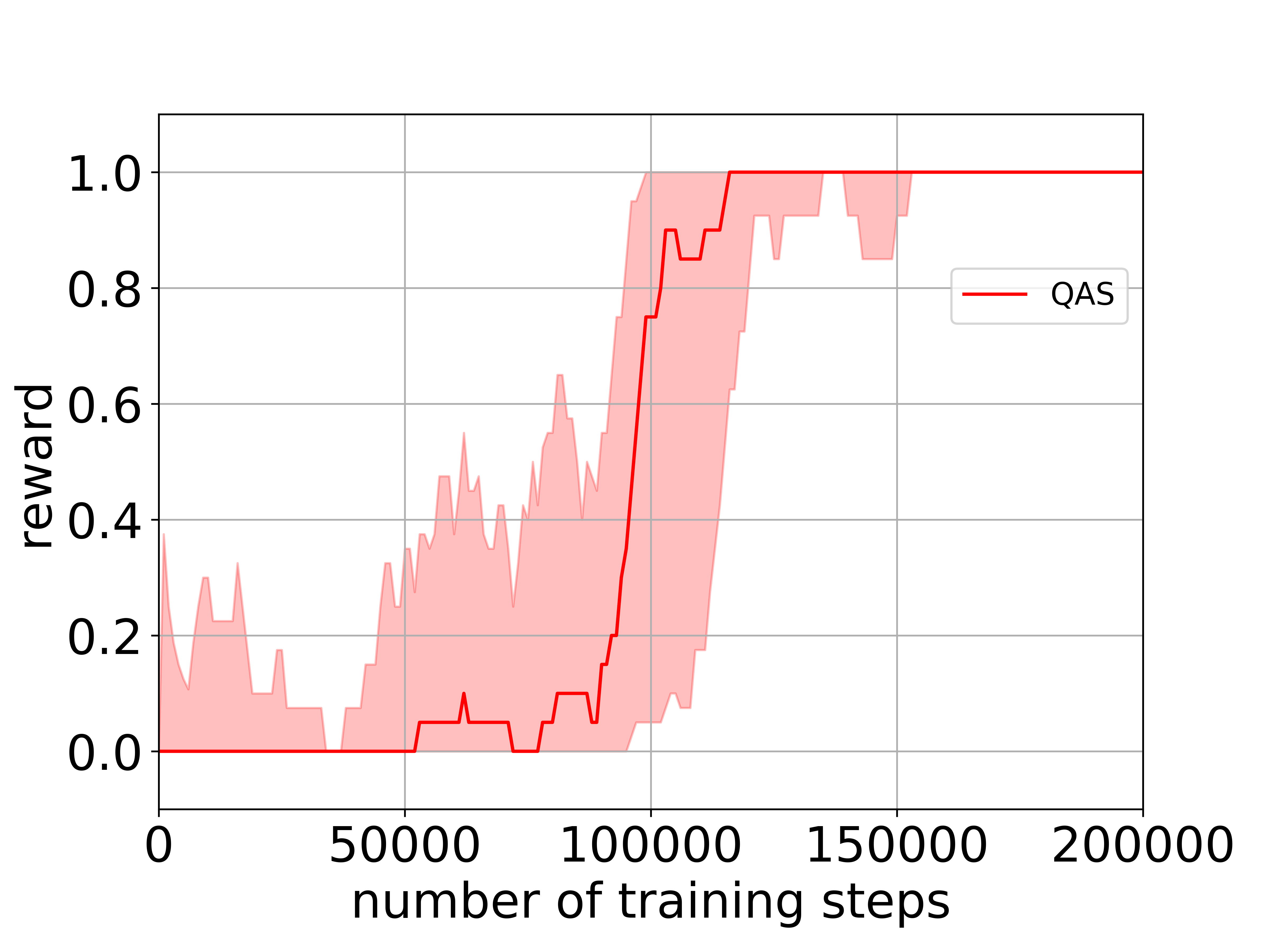}
		\caption{}
	\end{subfigure}
	\begin{subfigure}[b]{0.3\textwidth}
		\centering
		\includegraphics[width=\textwidth]{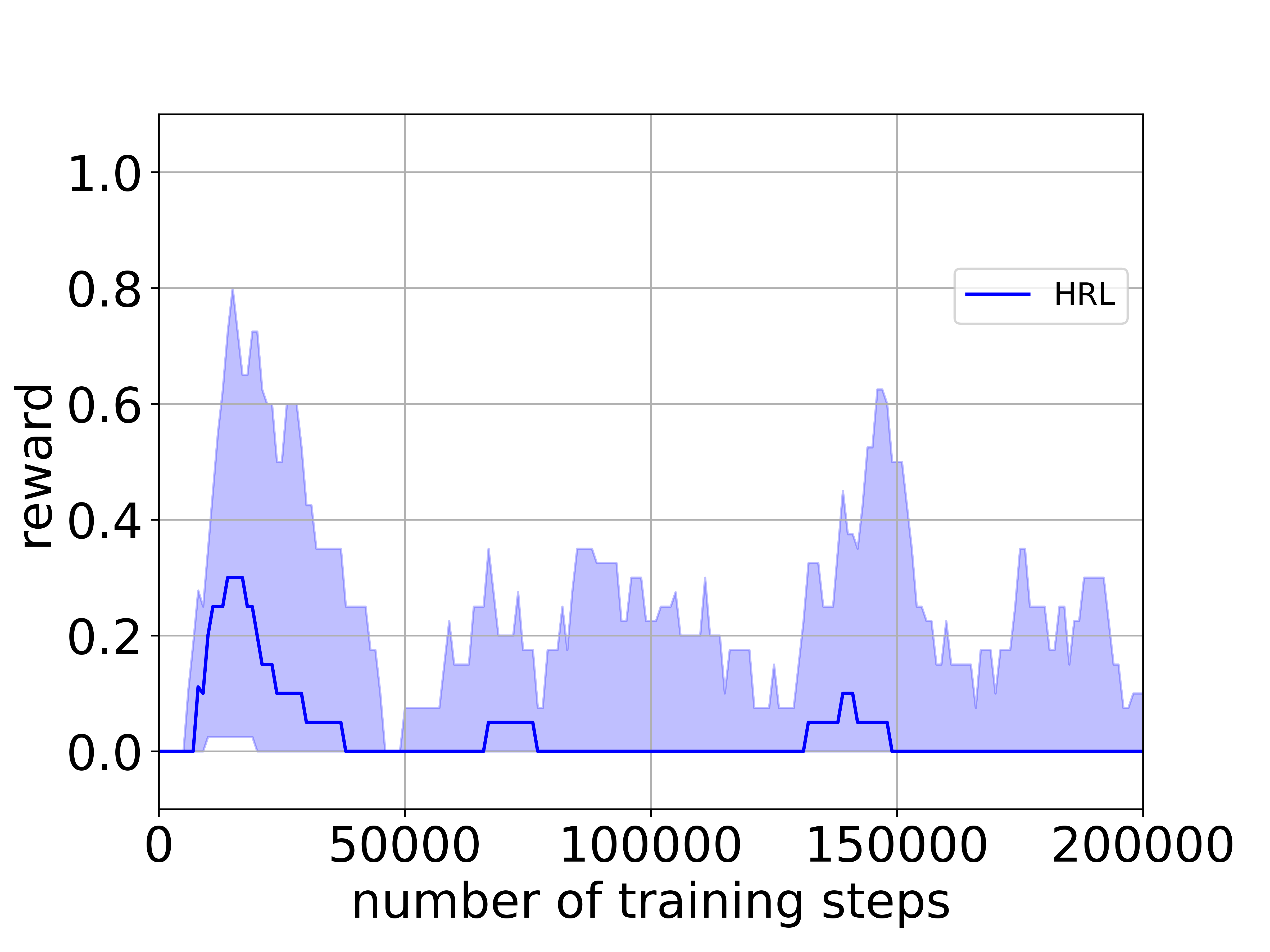}
		\caption{}
	\end{subfigure}
	\caption{Cumulative rewards of 10 independent simulation runs averaged for every 10 training steps for \officeB ~in the \office: (a) \methodA; (b) \methodB; (c) \methodC.}          
	\label{case1_task2}
\end{figure*}

\subsection{\OfficeC}
%For \officeC, the inferred hypothesis reward machine after 15000 training steps and 30000  training steps are shown in Fig. \ref{rm_c1_t3} (a) and (b), respectively. The inferred hypothesis reward machine after 15000 training steps is significantly different from the ground truth reward machine, while the inferred hypothesis reward machine after 30000 training steps is already very close to the ground truth reward machine. 
For \officeC, Figure \ref{rm_c1_t3} shows the inferred hypothesis reward machine in the last iteration of \algoName{}. Figure \ref{case1_task3} shows the cumulative rewards of 10 independent simulation runs averaged for every 10 training steps for \officeC.

%\begin{figure}[H]
%		\begin{tikzpicture}[shorten >=1pt,node distance=1.8cm,on grid,auto] 
%		\node[state,initial] (q_0)   {$\mealyCommonState_0$}; 
%		\node[state] (q_1) [right=of q_0] {$\mealyCommonState_1$}; 
%		\node[state] (q_2) [right=of q_1] {$\mealyCommonState_2$}; 
%		\node[state,accepting] (q_3) [right=of q_2] {$\mealyCommonState_3$};                                                    
%		\path[->] 
%		(q_0) edge [bend left, above] node [bend right=45, above] {($c$, 0)} (q_1)
%		edge [loop above] node {($\lnot \textrm{c}\wedge\lnot \textrm{m}$, 0)} ()  
%		(q_1) edge  node  {($m$, 0)} (q_2)
%		edge [bend left, above] node  [bend left=45, below] {($o$, 0)} (q_0)
%		(q_2) edge  node  {($o$, 1)} (q_3)
%		edge [loop above] node {($c$, 0)} ();
%		\end{tikzpicture}
%	\caption{The inferred hypothesis reward machine for \officeC~ in the last iteration of \algoName{}.}  
%	\label{rm_c1_t3}
%\end{figure}

\begin{figure}[H]
	\centering
	\begin{tikzpicture}[shorten >=1pt,node distance=2cm,on grid,auto] 
	\node[state,initial] (0) at (1, 0) {$\mealyCommonState_0$};
	\node[state] (1) at (2, -2) {$\mealyCommonState_1$};
	\node[state] (2) at (4, 2) {$\mealyCommonState_2$};
	\node[state] (3) at (-1, 2) {$\mealyCommonState_3$};
	\node[state, accepting] (4) at (1, 4) {$\mealyCommonState_4$};
	\draw[<-, shorten <=1pt] (0.west) -- +(-.4, 0);
	\draw[->] (0) to[loop above] node[align=center] {$(\lnot \textrm{o}, 0)$} ();
	\draw[->] (0) to[right] node[sloped, below, align=center] {$(\textrm{o}, 0)$} (1);
	\draw[->] (1) to[bend right=60] node[sloped, below, align=center] {$(\textrm{c}, 0)$} (2);
	\draw[->] (1) to[bend right=60] node[sloped, below, align=center] {$(\textrm{m}\vee\textrm{o}, 0)$} (0);
	\draw[->] (2) to[right] node[sloped, below, align=center] {$(\textrm{c}\vee\textrm{m}, 0)$} (0);
	\draw[->] (2) to[right] node[sloped, above, align=center] {$(\textrm{o}, 0)$} (3);
	\draw[->] (3) to[loop above] node[align=center] {$(\textrm{b}\vee\textrm{c}\vee\textrm{o}, 0)$} ();
	\draw[->] (3) to[bend right=60] node[sloped, below, align=center] {$(\textrm{a}\vee\textrm{d}, 0)$} (1);
	\draw[->] (3) to[right] node[sloped, below, align=center] {$(\textrm{m}, 1)$} (4);
	\end{tikzpicture}
	\caption{The inferred hypothesis reward machine for \officeC\ in the last iteration of \algoName{}.}  
	\label{rm_c1_t3}
\end{figure}

\begin{figure*}[t]
	\centering
	\begin{subfigure}[b]{0.3\textwidth}  
		\centering
		\includegraphics[width=\textwidth]{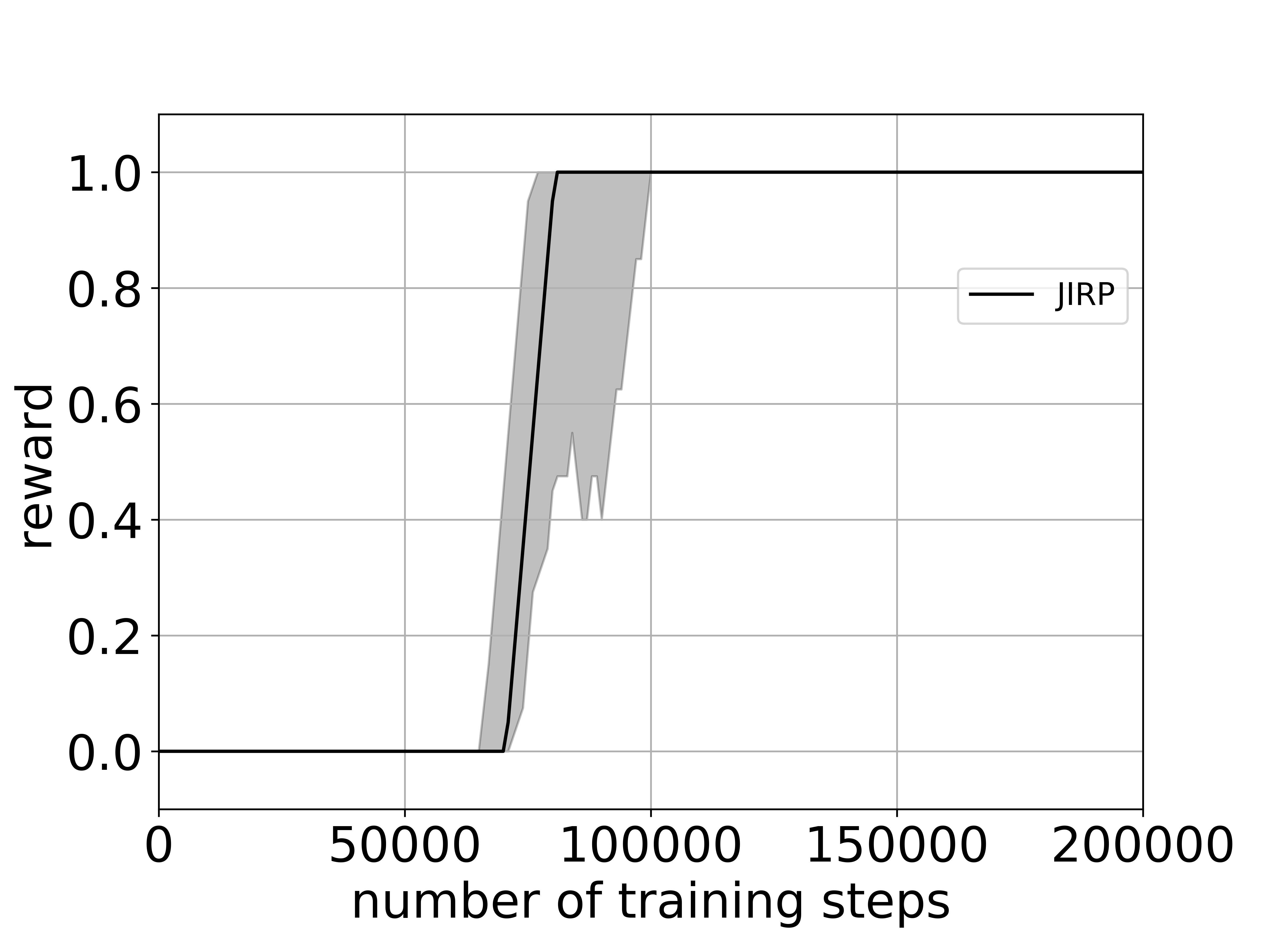}
		\caption{}
	\end{subfigure}
	\begin{subfigure}[b]{0.3\textwidth}
		\centering
		\includegraphics[width=\textwidth]{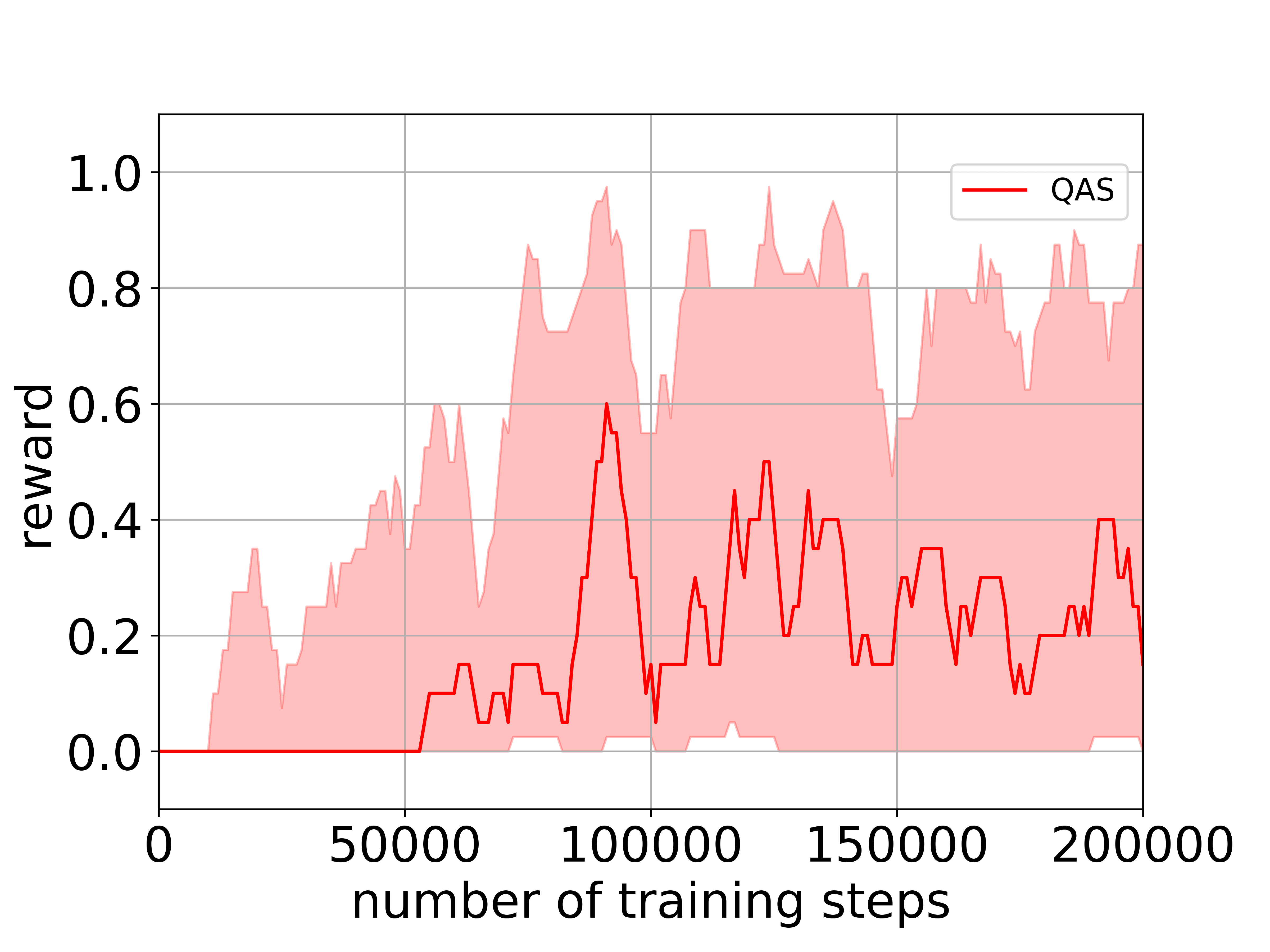}
		\caption{}
	\end{subfigure}
	\begin{subfigure}[b]{0.3\textwidth}
		\centering
		\includegraphics[width=\textwidth]{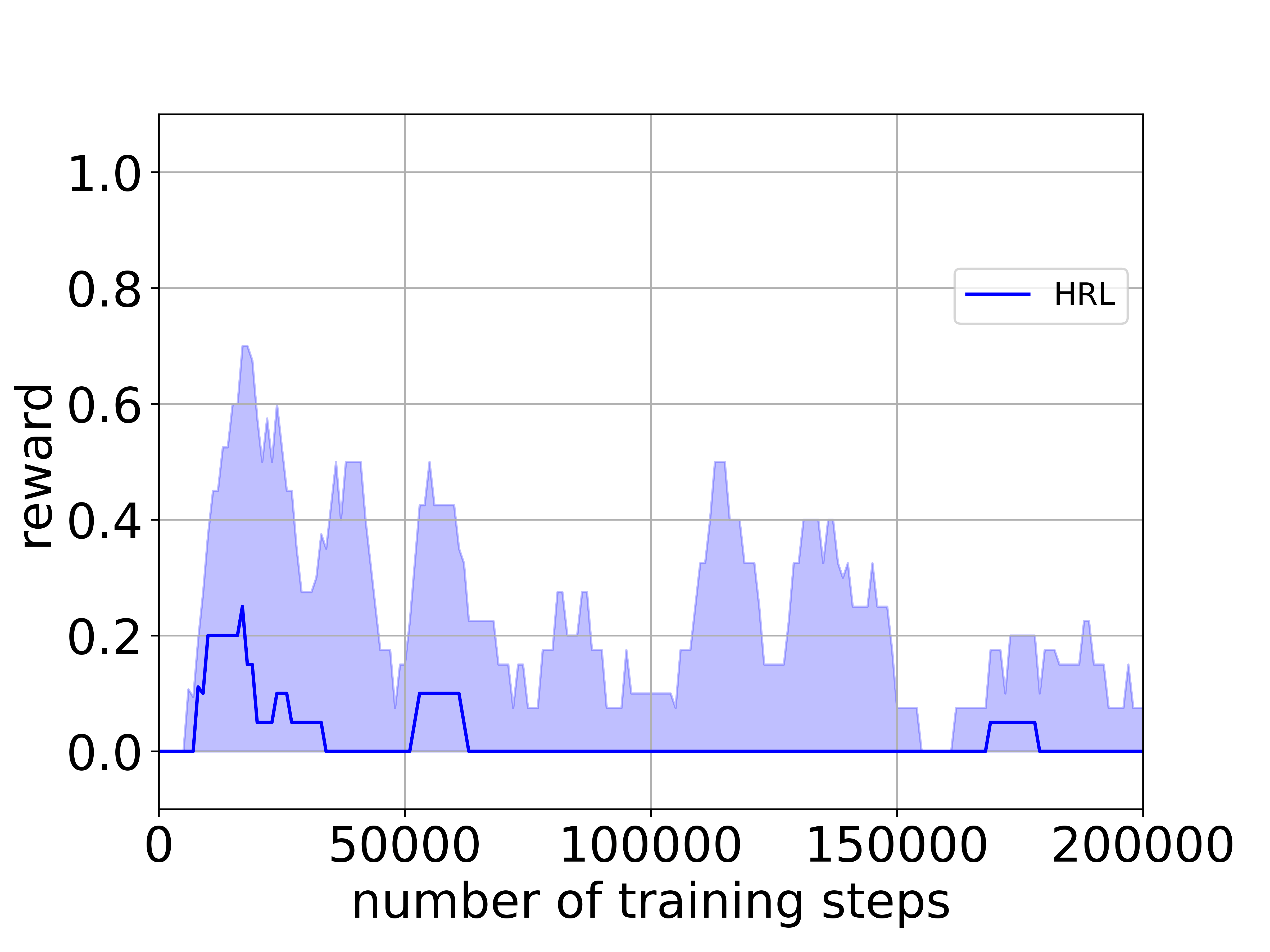}
		\caption{}
	\end{subfigure}
	\caption{Cumulative rewards of 10 independent simulation runs averaged for every 10 training steps for \officeC\ in the \office: (a) \methodA; (b) \methodB; (c) \methodC.}  
	\label{case1_task3}
\end{figure*}

\subsection{\OfficeD}
For \officeD, Figure \ref{case1_task4} shows the inferred hypothesis reward machine in the last iteration of \algoName{}. Figure \ref{case1_task4} shows the cumulative rewards of 10 independent simulation runs averaged for every 10 training steps for \officeD. 

\begin{figure}[H]
	\centering
	\begin{tikzpicture}[shorten >=1pt,node distance=2cm,on grid,auto] 
	\node[state,initial] (0) at (0.5, -1.5) {$\mealyCommonState_0$};
	\node[state] (1) at (3.5, -1.5) {$\mealyCommonState_1$};
	\node[state] (2) at (4, 3) {$\mealyCommonState_2$};
	\node[state] (3) at (-2, 3) {$\mealyCommonState_3$};
	\node[state] (4) at (0.5, 2) {$\mealyCommonState_4$};
	\node[state, accepting] (5) at (2, 3) {$\mealyCommonState_5$};
	\draw[<-, shorten <=1pt] (0.west) -- +(-.4, 0);
	\draw[->] (0) to[loop above] node[align=center] {$(\lnot \textrm{c}, 0)$} ();
	\draw[->] (0) to[bend right=40] node[sloped, above, align=center] {$(\textrm{c}, 0)$} (1);
	\draw[->] (1) to[bend right=40] node[sloped, below, align=center] {$(\textrm{o}, 0)$} (2);
	\draw[->] (1) to[loop above] node[align=center] {$(\lnot \textrm{o}, 0)$} ();
	\draw[->] (2) to[loop above] node[align=center] {$(\textrm{a}\vee\textrm{m}\vee\textrm{o}, 0)$} ();
	\draw[->] (2) to[bend right=50] node[sloped, above, align=center] {$(\textrm{c}, 0)$} (3);
	\draw[->] (3) to[loop above] node[align=center] {$(\textrm{a}\vee\textrm{b}\vee\textrm{o}, 0)$} ();
	\draw[->] (3) to[bend left=50] node[sloped, below, align=center] {$(\textrm{d}, 1)$} (5);
	\draw[->] (3) to[bend right=50] node[sloped, above, align=center] {$(\textrm{m}, 0)$} (0);	
	\draw[->] (3) to[right] node[sloped, below, align=center] {$(\textrm{c}, 0)$} (4);
	\draw[->] (4) to[right] node[sloped, below, align=center] {$(\textrm{d}, 1)$} (5);
	\draw[->] (4) to[bend right=40] node[sloped, above, align=center] {$(\textrm{a}, 0)$} (3);
	\draw[->] (4) to[loop below] node[align=center] {$(\lnot \textrm{a}\vee\lnot \textrm{d}, 0)$} ();
	\draw[->] (2) to[right] node[sloped, below, align=center] {$(\textrm{b}\vee\textrm{d}, 0)$} (0);
	\end{tikzpicture}
	\caption{The inferred hypothesis reward machine for \officeD\ in the last iteration of \algoName{}.}  
	\label{rm_c1_t4}
\end{figure}

\begin{figure*}[t]
	\centering
	\begin{subfigure}[b]{0.3\textwidth}  
		\centering
		\includegraphics[width=\textwidth]{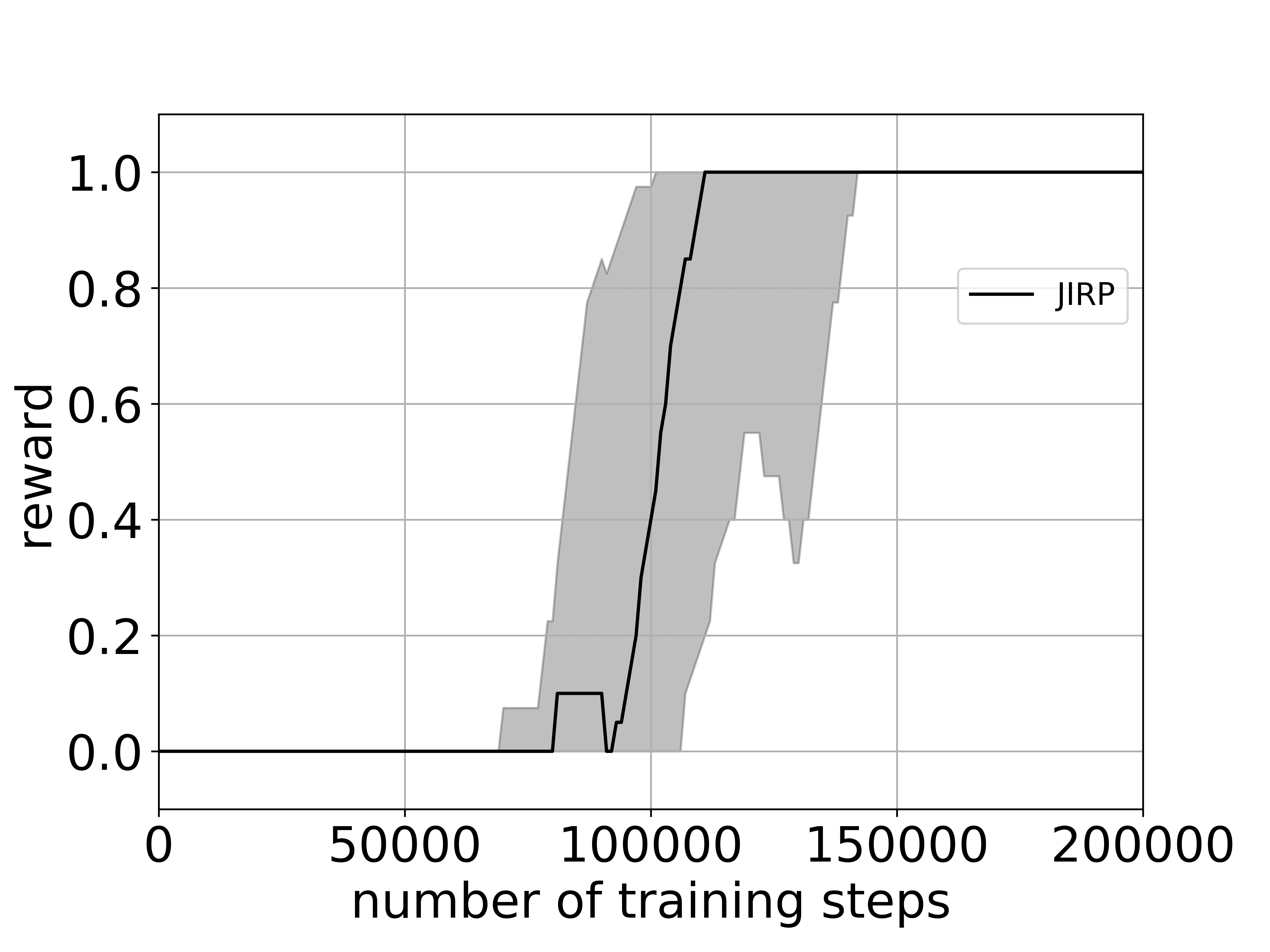}
		\caption{}
	\end{subfigure}
	\begin{subfigure}[b]{0.3\textwidth}
		\centering
		\includegraphics[width=\textwidth]{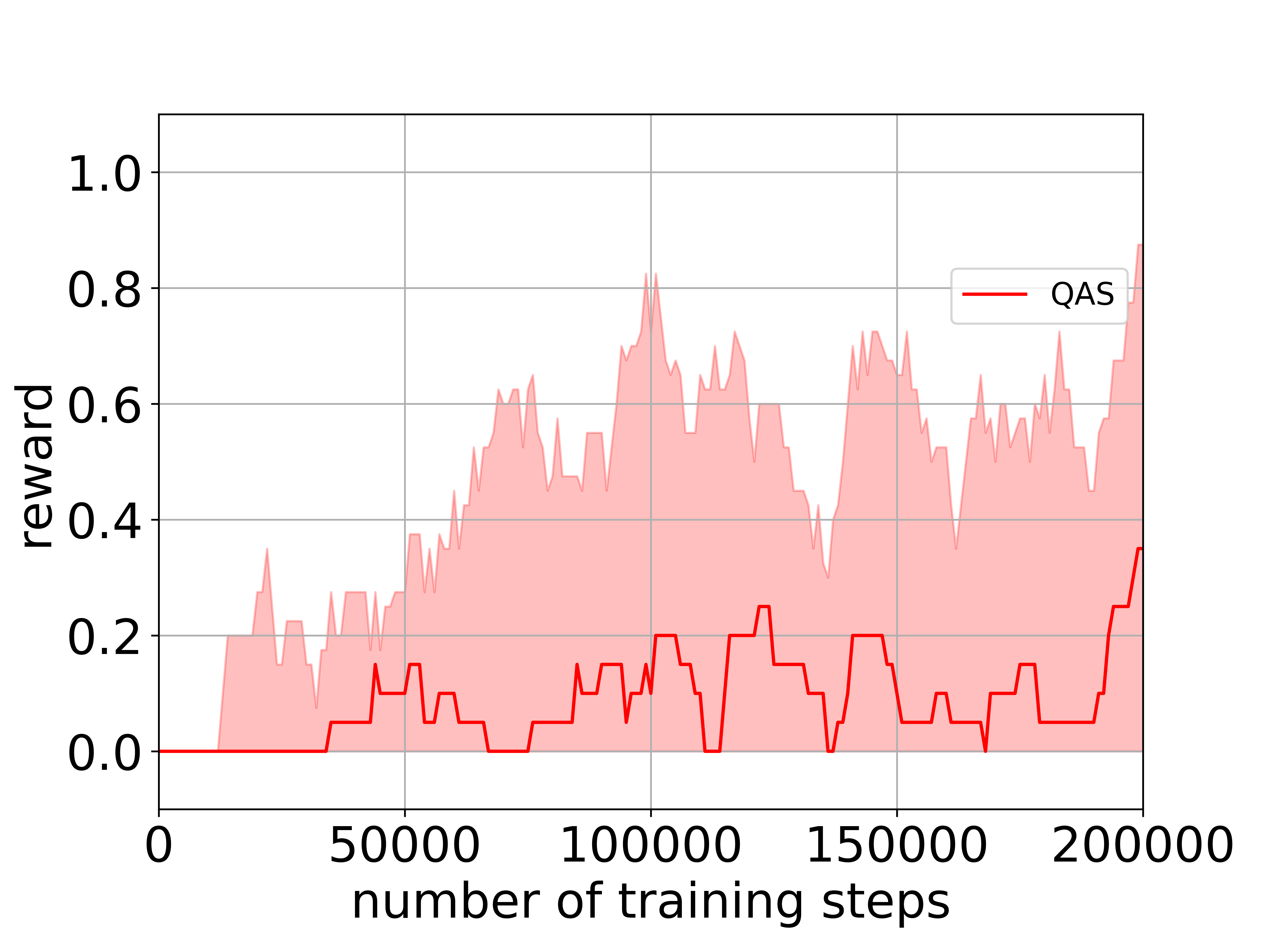}
		\caption{}
	\end{subfigure}
	\begin{subfigure}[b]{0.3\textwidth}
		\centering
		\includegraphics[width=\textwidth]{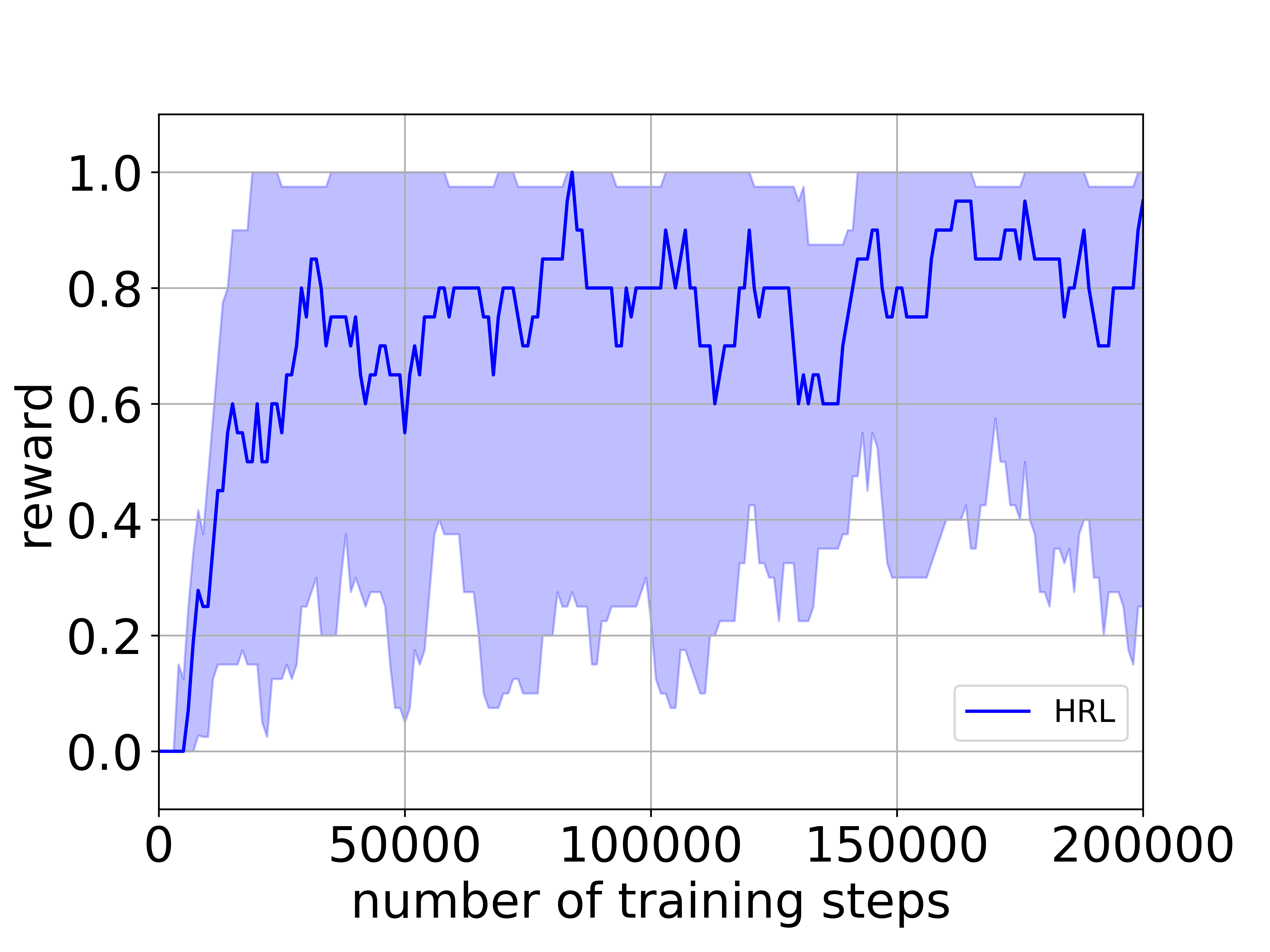}
		\caption{}
	\end{subfigure}
	\caption{Cumulative rewards of 10 independent simulation runs averaged for every 10 training steps for \officeD\ in the \office: (a) \methodA; (b) \methodB; (c) \methodC.}  
	\label{case1_task4}
\end{figure*}

\section{Details in \Craft}
\label{sec:app_craft}
We provide the detailed results in \craft. Figure \ref{craft_map} shows the map in the \office. We use the triangle to denote the initial position of the agent.
We consider the following four tasks:\\
\textbf{\CraftA}: make plank: get wood $\textrm{w}$, then use toolshed $\textrm{t}$ (toolshed cannot be used before wood is gotten);\\
\textbf{\CraftB}: make stick: get wood $\textrm{w}$, then use workbench $\textrm{h}$
(workbench can be used before wood is gotten);\\
\textbf{\CraftC}: make bow: go to workbench $\textrm{h}$, get wood $\textrm{w}$, then go to workbench $\textrm{h}$ and use factory $\textrm{f}$ (in the listed order);\\
\textbf{\CraftD}: make bridge: get wood $\textrm{w}$, get iron $\textrm{i}$, then get wood $\textrm{w}$ and use factory $\textrm{f}$ (in the listed order).

\begin{figure}[t]
	\centering
	\includegraphics[width=0.5\textwidth]{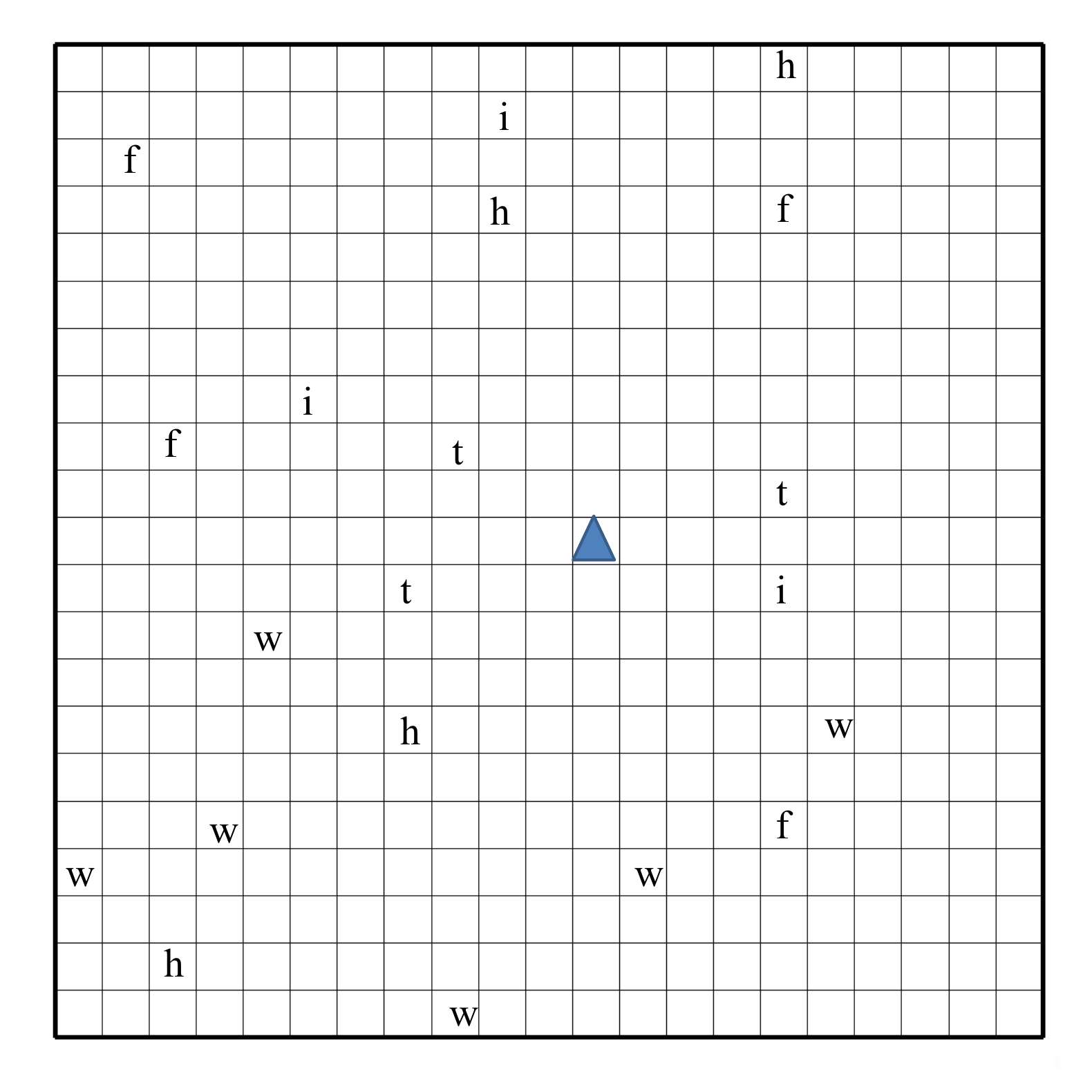}
	\caption{The map in the \craft.}
	\label{craft_map}
\end{figure}

\subsection{\CraftA}
For \craftA, Figure \ref{rm_c2_t1} shows the inferred hypothesis reward machine in the last iteration of \algoName{}. Figure \ref{case2_task1} shows the cumulative rewards of 10 independent simulation runs averaged for every 10 training steps for \craftA. 

\begin{figure}[H]
	\centering
	\begin{tikzpicture}[shorten >=1pt,node distance=2cm,on grid,auto] 
	\node[state,initial] (q_0)   {$\mealyCommonState_0$}; 
	\node[state] (q_1) [right=of q_0] {$\mealyCommonState_1$}; 
	\node[state] (q_3) at (2, -2) {$\mealyCommonState_3$}; 
	\node[state,accepting] (q_2) [right=of q_1] {$\mealyCommonState_2$}; 
	\path[->] 
	(q_0) edge  node {($\textrm{w}$, 0)} (q_1)
	edge [loop above] node {($\lnot\textrm{w}\wedge\lnot\textrm{t}$, 0)} ()    
	(q_0) edge  node [left] {($\textrm{t}$, 0)} (q_3)
	(q_1) edge  node  {($\textrm{t}$, 1)} (q_2)
	edge [loop above] node {($\lnot\textrm{t}$, 0)} ();
	\end{tikzpicture}
	\caption{The inferred hypothesis reward machine for \craftA~ in the last iteration of \algoName{}.}  
	\label{rm_c2_t1}
\end{figure}

\begin{figure*}[t]
	\centering
	\begin{subfigure}[b]{0.3\textwidth}  
		\centering
		\includegraphics[width=\textwidth]{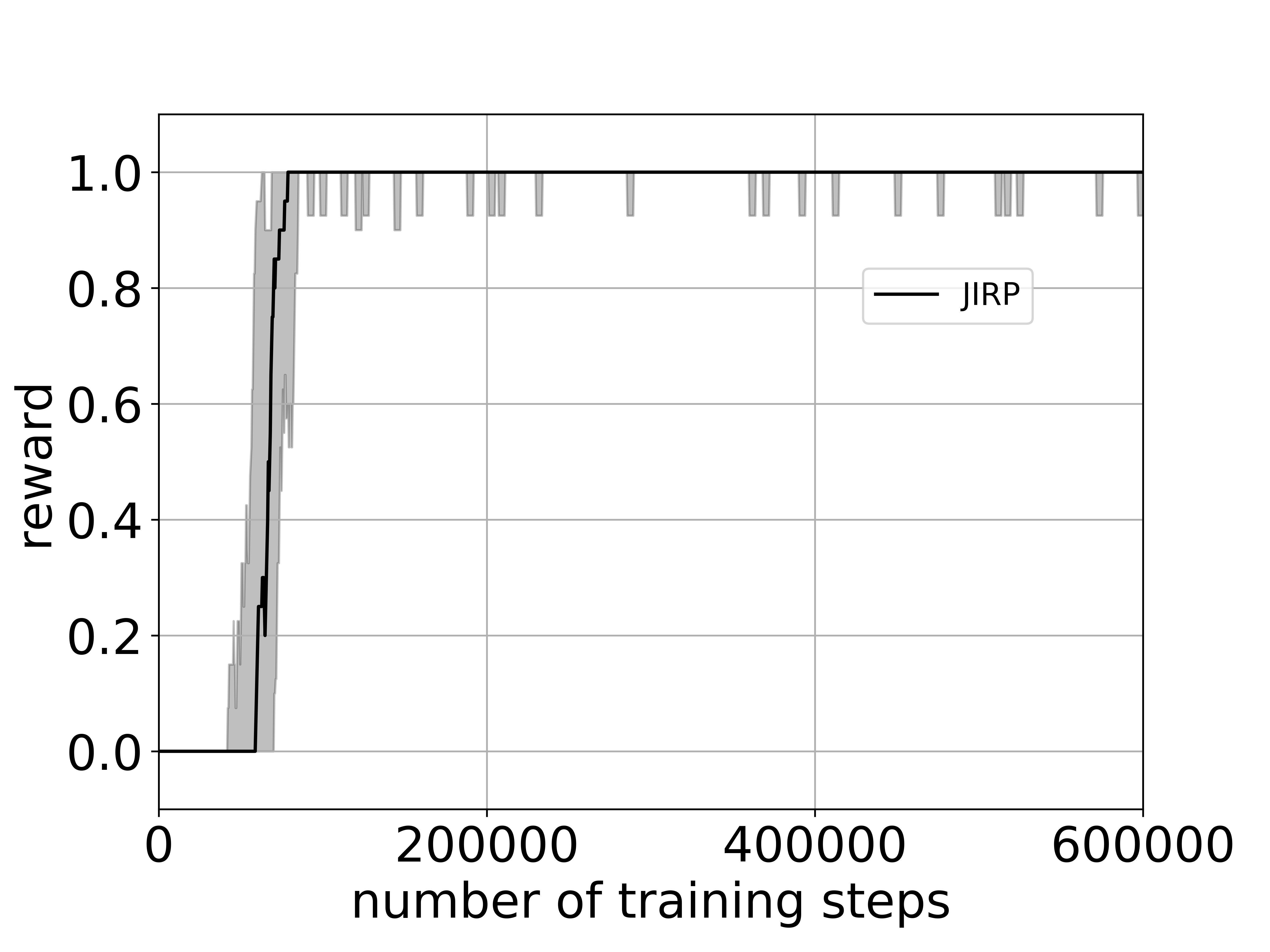}
		\caption{}
	\end{subfigure}
	\begin{subfigure}[b]{0.3\textwidth}
		\centering
		\includegraphics[width=\textwidth]{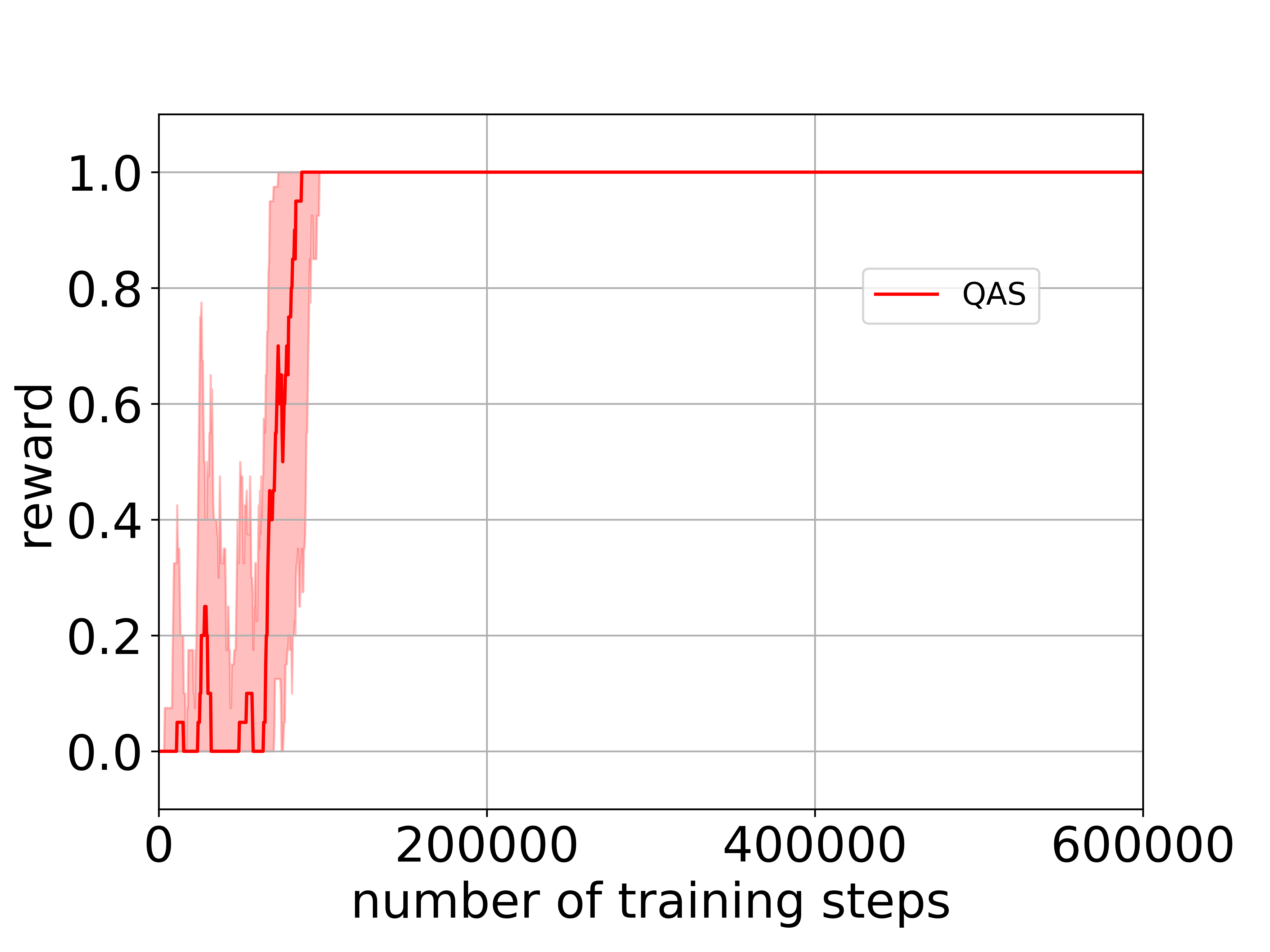}
		\caption{}
	\end{subfigure}
	\begin{subfigure}[b]{0.3\textwidth}
		\centering
		\includegraphics[width=\textwidth]{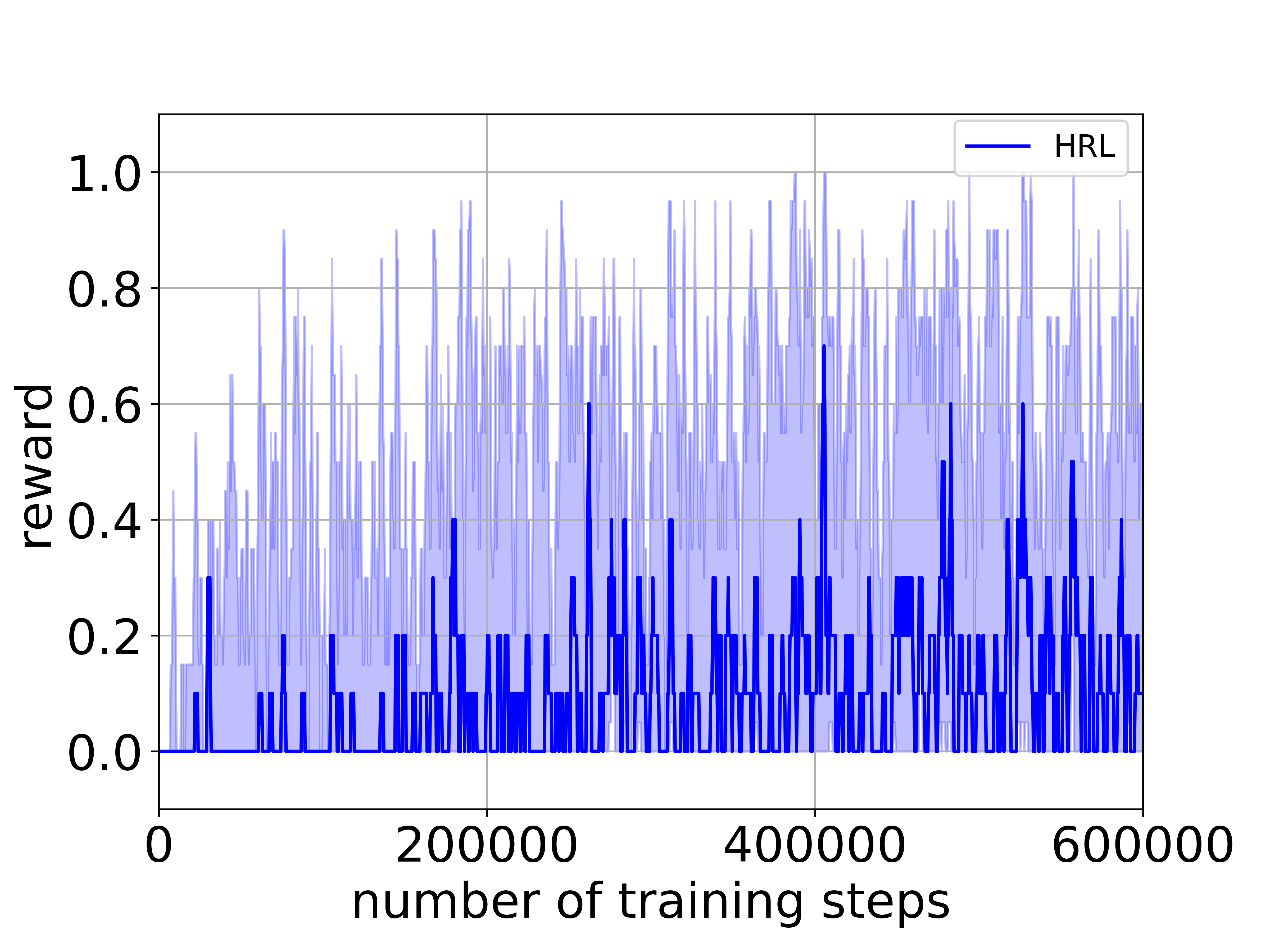}
		\caption{}
	\end{subfigure}
	\caption{Cumulative rewards of 10 independent simulation runs averaged for every 10 training steps for \craftA\ in the \craft: (a) \methodA; (b) \methodB; (c) \methodC.}  
	\label{case2_task1}
\end{figure*}

\subsection{ \CraftB}
For \craftB, Figure \ref{rm_c2_t2} shows the inferred hypothesis reward machine in the last iteration of \algoName{}. Figure \ref{case2_task2} shows the cumulative rewards of 10 independent simulation runs averaged for every 10 training steps for \craftB.                   

\begin{figure}[H]
	\centering
	\begin{tikzpicture}[shorten >=1pt,node distance=2cm,on grid,auto] 
\node[state,initial] (q_0)   {$\mealyCommonState_0$}; 
\node[state] (q_1) [right=of q_0] {$\mealyCommonState_1$}; 
\node[state,accepting] (q_2) [right=of q_1] {$\mealyCommonState_2$}; 
\path[->] 
(q_0) edge  node {($\textrm{w}$, 0)} (q_1)
edge [loop above] node {($\lnot\textrm{w}$, 0)} ()    
(q_1) edge  node  {($\textrm{h}$, 1)} (q_2)
edge [loop above] node {($\lnot\textrm{h}$, 0)} ();
\end{tikzpicture}
	\caption{The inferred hypothesis reward machine for \craftB~ in the last iteration of \algoName{}.}  
	\label{rm_c2_t2}
\end{figure}

\begin{figure*}[t]
	\centering
	\begin{subfigure}[b]{0.3\textwidth}  
		\centering
		\includegraphics[width=\textwidth]{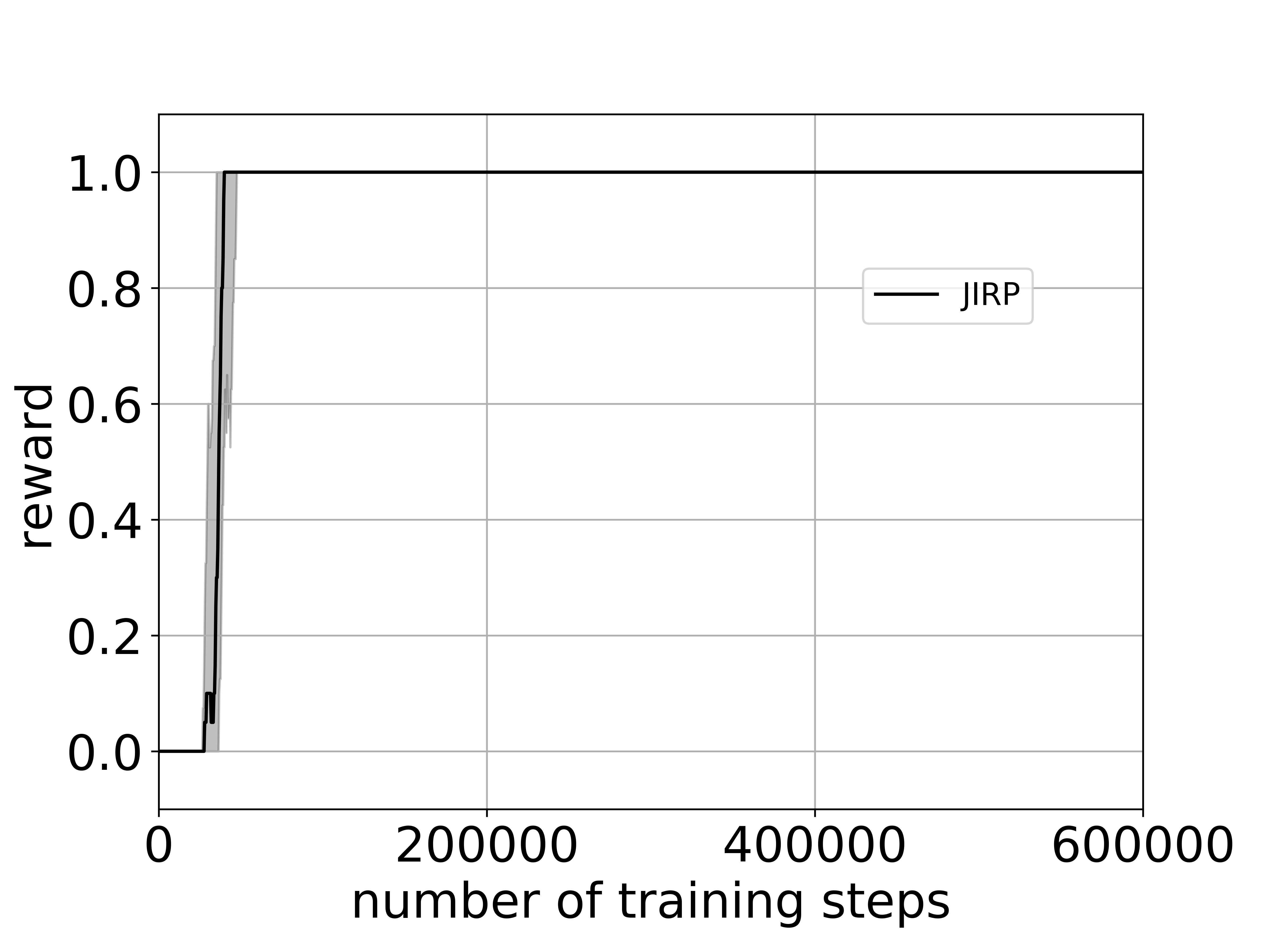}
		\caption{}
	\end{subfigure}
	\begin{subfigure}[b]{0.3\textwidth}
		\centering
		\includegraphics[width=\textwidth]{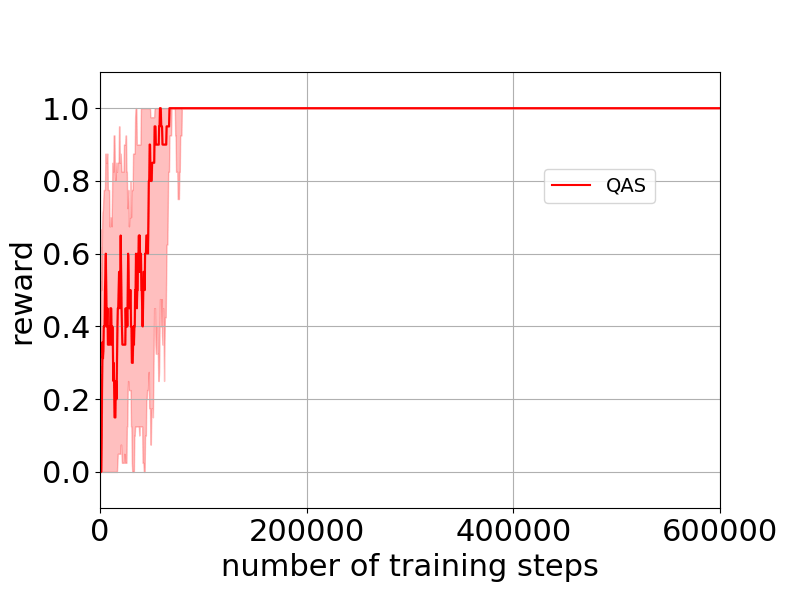}
		\caption{}
	\end{subfigure}
	\begin{subfigure}[b]{0.3\textwidth}
		\centering
		\includegraphics[width=\textwidth]{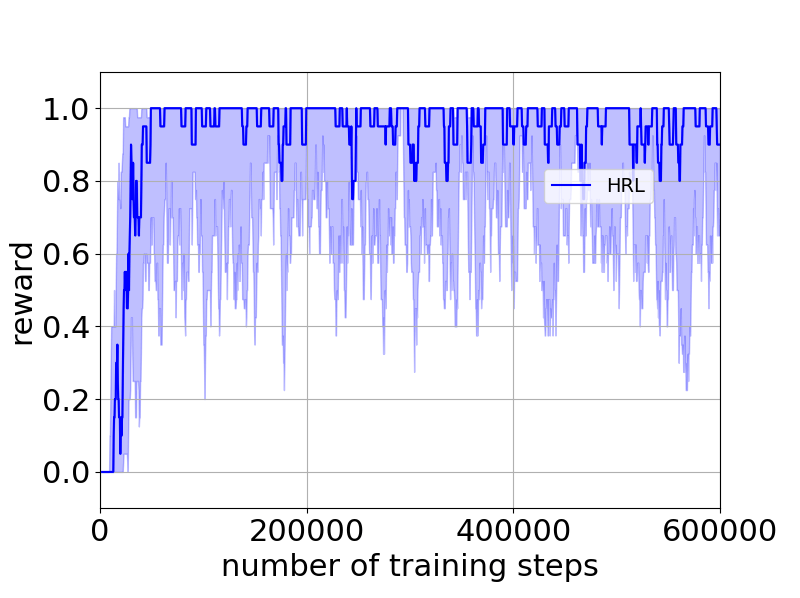}
		\caption{}
	\end{subfigure}
	\caption{Cumulative rewards of 10 independent simulation runs averaged for every 10 training steps for \craftB\ in the \craft: (a) \methodA; (b) \methodB; (c) \methodC.}  
	\label{case2_task2}
\end{figure*}

\subsection{ \CraftC}
For \craftC, Figure \ref{rm_c2_t3} shows the inferred hypothesis reward machine in the last iteration of \algoName{}. Figure \ref{case2_task3} shows the cumulative rewards of 10 independent simulation runs averaged for every 10 training steps for \craftC. 

\begin{figure}[H]
	\centering
	\begin{tikzpicture}[shorten >=1pt,node distance=2cm,on grid,auto] 
	\node[state,initial] (0) at (0.5, 0) {$\mealyCommonState_0$};
	\node[state] (1) at (3.5, 0) {$\mealyCommonState_1$};
	\node[state] (2) at (4, 4) {$\mealyCommonState_2$};
	\node[state] (3) at (-2, 3) {$\mealyCommonState_3$};
	\node[state, accepting] (4) at (0, 3) {$\mealyCommonState_4$};
	\draw[<-, shorten <=1pt] (0.west) -- +(-.4, 0);
	\draw[->] (0) to[loop above] node[align=center] {$(\lnot \textrm{h}, 0)$} ();
	\draw[->] (0) to[bend right=40] node[sloped, above, align=center] {$(\textrm{h}, 0)$} (1);
	\draw[->] (1) to[bend right=40] node[sloped, below, align=center] {$(\textrm{w}, 0)$} (2);
	\draw[->] (1) to[loop above] node[align=center] {$(\lnot\textrm{w}, 0)$} ();
	\draw[->] (1) to[bend right=40] node[sloped, above, align=center] {$(\textrm{h}\vee\textrm{f}, 0)$} (0);
	\draw[->] (2) to[loop above] node[align=center] {$(\lnot\textrm{h}, 0)$} ();
	\draw[->] (2) to[bend right=50] node[sloped, below, align=center] {$(\textrm{h}, 0)$} (3);
	\draw[->] (3) to[loop below] node[align=center] {$(\lnot\textrm{f}, 0)$} ();
	\draw[->] (3) to[right] node[sloped, below, align=center] {$(\textrm{f}, 1)$} (4);
	\draw[->] (3) to[right] node[sloped, below, align=center] {$(\textrm{t}, 0)$} (0);
	\end{tikzpicture}
	\caption{The inferred hypothesis reward machine for \craftC~ in the last iteration of \algoName{}.}  
	\label{rm_c2_t3}
\end{figure}

\begin{figure*}[t]
	\centering
	\begin{subfigure}[b]{0.3\textwidth}  
		\centering
		\includegraphics[width=\textwidth]{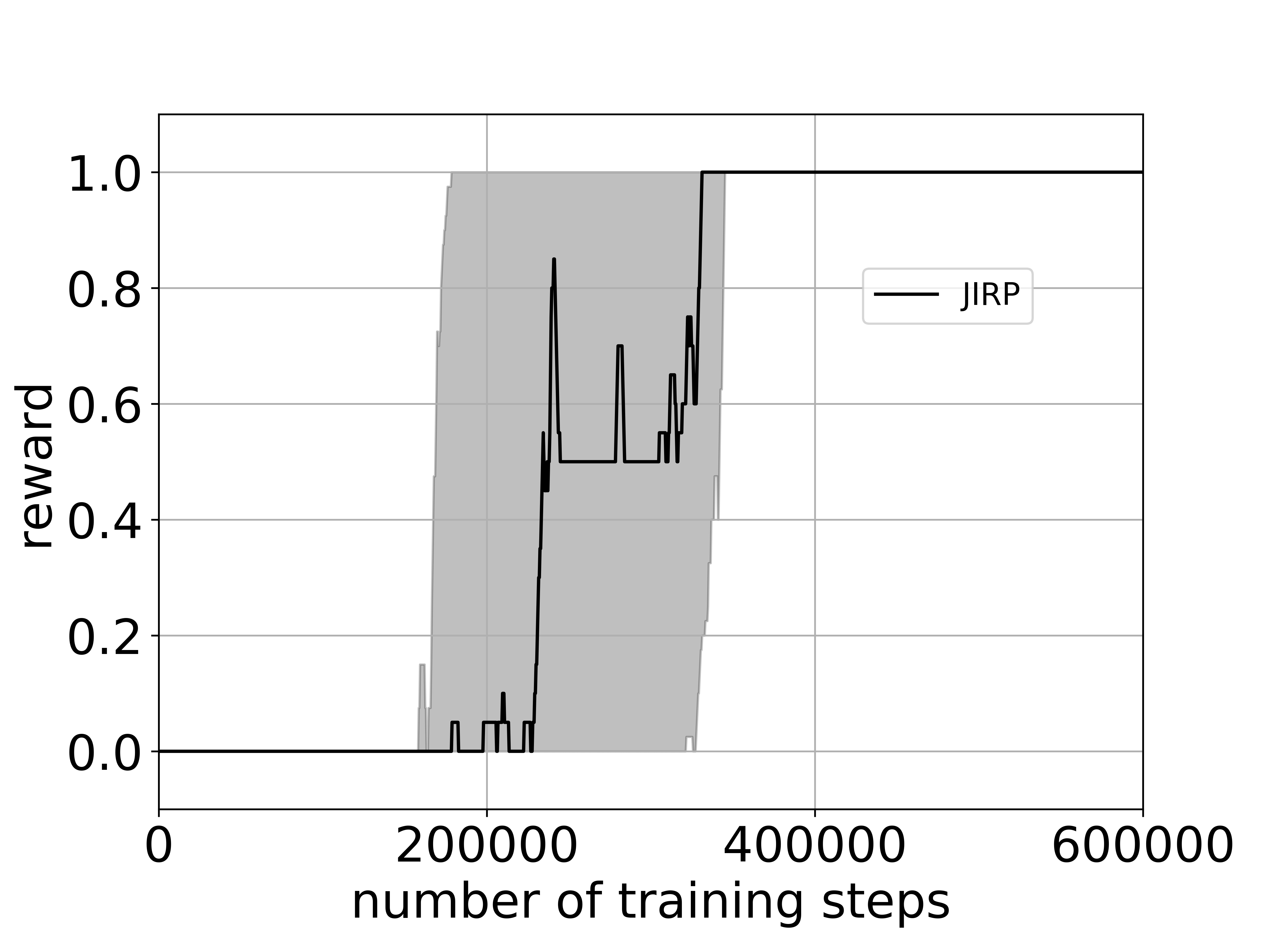}
		\caption{}
	\end{subfigure}
	\begin{subfigure}[b]{0.3\textwidth}
		\centering
		\includegraphics[width=\textwidth]{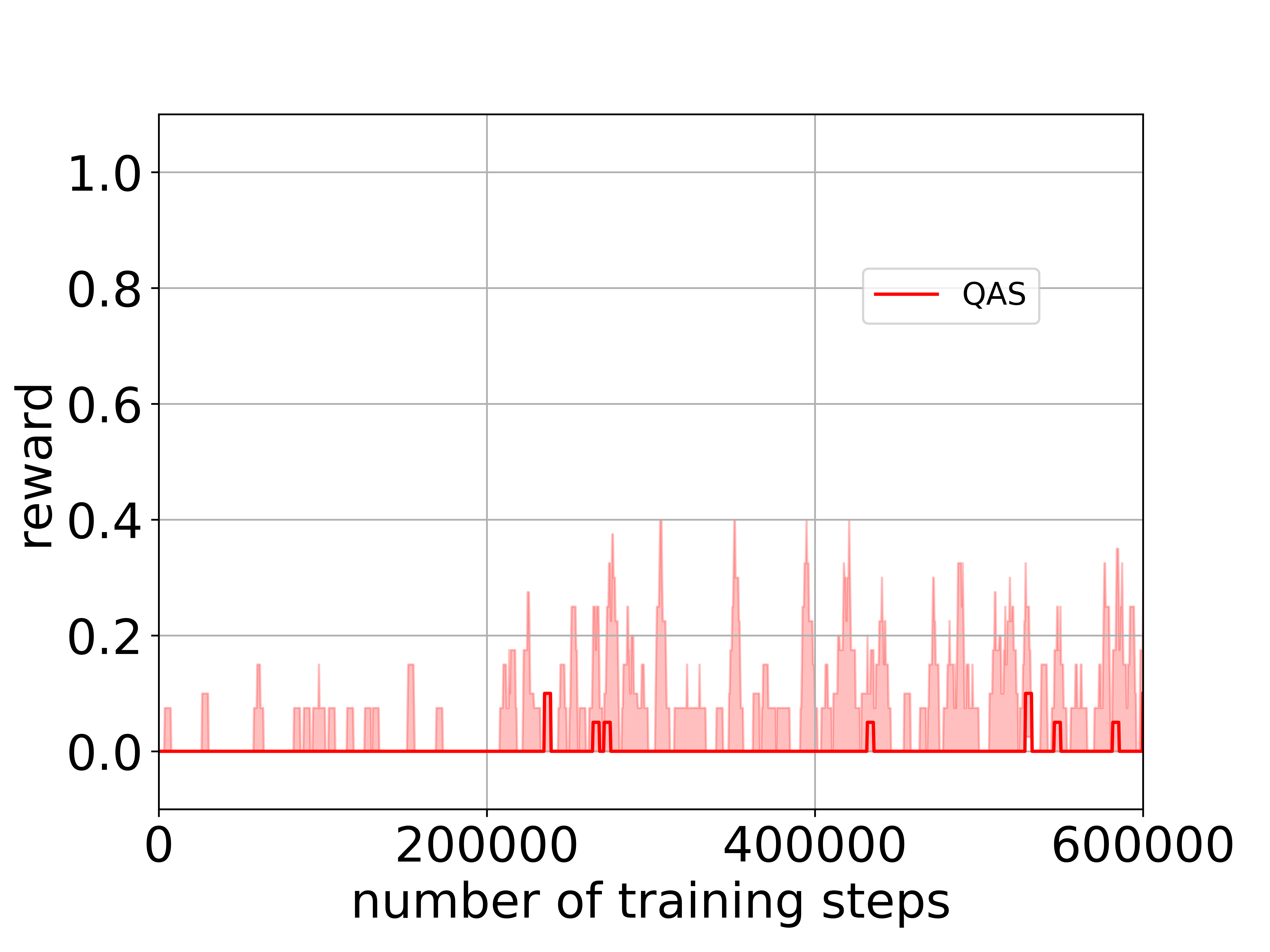}
		\caption{}
	\end{subfigure}
	\begin{subfigure}[b]{0.3\textwidth}
		\centering
		\includegraphics[width=\textwidth]{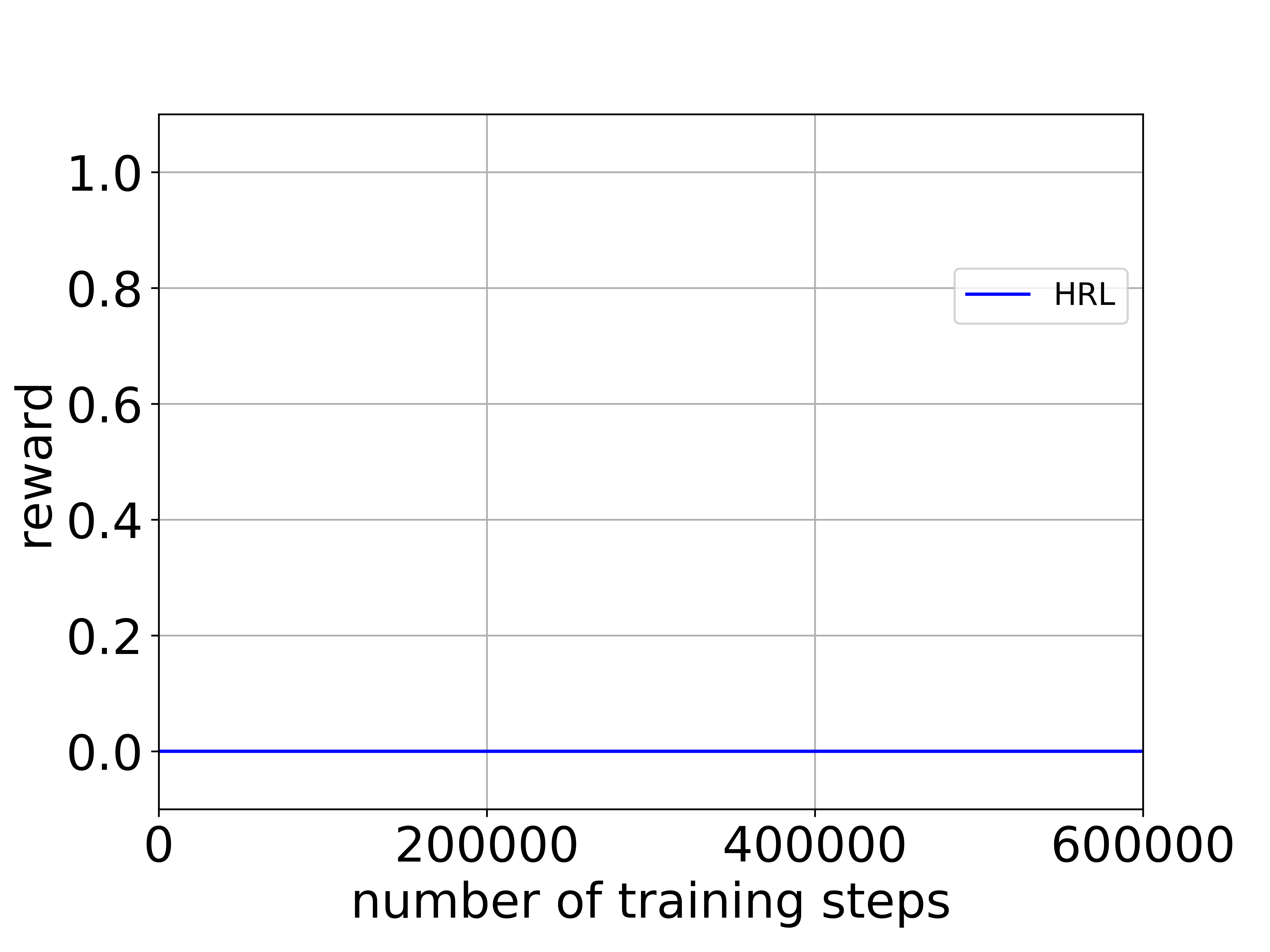}
		\caption{}
	\end{subfigure}
	\caption{Cumulative rewards of 10 independent simulation runs averaged for every 10 training steps for \craftC\ in the \craft: (a) \methodA; (b) \methodB; (c) \methodC.}  
	\label{case2_task3}
\end{figure*}

\subsection{ \CraftD}
For \craftD, Figure \ref{rm_c2_t4} shows the inferred hypothesis reward machine in the last iteration of \algoName{}. Figure \ref{case2_task4} shows the cumulative rewards of 10 independent simulation runs averaged for every 10 training steps for \craftD. 

\begin{figure}[H]
	\centering
	\begin{tikzpicture}[shorten >=1pt,node distance=2cm,on grid,auto] 
	\node[state,initial] (0) at (0.5, 0) {$\mealyCommonState_0$};
	\node[state] (1) at (3.5, 0) {$\mealyCommonState_1$};
	\node[state] (2) at (-1.5, 3) {$\mealyCommonState_2$};
	\node[state] (3) at (0.5, 4) {$\mealyCommonState_3$};
	\node[state, accepting] (4) at (3.5, 4) {$\mealyCommonState_4$};
	\draw[<-, shorten <=1pt] (0.west) -- +(-.4, 0);
	\draw[->] (0) to[loop above] node[align=center] {$(\lnot \textrm{w}, 0)$} ();
	\draw[->] (0) to[bend right=40] node[sloped, above, align=center] {$(\textrm{w}, 0)$} (1);
	\draw[->] (1) to[bend right=40] node[sloped, below, align=center] {$(\textrm{i}, 0)$} (2);
	\draw[->] (1) to[loop right] node[align=center] {$(\textrm{w}\vee\textrm{f}, 0)$} ();
	\draw[->] (1) to[bend right=40] node[sloped, above, align=center] {$(\textrm{h}\vee\textrm{t}, 0)$} (0);
	\draw[->] (2) to[loop above] node[align=center] {$(\textrm{f}\vee\textrm{t}\vee\textrm{i}, 0)$} ();
	\draw[->] (2) to[right] node[sloped, above, align=center] {$(\textrm{w}, 0)$} (3);
	\draw[->] (3) to[loop above] node[align=center] {$(\lnot\textrm{f}, 0)$} ();
	\draw[->] (3) to[right] node[sloped, below, align=center] {$(\textrm{f}, 1)$} (4);
	\draw[->] (2) to[right] node[sloped, below, align=center] {$(\textrm{h}, 0)$} (0);
	\end{tikzpicture}
	\caption{The inferred hypothesis reward machine for \craftD~ in the last iteration of \algoName{}.}  
	\label{rm_c2_t4}
\end{figure}

\begin{figure*}[t]
	\centering
	\begin{subfigure}[b]{0.3\textwidth}  
		\centering
		\includegraphics[width=\textwidth]{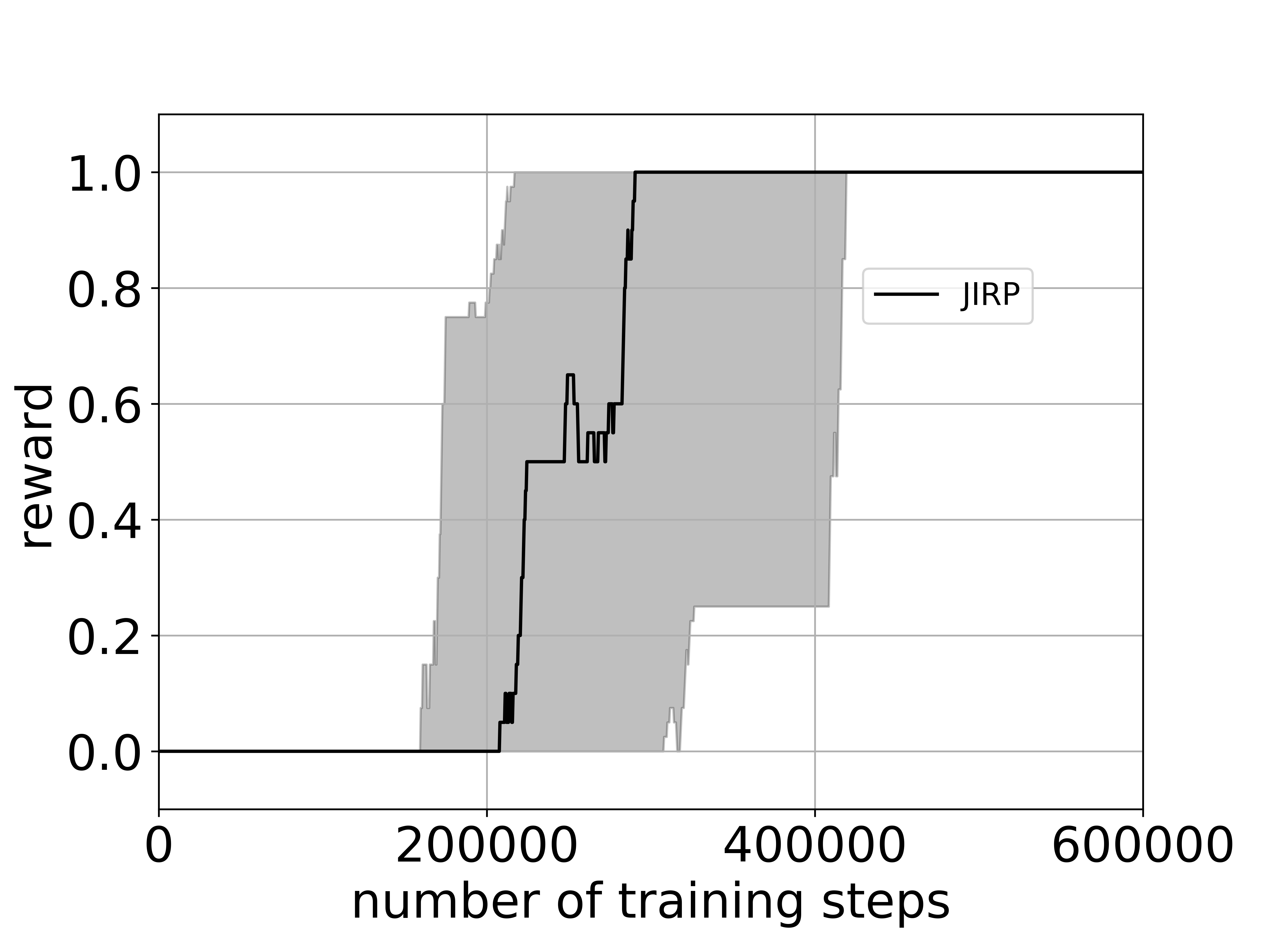}
		\caption{}
	\end{subfigure}
	\begin{subfigure}[b]{0.3\textwidth}
		\centering
		\includegraphics[width=\textwidth]{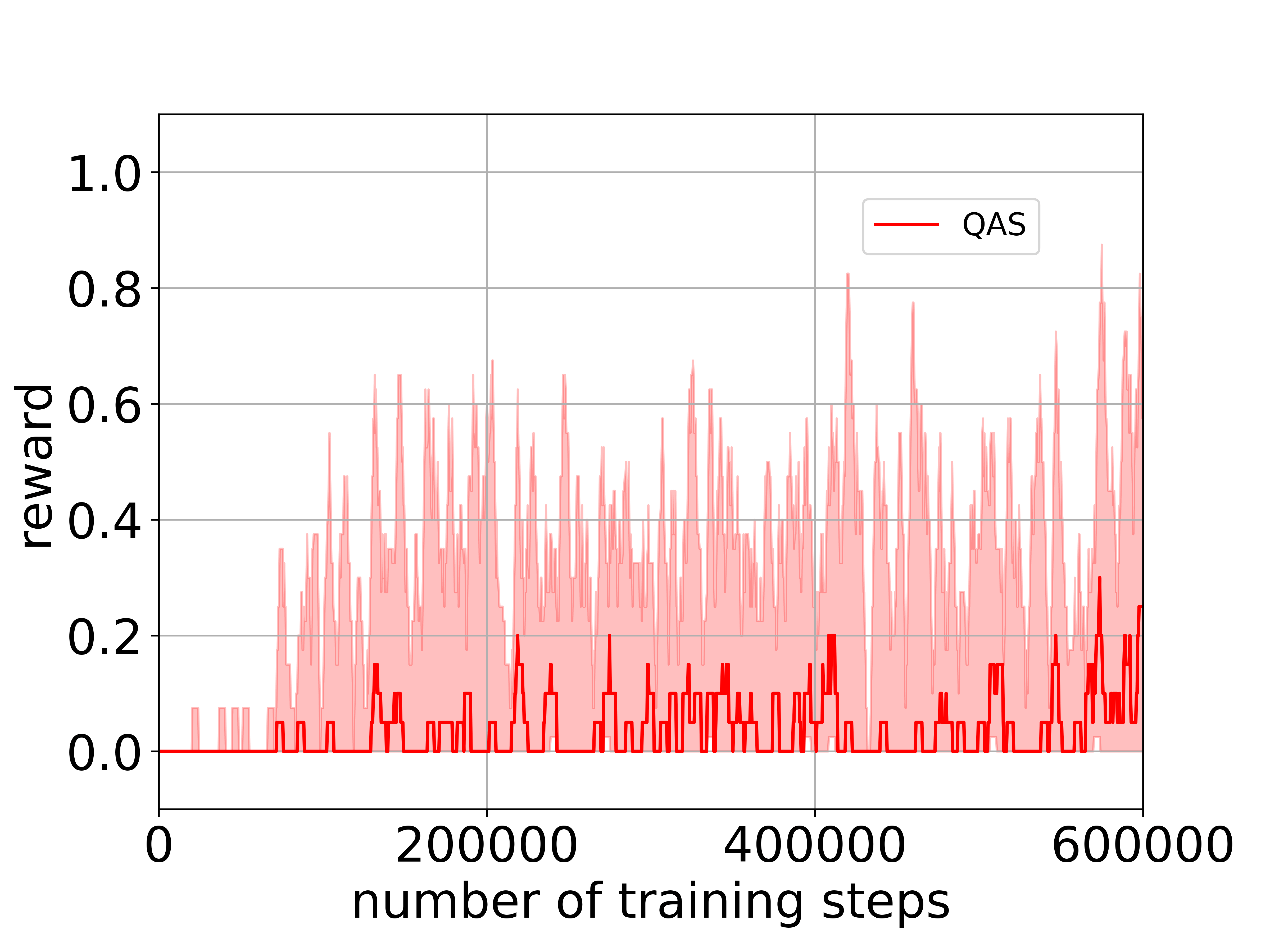}
		\caption{}
	\end{subfigure}
	\begin{subfigure}[b]{0.3\textwidth}
		\centering
		\includegraphics[width=\textwidth]{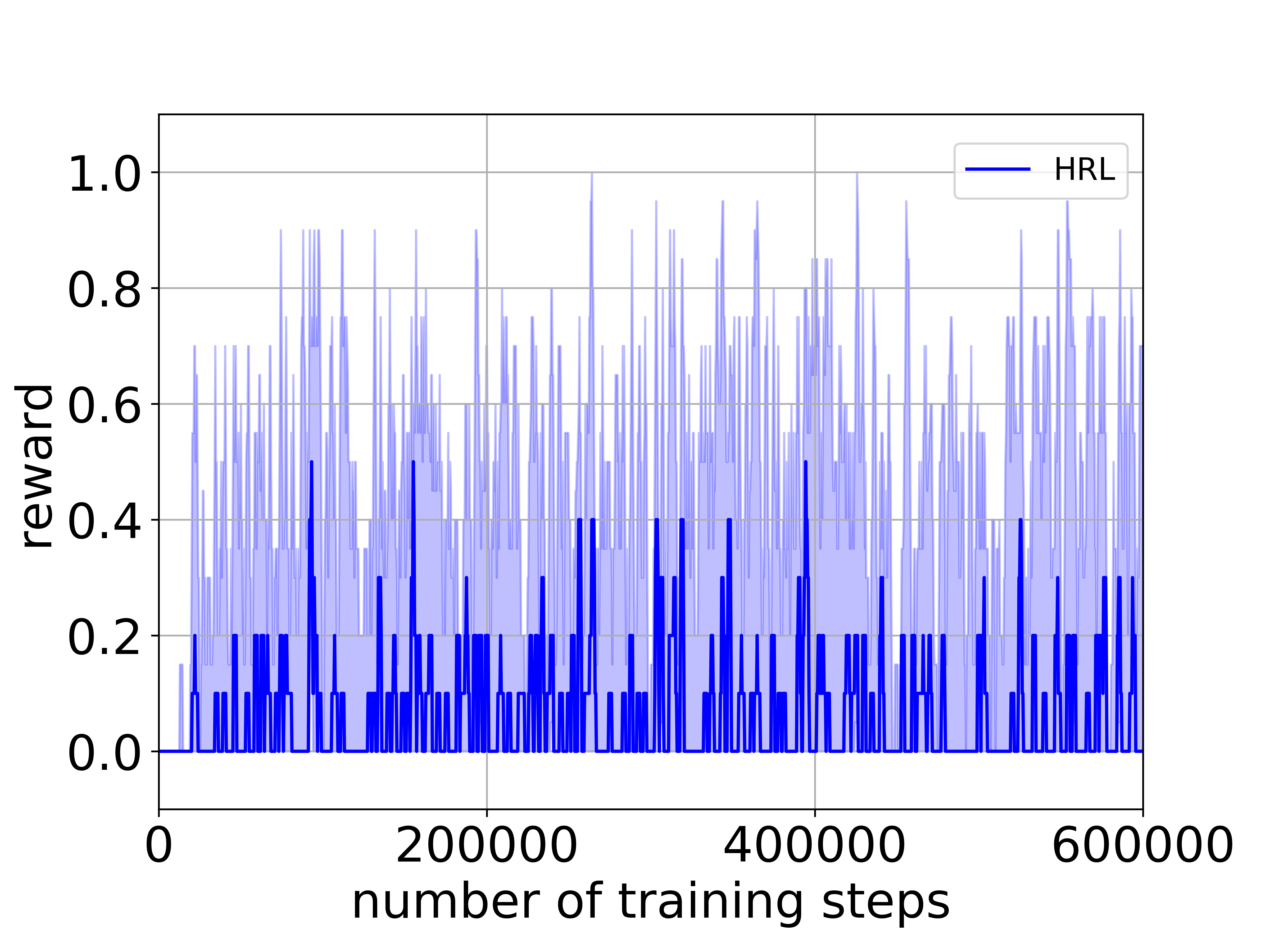}
		\caption{}
	\end{subfigure}
	\caption{Cumulative rewards of 10 independent simulation runs averaged for every 10 training steps for \craftD\ in the \craft: (a) \methodA; (b) \methodB; (c) \methodC.}  
	\label{case2_task4}
\end{figure*}

}{}

%\end{document}

}{}

\end{document}